\documentclass{article}

\usepackage{Yao}

\usepackage[utf8]{inputenc} 
\usepackage[T1]{fontenc}    
\usepackage{url}            
\usepackage{booktabs}       
\usepackage{amsfonts}       
\usepackage{nicefrac}       
\usepackage{microtype}      
\usepackage{xcolor}         

\usepackage[toc,page,header]{appendix}
\usepackage{minitoc}
\usepackage{silence}
\usepackage{bibunits}

\renewcommand{\clc}[1]{}
\renewcommand{\kj}[1]{}
\renewcommand{\kblue}[1]{{#1}}
\renewcommand{\blue}[1]{{#1}}

\usepackage{microtype}
\usepackage{graphicx}
\usepackage{subfigure}
\usepackage{booktabs} 

\usepackage{hyperref}

\usepackage[accepted]{icml2023_arxiv}

\usepackage[capitalize,noabbrev]{cleveref}

\usepackage[textsize=tiny]{todonotes}

\usepackage{comment}

\icmltitlerunning{Revisiting Simple Regret: Fast Rates for Returning a Good Arm}

\begin{document}
\doparttoc 
\faketableofcontents 

\begin{bibunit}[plainnat]




\twocolumn[
\icmltitle{Revisiting Simple Regret: Fast Rates for Returning a Good Arm}



\icmlsetsymbol{equal}{*}

\begin{icmlauthorlist}
\icmlauthor{Yao Zhao}{US_UA}
\icmlauthor{Connor James Stephens}{CA_UA}
\icmlauthor{Csaba Szepesvári}{CA_UA,DM}
\icmlauthor{Kwang-Sung Jun}{US_UA}
\end{icmlauthorlist}

\icmlaffiliation{US_UA}{University of Arizona}
\icmlaffiliation{CA_UA}{Amii, University of Alberta}
\icmlaffiliation{DM}{DeepMind}

\icmlcorrespondingauthor{Kwang-Sung Jun}{kjun@cs.arizona.edu}

\icmlkeywords{Machine Learning, ICML}

\vskip 0.3in
]
\printAffiliationsAndNotice{}

\begin{abstract}
Simple regret is a natural and parameter-free performance criterion for pure exploration in multi-armed bandits yet is less popular than the probability of missing the best arm or an $\epsilon$-good arm, perhaps due to lack of easy ways to characterize it.
In this paper, we make significant progress on minimizing simple regret in both data-rich ($T\ge n$) and data-poor regime ($T \le n$) where $n$ is the number of arms, and $T$ is the number of samples. 
At its heart is our improved instance-dependent analysis of the well-known Sequential Halving (SH) algorithm, where we bound the probability of returning an arm whose mean reward is not within $\epsilon$ from the best (i.e., not $\epsilon$-good) for \textit{any} choice of $\epsilon>0$, although $\epsilon$ is not an input to SH.
Our bound not only leads to an optimal worst-case simple regret bound of $\sqrt{n/T}$ up to logarithmic factors but also essentially matches the instance-dependent lower bound for returning an $\epsilon$-good arm reported by Katz-Samuels and Jamieson (2020).
For the more challenging data-poor regime, we propose Bracketing SH (BSH) that enjoys the same improvement even without sampling each arm at least once.
Our empirical study shows that BSH outperforms existing methods on real-world tasks.

\end{abstract}

\section{Introduction} 
\label{sec-intro}

We consider the pure exploration problem in multi-armed bandits.
In this problem, given $n$ arms, the learner sequentially chooses an arm $\dsym{I_t} \in [n]:=\{1,\ldots,n\}$ at time $t$ to observe a reward $\dsym{r_t} = \mu_{I_t} + \eta_t$ where $\dsym{\mu_i}$ is the mean reward of arm $i$ and $\dsym{\eta_t}$ is stochastic $\sigma^2$-sub-Gaussian noise.
Without loss of generality, we assume that $1\ge \mu_1 \ge \cdots \ge \mu_n\ge0$.
After $T$ time steps, the learner is required to output an arm $\dsym{J_T}$ that is estimated to have the largest mean reward.
The budget $T$ may or may not be given to the algorithm as input.
Such a problem was formalized by~\citet{even2006action} and~\citet{bubeck2009pure}, but similar problems were considered much earlier~\citep{bechhofer58sequential,chernoff59sequential}.
Recently, pure exploration has found many applications such as efficient crowdsourcing for cartoon caption contests~\citep{tanczos17akllucb}, hyperparameter optimization~\citep{li18hyperband, li20asystem}, and improving the time complexity of clustering algorithms~\citep{baharav19ultra}.
Hereafter, we take $\sig^2=1$ for ease of exposition, but all our results can easily be extended to $\sig^2$-sub-Gaussian noise.

Among many performance criteria, \textit{simple regret} proposed by \citet{bubeck2009pure} is a natural measure of the quality of the estimated best arm $J_T$:
\begin{align}\label{eq:sreg}
   \dsym{\SReg_T} := \EE\sbr{\mu_1 - \mu_{J_T}}~.
\end{align}
Simple regret is an \textit{unverifiable} performance measure, meaning that the algorithm need not verify its performance to a prescribed level. 
This is in stark contrast to the fixed confidence setting \citep{even2006action} such as ($\eps$,$\delta$)-PAC that requires the algorithm to \textit{verify} the quality of the estimated best arm to the prescribed target accuracy $\eps$ and confidence level $\delta$.
While certain applications do need such verification, there are many applications where the sampling budget is too limited to show meaningful guarantees, such as cartoon caption contests~\citep{jain2020new} and biological experiments~\citep{jun16top}.
In these cases, algorithms with verifiable guarantees may not be meaningful.
Furthermore, existing algorithms with verifiable guarantees are vacuous in the data-poor regime of $T \le n$ as it requires each arm to be pulled at least once~\cite[Section C.1]{Katz20}. 
The fixed budget setting considers an unverifiable measure of $\pp{\mu_{J_T} \neq \mu_{1}}$~\citep{audibert10best}, but this quantity inevitably depends on the smallest gap $\mu_1 - \mu_2$~\citep{Carpentier2016}, which is usually not meaningful until the sample size is very large.

\begin{table*}[t]
\caption{ 
    Sample complexity comparison with ME~\citep{even2006action}, $\cP_3$ \citep{chaudhuri2019pac},  and BUCB~\citep{Katz20}.
    The lower bound for Polynomial($\alpha$) instances (\cref{thm:gen_lb}, \cref{lem:poly_lb}) is a contribution of this paper which improves on \citet[Theorem 1]{Katz20} to include the correct dependence on $\dt$. The lower bound for EqualGap($m$) was first shown by \citet{chaudhuri2017pac}, which works for the fixed budget setting as well.
    We omit constant and logarithmic factors in $n$ and $m$.
    Methods with asterisk consider the fixed confidence setting, and those with dagger have an in expectation bound rather than the high probability one.
    Anytime algorithms do not require the knowledge of the budget $T$ as input, and parameter-free ones do not require any of $\{\alpha, m, \eps,\delta\}$ as input.
    }
    \label{tab:summary}
    \vspace{.5em}
  {\footnotesize
    \centering
    \renewcommand{\arraystretch}{1.5} 
    \begin{tabular}{M{0.10\linewidth}P{0.23\linewidth}P{0.23\linewidth}P{0.10\linewidth}P{0.06\linewidth}P{0.12\linewidth}}
      \hline
      & Polynomial$(\alpha)$ with $\alpha>\fr12$ & EqualGap$(m)$ with $m\le n/4$ & Data-poor & Anytime & Parameter-free \\ \hline
      ME$^*$ & $\fr{n}{\eps^2}\log\del{\fr{1}{\delta}}$ & $\fr{n}{\eps^2}\log\del{\fr{1}{\delta}}$ & \no & \no & \no\\ 
      $\cP_3^*$ & $\fr{1}{\eps^{2+(1/\alpha)}}\log\rbr{\fr{1}{\delta}}\vee\fr{1}{\eps^{2}}\log^2\rbr{\fr{1}{\delta}}$ & $\frac{1}{\eps^2}\del{\frac{n}{m}\log\del{\fr{1}{\delta}}+\log^2\del{\fr{1}{\delta}}}$ & \yes & \no & \no\\ 
      BUCB$^\dagger$ & $\fr{n}{\eps^{2-(1/\alpha)}}\log\del{\fr{1}{\delta}}$ & $\fr{n}{m\eps^2}\log\del{\fr{1}{\delta}}$ & \yes & \yes & \no \\ \hline
      DSH (ours)      & $\fr{1}{\eps^2}\log\del{\fr{1}{\delta}}\vee\fr{n}{\eps^{1/\alpha}}$ & $\fr{n}{m\eps^2}\log\del{\fr{1}{\delta}}\vee\fr{n}{\eps^2}$ & \no & \yes & \yes\\ 
      BSH (ours) & $\fr{1}{\eps^2}\log\del{\fr{1}{\delta}}$ & $\fr{n}{m\eps^2}\log\del{\fr{1}{\delta}}$ & \yes & \yes & \yes \\ \hline
      Lower bound & $\fr{1}{\eps^{2}}\log\del{\fr{1}{\delta}}$ & $\fr{n}{m\eps^2}\log\del{\fr{1}{\delta}}$ &   &   &  \\ \hline
    \end{tabular}
    \renewcommand{\arraystretch}{1.0}
  }
\end{table*}

Despite being attractive, the analysis of simple regret has been elusive.
Existing studies focus on either achieving minimax simple regret bound~\citep{lattimore16causal} or characterization of how simple regret is different from the cumulative regret~\citep{bubeck2009pure}.
It is not known whether simple algorithms like Sequential Halving (SH)~\citep{karnin2013almost} enjoy the optimal simple regret or not.

In this paper, we revisit simple regret and make two main contributions.
First, we provide a novel and tight analysis of Sequential Halving (SH), one of the state-of-the-art algorithms for pure exploration.
We analyze a strong performance criterion that we call \textit{$\eps$-error probability}:
\begin{align}\label{eq:def-eps-error}
  \pp{\mu_1 - \mu_{J_T} > \eps}
\end{align}
where $\epsilon$ is not given to the algorithm as input.
This is precisely characterizing the distribution of $\mu_{J_T}$, which was mentioned as an interesting direction of research in~\citet[Section 33.4 Note 9]{lattimore2020bandit}.
If an algorithm $\cA$ does not take $\epsilon$ as input yet has a guarantee for any $\eps$, we say $\cA$ enjoys a  \textit{uniform $\eps$-error  probability bound}. 
The $\eps$-error probability is stronger than simple regret since 
\begin{align}\label{eq:simpleregret-identity}
  \SReg_T = \EE[\mu_1 - \mu_{J_T}] = \int_0^\infty \PP(\mu_1 - \mu_{J_T} > \eps) \dif \eps~.
\end{align}
Thus, a uniform $\eps$-error probability bound implies a simple regret bound.

We call an arm $i$ to be $\eps$-good if $\mu_{i} \ge \mu_1 - \eps$.
Let $\Delta_i := \mu_1 - \mu_i$ be the gap of arm $i$ and $g\rbr{\eps}:=\abr{\cbr{i\in [n] \mid \mu_i\ge\mu_1-\eps}}$ be the number of $\eps$-good arms.
We prove that SH~\citep{karnin2013almost} has the following bound for any $\eps>0$ (Theorem~\ref{them:ins_dep}):
\begin{align}\label{eq:guarantee}
    \pp{\mu_{J_T}<\mu_1-\eps}\le\ep{-\hTT{\fr{T}{H_2(\eps)}}},
\end{align}
where $\hTT{\cd}$ hides logarithmic factors and
\begin{align*}
    H_2(\eps):=\frac{1}{g(\eps/2)}\max_{i\ge g(\eps)+1}\frac{i}{\Delta_i^2}~.
\end{align*} 
is the sample complexity function.
By setting $\eps = 0$, the complexity becomes $H_2(0) = \max_{i\ge2} i \Delta_i^{-2} =: H_2$, so our bound is never worse than the original bound of SH presented in~\citet{karnin2013almost}:  $\PP(\mu_{J_T} \neq \mu_1) \le \exp(-\tilde\Theta(\frac{T}{H_2}))$. 
Furthermore, our bound essentially matches the lower bound presented in~\citet{Katz20}.
We also derive a minimax-style bound on the $\eps$-error probability, which can be easily integrated out over $\eps$ to obtain $\SReg_T = \tilde O(\sqrt{n/T})$, and further show that a variant of SH results in removing logarithmic factors from the bound and achieve $\SReg_T = O(\sqrt{n/T})$.
Finally, we derive an anytime version of SH that we call Doubling SH (DSH), which enjoys the same guarantees without requiring $T$ as input.
We present details of these claims in Section~\ref{sec-sh}.

As the second contribution, we propose an improved pure exploration algorithm for the data-poor regime.
This is motivated by the fact that the guarantee~\eqref{eq:guarantee} becomes vacuous in the data-poor regime ($T \le \max_{i\ge g(\eps)+1} i \Delta_i^{-2}$ to be precise) as it does not even allow us to pull every arm once.
Inspired by~\citet{Katz20}, we propose Bracketing SH (BSH) that progressively subsamples a larger set of arms (i.e., brackets) for which we invoke SH to find a good arm.
Our analysis shows that BSH enjoys a bound that is essentially the same as~\eqref{eq:guarantee} but works even for $T \le \max_{i\ge g(\eps)+1} i \Delta_i^{-2}$. 
Compared to BUCB~\citep{Katz20}, the state-of-the-art algorithm for the data-poor regime, BSH achieves three improvements: (i) BSH is parameter-free whereas BUCB requires the target error rate $\delta$ as input, (ii) BSH has a high probability sample complexity bound as opposed to the expected sample complexity bound of BUCB, (iii) BSH shows a much stronger bound than BUCB. 
We provide details on BSH in Section~\ref{sec-bsh}.

\begin{figure}[t]
\centering
\includegraphics[width=\linewidth]{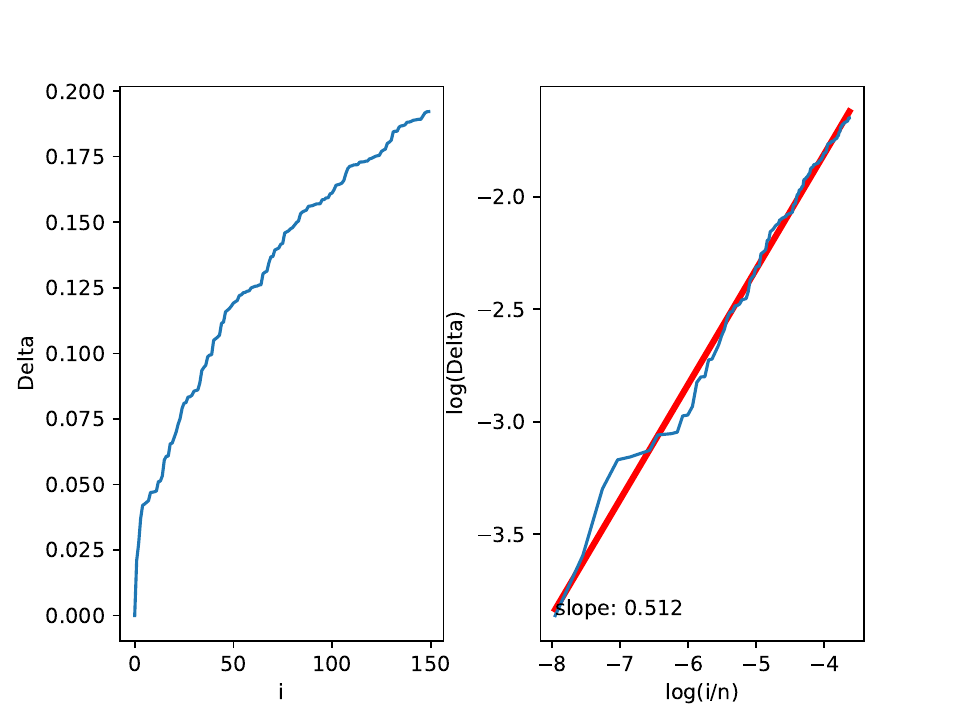}
\caption{The polynomial fitting for the best $150$ arms of the New Yorker Cartoon Caption contest 788. }
\label{fig-polyfit}
\end{figure}
We compare our bounds with existing bounds in various setups in Table~\ref{tab:summary}.
For ease of comparison, we consider special instances called Polynomial$(\alpha)$ and EqualGap$(m)$.
Polynomial$(\alpha)$ is an instance with gap structure $\Delta_i=\rbr{\fr{i}{n}}^\alpha$ for $2\le i \le n$, where $\alpha>0.5$~\citep{jamieson2013finding}, and EqualGap$(m)$ is an instance that has $m-1$ arms with the gap $\eps/4$ and $n-m$ arms with the gap $\fr54 \eps$.
Both $\alpha$ and $m$ have the role of adjusting the number of good arms.
Our bound for BSH matches the lower bounds for the first time to our knowledge, thereby establishing the optimality for these two instances. Note that many real-world instances follow Polynomial$(\alpha)$; e.g., see Figure~\ref{fig-polyfit}.
One interesting aspect of Polynomial($\alpha$) is that, although $\alpha$ affects the fraction of good arms, the optimal rate $1/\eps^2$ is independent of $\alpha$.
This is because the hardness of returning an $\eps$-good arm depends not just on how many good arms we have but also on how far they are from bad arms as we explain in Section \ref{sec-sh}.

Our result also resolves the open problem of having a factor of $\ln^2(1/\dt)$ rather than $\ln(1/\dt)$ when obtaining accelerated sample complexities under the presence of multiple good arms.
Such an issue appeared repeatedly in the literature; e.g., \citet{chaudhuri2019pac}, \citet[Section E]{Katz20}, and \citet[Section 5]{aziz2018pure}.
They all rely on the subsampling device or its variants.
Our improved bound can be attributed to the fact that an accelerated sample complexity can be obtained even \textit{without the subsampling device} -- we are the first to achieve such a result, to our knowledge. 

We evaluate our algorithm on the cartoon caption contest dataset in Section~\ref{sec-expr} and conclude our paper with future work in Section~\ref{sec-future}.
Due to space constraints, we have the related work section in the appendix, but directly-related works are discussed throughout the paper.
All the proofs are deferred to the appendix.

\section{Preliminaries}
\label{sec-prelim}

We focus on the fixed budget and anytime setting in pure exploration.
In the \textit{fixed budget} setting~\citep{audibert10}, the learner is given a budget $T \ge 1$ as input and is required to return an estimated best arm $J_T \in [n]$ at the end of $T$-th time step. 
In the \textit{anytime} setting, there is no prescribed $T$, and the learner is required to output $J_t \in [n]$ for every time step $t\ge1$.
In these settings, unlike the fixed confidence setting, the learner is not required to verify the quality of the returned arm $J_T$.
We collectively call these settings the \textit{unverifiable} setting and refer to the fixed confidence setting as the \textit{verifiable} setting.
Note that for any verifiable setting, one can consider an unverifiable version.
For example, the verifiable $(m,\eps)$-good arm identification problem~\citep{chaudhuri2017pac} requires that the learner, given $(m,\eps,\dt)$, identify an arm $i$ that is $(m,\eps)$-good (i.e., $\mu_i\ge\mu_m-\eps$) w.p. at least $1-\dt$. 
One can consider the unverifiable $(m,\eps)$-good arm identification problem where the goal is to minimize $\PP(\mu_{J_T}\ge\mu_m-\eps)$ after $T$ time steps.

\textbf{Sample Complexity.}
When comparing bounds for simple regret or $\eps$-error probability, it is often easier to consider the sample complexity, which measures the number of samples required to achieve the target performance level.
For example, the sample complexity for the $\eps$-error probability is the smallest time step $\tau$ such that $\forall t\ge \tau, \pp{\mu_1 - \mu_{t} > \eps} \le \dt$.

\textbf{Notations.~}
Throughout the paper, ``const'' is a universal and positive constant, which may have different values for different expressions. 
Both $\tilde\Theta(\cd)$ and $\tilde \cO(\cd)$ omits logarithmic factors on $m$, $n$, but not $1/\dt$.

\section{Improved Analyses of Sequential Halving}
\label{sec-sh}

Sequential Halving (SH)~\citep{karnin2013almost} is a simple and parameter-free algorithm that has had an outsize impact on pure exploration since its introduction. Many recent practical algorithms for pure exploration \citep{aziz2022identifying, li18hyperband, li20asystem, de2021bandits} make use of this algorithm, either with slight modifications or as a subroutine. These works have demonstrated that the algorithm's practical and theoretical significance is greater than its original presentation. 

We present SH in Algorithm~\ref{Alg:SH}.
SH consists of $\lceil \log_2(n) \rceil$ stages.
In each stage $\ell$, the algorithm performs a uniform sampling over the surviving arms $S_{\ell}$ followed by eliminating the bottom half of $S_\ell$ w.r.t. empirical means and sets $S_{\ell+1}$ as the resulting arm set.

SH is known to achieve a nearly optimal instance-dependent bound on the $(\eps=0)$-error probability, which we detail in Table~\ref{tab:sh} in the appendix. 
However, it is not known if SH is a good algorithm for returning an $\eps$-good arm.
For $\eps>0$, though in the fixed confidence setting, Median Elimination (ME) achieves the sample complexity of $\frac{n}{\eps^2} \log\rbr{\fr{1}{\delta}}$~\citep{even2006action}.
Does SH achieve a similar sample complexity for $\eps$-error probability?
Despite the similarity between SH and ME, the answer is not clear because the sampling scheme of SH is different from ME. 
In light of the ubiquity of this algorithm and recent empirical success, 
a stronger characterization of the performance of SH is of great interest to the community.

In this section, via a powerful analysis of the $\eps$-error probability incurred by SH, we show that this algorithm achieves an optimal bound for the unverifiable $(m,\eps)$-good arm identification problem that further implies a minimax optimal simple regret bound~\cite[Exercise 33.1]{lattimore2020bandit}. 
We also provide a near-optimal instance-dependent uniform $\eps$-error probability bound for the first time.

Our main theorem shows the following instance-dependent uniform $\eps$-error probability bound for SH.
\begin{theorem}
\label{them:ins_dep}
For any $\eps\in(0,1)$, the $\eps$-error probability of SH satisfies
\begin{align*}
    \pp{\mu_{J_T}<\mu_1-\eps}\le\ep{-\hTT{\fr{T}{H_2(\eps)}}},
\end{align*}
for $T\ge\hTT{\max_{i\ge g(\eps)+1}\frac{i}{\Delta_i^2}}$.
\end{theorem}
The proof is in Appendix \ref{app-sec-th1}. 
Note the requirement of $T\ge\hTT{\max_{i\ge g(\eps)+1}\frac{i}{\Delta_i^2}}$ is not a weakness of our result. One can easily see that if $\Delta_i = \eps, \forall i\ge 2$, the standard result of SH~\citep{karnin2013almost} becomes vacuous for $T = o(\frac{n}{\eps^2})$.
We stress that, unlike all existing work we are aware of, Theorem~\ref{them:ins_dep} bounds the $\eps$-error probability for an algorithm that does not require $\eps$ as input. That is to say, $\eps$ is a parameter for analysis only. 
This upper bound of $\eps$-error probability implies a sample complexity of $\tilde O(H_2(\eps)\ln(1/\dt))$.

\textbf{Polynomial gap instances.}
The class of polynomial gap instances was introduced by \citet{jamieson2013finding}, who referred to them as `$\alpha$-parameterized' instances. The authors studied these instances in the best-arm identification setting, showing that for $\alpha > 1/2$, the sample complexity $H=\sum_{i=2}^n \Delta_i^{-2}$ scales like $\Omega \rbr{n^{2\alpha}}$, with the increase in sample complexity corresponding to the growing number of close-to-optimal arms as $\alpha$ increases. In one of the highlight results of this work, we show that when $\eps>0$ is not pathologically small, SH returns $\eps$-good arms on these instances with an optimal sample complexity that is independent of the number of arms $n$. 

\begin{corollary} 
\label{cor:poly}
For Polynomial$(\alpha)$ with $\alpha>\fr12$, we have
\begin{align*}
H_2(\eps)=\frac{1}{g\rbr{\fr{\eps}{2}}} \max_{i\ge g(\eps)+1}\frac{i}{\Delta_i^2}<\frac{4}{\eps^2}.
\end{align*}
\end{corollary}
The proof is in Appendix \ref{app-sec-poly}.
While this result initially seems surprising, the lack of dependence on $n$ stems from the ratio between the number of `well-separated' $\eps$-good arms and the sample complexity of filtering out $\eps$-bad arms (i.e., non-$\eps$-good arms) on these instances. 
As $n$ increases, due to the polynomial gap structure, the number of well-separated $\eps$-good arms (e.g., the collection of $\eps/2$-good arms) increases linearly with $n$, which balances the linear growth in $n$ of the sum over the $\eps$-bad arm gaps $\sum_{i=g(\eps)+1}^n \Delta_i^{-2}$. 
This second term is effectively the sample complexity of discarding all $\eps$-bad arms. 
Since we only require that only one of the good arms to be returned, we get to divide $\sum_{i=g(\eps)+1}^n \Delta_i^{-2}$ by $g(\eps/2)$.

\textbf{Lower bounds.}
We compare our upper bound in Theorem~\ref{them:ins_dep} with lower bounds.
For brevity, we present an informal version here and defer the full version to Appendix~\ref{app-sec-lb} along with a  slightly stronger version.
\begin{theorem}(Informal version of \citet[Theorem 1]{Katz20})\label{thm:katz_lb}
    For any algorithm, there exists an instance where the expected unverifiable sample complexity for returning an $\eps$-good arm with probability greater than 15/16 satisfies
    \begin{align*}
        \Omega\del{H^{\text{low}}(\eps) := {\frac{1}{g(\eps)}\sum_{i=g(\eps)+1}^{n}\rbr{\frac{1}{\Delta_i^2}}-\frac{1}{\Delta_{g(\eps)+1}^2}}}.
    \end{align*}
\end{theorem}

\begin{corollary}[Comparison with previous lower bound] 
\label{cor:kjlow}
In the case that $g(\eps)$ and $g\rbr{\frac{\eps}{2}}$ have the same order ($\fr{g\rbr{\fr{\eps}{2}}}{g\rbr{\eps}}$ is irrelevant to $\eps$), the sample complexity measure $H_2(\eps)$ satisfies
\begin{align*}
    H^{\text{low}}(\eps)\lesssim H_2(\eps)\lesssim H^{\text{low}}(\eps) +\frac{2}{\Delta_{g(\eps)+1}^2}
\end{align*}
where $\lesssim$ hides constants and logarithms of $n$ and $1/\dt$.
\end{corollary}
The sample complexity implied by Theorem~\ref{thm:katz_lb} for returning an $\eps$-good arm with probability $1-\delta$, with $\delta \in (0,1/16)$, is $\Omega\rbr{H^{\text{low}}(\eps)}$, which does not have the factor $\ln(1/\dt)$.
We contribute a new lower bound that suggests a sample complexity of $\Omega\rbr{\tilde{H}(\eps)\log(1/\dt)}$, with a complexity term $\tilde{H}(\eps)$ that we describe next, and with which we show that Theorem~\ref{them:ins_dep} is essentially tight on polynomial gap instances with $\alpha > 1/2$.

For an $n$-armed bandit instance $\nu$, let
\begin{equation*}
    \tilde{H}(\nu, \eps) = \sum_{i=2}^{n/g(\eps)} \Delta_{i\cd g(\eps)}^{-2},
\end{equation*}
where for simplicity we assume that $n$ is an integer multiple of $g(\eps)$.

\begin{theorem}\label{thm:gen_lb}
  Fix $n>2$, $\eps>0$ and an $n$-armed unit-variance Gaussian instance $\nu$. With $a \coloneq \tH(\nu, \eps)$ we define $\cB_n(a,\eps)$, the collection of $n$-armed bandit instances with unit-variance gaussian arms for which $\tilde{H}(\nu', \eps) \le a$ for each $\nu'\in \cB_n(a,\eps)$. Then for any algorithm $\pi$ and sample budget $T \in \mathbb{N}$,
  \begin{align*}
    \sup_{\nu' \in \cB_n(a,\eps)} &\PP_{\nu',\pi}\rbr{\mu_{J_T} < \mu_1 - \eps}\\
    &\ge \max\set{\frac{1}{2}, \frac{1}{36} \exp\rbr{-24T/a}}.
  \end{align*}
\end{theorem}

    The complexity term $\tilde{H}(\nu, \eps)$ plays the same role in our result as that of $H(\nu) = \sum_{i=2}^n \Delta_i^{-2}$ in \citet[Theorem 1]{Carpentier2016}. The proof of Theorem~\ref{thm:gen_lb} can be found in Appendix \ref{app-sec-genlb}.

\begin{lemma}\label{lem:poly_lb}
Consider a Polynomial($\alpha$) instances with fixed $\alpha > 1/2$. Then for any $\eps \le 1/2$, $\tilde{H} \ge \frac{c(\alpha)}{\eps^2}$ where $c(\alpha)$ is a problem-dependent constant which does not depend on $\eps$.
\end{lemma}
The proof is in Appendix~\ref{app-sec-polylb}.
Combining Corollary~\ref{cor:poly} and Lemma~\ref{lem:poly_lb}, we see that Theorem~\ref{them:ins_dep} is optimal up to logarithmic factors in the sample complexity for polynomial gap instances with $\alpha > 1/2$.

\subsection{Unverifiable \texorpdfstring{$(m,\eps)$}--Good Arm Identification}
We show the error probability of SH for unverifiably identifying an $(m,\eps)$-good arm (defined in Section~\ref{sec-prelim}) and discuss its implications.

\begin{algorithm}[t]
\label{Alg:SH}
\caption{Sequential Halving (SH)}
\textbf{Input:} budget: $T$, arms: $[n]$\\
\textbf{Initialize:} $S_1=[n]$\\

\For{$\ell=1,\dots,\ur{\log_2 n}$}{ 
    Sample each arm $i\in S_\ell$ for $\dsym{T_\ell}$ times where 
    \begin{align*}
       T_\ell=\dr{\fr{T}{|S_\ell|\ur{\log_2 n}}} 
    \end{align*}
    Let $S_{\ell+1}$ be the set of $\ur{S_{\ell}/2}$ arms in $S_{\ell}$ with the largest empirical rewards.
}
Set $J_T$ as the only arm in $S_{\ur{\log_2 n}}$.\\
\textbf{Output:} $J_T$
\end{algorithm}

\begin{theorem}
\label{them:minmax}
For any $m\le n$ and any $\eps\in(0,1)$, the error probability of SH for identifying an $(m,\eps)$-good arm satisfies
\begin{align*}
    &\pp{\mu_{J_T}<\mu_m-\eps}\le\\
    &\log_2n\cd\ep{-\textup{const}\cd m\cd\rbr{\fr{\eps^2T}{4n\log_2^2(2m)\log_2n}-\ln(4e)}}.
\end{align*}
Thus, there exist an absolute constant $c_1>0$ s.t. if $T\ge c_1\fr{\rbr{\ln(4e)+\ln\ln n}n\log_2^2(2m)\log_2n}{\eps^2}=\tilde{\Theta}\rbr{\fr{n}{\eps^2}}$, we have 
\begin{align*}
\pp{\mu_{J_T}<\mu_m-\eps}\le\ep{-\tilde{\Theta}\rbr{m\fr{\eps^2T}{n}}}.
\end{align*}
\end{theorem}
The proof is in Appendix \ref{app-sec-th2}.

\begin{remark}
Note that our bound implies a \textit{verifiable} sample complexity bound for the $(m,\eps)$-good arm identification setup as well because, given $(m,\eps,\delta)$, one can work out the sufficient samples size $T'$ to control the RHS of Theorem~\ref{them:minmax} to be at most $\dt$ and then run SH with  $T'$.
\end{remark}

\textbf{Implications for $m=1$.~}
The worst-case upper bound of Theorem~\ref{them:minmax} for $m=1$ corresponds to the sample complexity of $\hO{\fr{n}{\eps^2}\log\fr{1}{\delta}}$. 
This is the same as the classic lower bound result of the $(\eps,\delta)$-PAC problem \citep{mannor2004sample} up to logarithmic factors.
While numerous algorithms achieve a matching upper bound including \citet{even2006action} and \citet{hassidim2020optimal}, our result in Theorem~\ref{them:ins_dep} indicates a tighter bound when there are many good arms.

We now apply~\eqref{eq:simpleregret-identity} to convert a uniform $\eps$-error probability bound into a simple regret bound.
\begin{corollary}
\label{cor-2}
SH satisfies  $ \SReg_T\le\hO{\sqrt{\fr{n}{T}}}$.
\end{corollary}
The proof is in Appendix \ref{app-sec-sr}.
Note that our bound is minimax optimal up to logarithmic factors~\cite[Exercise 33.1]{lattimore2020bandit}.
To our knowledge, the only algorithm that we are aware of that achieves the minimax simple regret is MOSS~\citep{audibert2009minimax,lattimore16causal}. 
However, as MOSS is designed to minimize cumulative regret, it cannot enjoy low instance-dependent bounds for pure exploration tasks~\citep{bubeck11pure-tcs}.
While uniform sampling can also achieve a near-optimal simple regret, one can show that uniform sampling can be arbitrarily worse than SH w.r.t. the $(\eps=0)$-error probability as its sample assignments over the arms are inherently non-adaptive.
\begin{remark}
In Appendix~\ref{app-sec-1}, we show that one can adjust $T_\ell$ in Algorithm~\ref{Alg:SH} to achieve the minimax optimality up to a constant factor (i.e., no logarithmic factors), but it is not clear if this variant achieves a similar instance-dependent sample complexity as Algorithm~\ref{Alg:SH}.
\end{remark}

\textbf{Implications for $m>1$.~}
We now show that SH enjoys the optimal worst-case guarantee for the unverifiable $(m,\eps)$-good arm identification problem up to logarithmic factors.
Our result is especially attractive since the target number of good arms $m$ is not an input to SH.

\begin{algorithm}[t]
\label{Alg:DSH}
\caption{Doubling Sequential Halving (DSH)}
\textbf{Input} arms: $[n]$\\
\textbf{Initialize} $T_1=\ur{n\log_2n}$, $\hat j = j$ for some arbitrarily chosen $j\in[n]$, $\hat \mu^* = -\infty$. Define the time blocks $\cT_1 = \{1,\ldots, T_1\}$ and $\cT_k = \{T_1 (2^{k-1} - 1) + 1, \ldots, T_1 (2^{k} - 1) \}, \forall k \ge 2$, which satisfies $|\cT_k| = T_1 2^{k-1}$. $k=1$. An instance of SH denoted by $\cS_1$ with budget $|\cT_1|$.\\
\For{$t=1, 2,\dots$}{
    Pull arm $I_t$ according to the recommendation from $\cS_k$\\
    Receive reward $R_t$ and send it to $\cS_k$\\
    \If{$t = \max \cT_k$}{
        Set $(\hat j, \hat\mu^*)$ as the output arm and its empirical mean from the last stage of $\cS_k$\\
        $k \leftarrow k + 1$\\
        Initialize a new instance of SH denoted by $\cS_k$ with budget $|\cT_k|$
    }
    \textbf{Output:} $J_t = \hat j$ and $M_t = \hat\mu^*$.
}
\end{algorithm}

Our upper bound of Theorem~\ref{them:minmax} corresponds to a sample complexity of $\hO{\frac{n}{m\eps^2}\log\rbr{\frac{1}{\delta}}}$, which matches the lower bound in \citet{chaudhuri2017pac}, thereby closing the gap between the upper and lower bounds from previous work; the algorithms therein have suboptimal upper bounds that scale with  $\log^2\rbr{\frac{1}{\delta}}$.
Furthermore, these algorithms all require $(m,\eps,\delta)$ as input, which is natural for obtaining verifiable sample complexities but becomes a limitation in other settings. 
In contrast, Theorem~\ref{them:minmax} allows $(m,\eps)$ to be chosen freely to measure the algorithm's performance since the parameters for measuring the performance are no longer necessary for executing the algorithm.
Therefore, one can rewrite Theorem~\ref{them:minmax} as follows:
\begin{align*}
     \pp{\mu_{J_T}<\mu_1-\eps}\le
     \min_{\cbr{\rbr{m',\eps'}:\Delta_{m'}+\eps'\le\eps}}\ep{-\tilde{\Theta}\rbr{m'\fr{{\eps'}^2T}{n}}} ~.
\end{align*}

\subsection{Anytime version of SH}

To support anytime simple regret minimization, we combine SH with the doubling trick~\citep{jun16anytime}, which we call Doubling SH (DSH). 
Specifically, DSH repeatedly runs SH.
The $k$-th SH receives a doubled budget compared with the $(k-1)$-th SH. Before it finishes $k$-th SH, it always returns the output of $(k-1)$-th SH. 
The pseudocode can be found in {Algorithm \ref{Alg:DSH}}.
The following theorem shows that DSH enjoys the same guarantee as SH up to a constant factor.

\begin{theorem}
\label{them:dsh}
Let $\eps\in(0,1)$.
A single run of DSH satisfies the following $\eps$-error probability 
\begin{align*}
    \pp{\mu_{J_T}<\mu_1-\eps}\le\ep{-\hTT{\fr{T}{H_2(\eps)}}},
\end{align*}
simultaneously for all $T\ge\hTT{\max_{i\ge g(\eps)+1}\frac{i}{\Delta_i^2}}$
\end{theorem}
This is a direct consequence of the fact that at time $t$, the latest finished SH was run with a budget of at least $t/4$. 
The full proof is in Appendix~\ref{app-sec-2}.

\section{Simple Regret in the Data-Poor Regime}
\label{sec-bsh}

Despite the optimality of SH shown in the last section, it requires at least $\hTT{\max_{i\ge g(\eps)+1}\frac{i}{\Delta_i^2}}$ samples, which in the worst case could be $\tilde{\Theta}\rbr{n/\eps^2}$. This requirement is unacceptable when the number of arms $n$ is large (or even infinite) while the fraction of the $\eps$-good arms is kept constant.
This setup is commonly referred to as the data-poor regime, where one would like to return a good arm even before the algorithm samples every arm once.

To cope with the data-poor regime, we take inspiration from BUCB~\citep{Katz20} and propose Bracketing SH (BSH) that enjoys a similar guarantee as DSH but enjoys non-vacuous sample complexity in the data-poor regime.
To understand the bracketing trick, suppose that we uniformly sample a subset of size $n/m$ $(0<m\le n)$ from the entire arm set $[n]$. 
With constant probability, this subset includes at least one of the top $m$ arms. 
By applying any pure exploration algorithm to this subset, we expect to find the best arm within the subset with the sample complexity that scales with $n/m$ rather than $n$.
Such a subset is called a \textit{bracket}. 
Note that there is a natural trade-off here with the bracket size.
If the size is large, we are likely to include many good arms, but the sample complexity of identifying a good arm becomes large because we have a lot of arms to pull at least once.
On the other hand, if the size is small, we are not likely to include any good arm, although the sample complexity of identifying an arm that is good relative to the braacket is small.
As we show later, one can precisely work out such a tradeoff mathematically and find the best bracket size.
The challenge is, however, that the best bracket size typically requires knowledge of the number of good arms.

To avoid requiring such knowledge, BSH adopts the bracketing technique of~\citet{Katz20} that progressively creates a larger and larger bracket size as the time step $t$ gets larger while invoking a base algorithm for each bracket by cycling through them (i.e., Round Robin).
Specifically, BSH uses DSH as the base algorithm.
At each time step, BSH takes in the estimated best arms and their empirical means from all the brackets and outputs the one with the largest empirical mean.

We summarize BSH in Algorithm~\ref{Alg:BSH} where the operator $U\rbr{[n],k}$ samples $k$ items with replacement from $[n]$ (uniformly at random). 
If $k\ge n$, $U([n],k)$ returns $[n]$.
To define terminology, we say a new bracket is \textit{opened} when a new bracket is sampled.
We set the bracket-opening schedule such that the $B$-th bracket is opened at time step $(B-1)\cdot2^{B-1}$. 
The arm pulls between the time step $(B-1)\cdot2^{B-1}$ and $B\cdot2^{B}$ are equally allocated to the opened brackets (total $B$ of them). 
We say an arm \textit{represents} bracket $A_B$ at time step $t$ if it is the output returned by the DSH on bracket $A_B$ at time step $t$.

\begin{algorithm}[t]
\label{Alg:BSH}
\caption{Bracketing SH (BSH)}
\textbf{Input:} arms: $[n]$\\
\textbf{Initialize:} $t=1$, $B=0$, $b_1 = 1$. Define $(I(\cD), J(\cD), M(\cD))$ to be the arm to be pulled next, the current estimated best arm, and its empirical mean from an algorithm $\cD$, respectively.\\
\For{$t=1, 2, 3\dots$}{
    \If{$t \ge B 2^B$}{
        $B \leftarrow B + 1$\\
        Sample a bracket $A_{B}=U\rbr{[n],n\vee2^{B}}$\\
        Initialize a new instance of DSH, denoted by $\cD_B$, with the bracket $A_B$.
    }
    Pull arm $I_t = I(\cD_{b_t})$\\
    Receive a reward $R_t$ from the environment and send $R_t$ to $\cD_{b_t}$\\
    Output $J_t = J(\cD_{\hat b})$ where $\hat b = \arg \max_{b\in[B]} M(\cD_b)$\\
    $b_{t+1} = \begin{cases}
      b_t+1 & \text{if $b_t < B$}\\
      1   & \text{otherwise}
    \end{cases}$  
}
\end{algorithm}

Let us first show the properties of the bracketing technique.
The bracketing design ensures the diversity of the size of opened brackets. 
Specifically, the smallest bracket has $2$ arms, and the largest bracket has $\hTT{t}$ arms. 
Moreover, it also ensures that all the opened brackets have received an (order-wise) equal amount of sampling budget at any time.

We now present the $\eps$-error probability bound of BSH.
\begin{theorem} 
\label{them:datapoor}
Let $\eps\in(0,1)$.
A single run of BSH satisfies the following $\eps$-error probability 
\begin{align*}
    \pp{\mu_1-\mu_{J_t}>\eps}\le\ep{-\hTT{\fr{t}{\max\cbr{H_2\rbr{\fr{\eps}{2}},\frac{1}{\eps^2}}}}},
\end{align*}
simultaneously for all $t\ge\hTT{H_2\rbr{\fr{\eps}{2}}}$.
\end{theorem}
The proof is in Appendix \ref{app-sec-them_datapoor}. 
The key difference between Theorem~\ref{them:dsh} and Theorem \ref{them:datapoor} is the latter does not require the number of samples to be at least $T\ge\hTT{\max_{i\ge g(\eps)+1}\frac{i}{\Delta_i^2}}$.

To show the effectiveness of Theorem \ref{them:datapoor}, we now discuss the implication of our result and compare it with BUCB by~\citet{Katz20} using their performance measure called $(\eps,\delta)$-unverifiable sample complexity.
\begin{definition}[$(\eps,\delta)$-unverifiable sample complexity] \label{def-uver} \citep{Katz20}. For an algorithm $\pi$ and an instance $\rho$. Let $\tau_{\eps,\delta}$ be a stopping time such that
\begin{align*}
    \pp{\forall t\ge\tau_{\eps,\delta}: \mu_{J_t}>\mu_1-\eps}\ge1-\delta.
\end{align*}
Then $\tau_{\eps,\delta}$ is called $(\eps,\delta)$-unverifiable sample complexity of the algorithm with respect to $\rho$.
\end{definition}
The $(\eps,\delta)$-unverifiable sample complexity indicates how many samples are sufficient for an algorithm to begin to output an $\eps$-good arm with high probability. Compared with the verifiable sample
complexity, it does not require the algorithm to verify the output is $\eps$-good. More discussion can be found in \citet{Katz20}. 
While the verbatim requirement of $(\eps,\delta)$-unverifiable sample complexity is stronger than the uniform $\eps$-error probability bound, we show that the latter implies the former under a mild assumption in Appendix~\ref{sec-unverifiable-sc}.

To make explicit comparisons, we turn to specific problem instances and show the $(\eps,\delta)$-unverifiable sample complexity bounds for BUCB \citep[Theorem 7]{Katz20} and BSH. 
\begin{corollary}
\label{cor-5}
Consider the EqualGap$(m)$ instance.
BUCB achieves an expected $(\eps,\delta)$-unverifiable sample complexity as
\begin{align*}
    \ee{\tau_{\eps,\delta}}\le\hO{\fr{n}{\eps^2m}\log\rbr{\fr{1}{\delta}}}. 
\end{align*}
BSH achieves an $(\eps,\delta)$-unverifiable sample complexity as, with probability $1-\delta$,
\begin{align*}
   \tau_{\eps,\delta}\le\hO{\fr{n}{\eps^2m}\log\rbr{\fr{1}{\delta}}}.
\end{align*}
\end{corollary}
The proof is in Appendix \ref{app-sec-bucb_equal}.
In this instance, BUCB and BSH have nearly the same performance guarantees except that BUCB's $(\eps,\delta)$-unverifiable sample complexity is stated as an expected sample complexity.
Interestingly, \citet[Section E]{Katz20} remarks that their attempt to derive a high-probability bound resulted in the factor of $\ln^2(1/\dt)$, which we believe is due to the suboptimal bound (i.e., no accelerated rates) of their choice of the base algorithm. 

Scaling with $n/m$ instead of $n$ reveals the merit of the $(\eps,\delta)$-unverifiable sample complexity that was not achieved by the algorithms designed to achieve verifiable sample complexity such as Median Elimination~\citep{even2006action}.
\begin{corollary}
\label{cor-4}
Consider the Polynomial$(\alpha)$ instance; i.e., $\Delta_i=\rbr{\fr{i}{n}}^\alpha$ with $\alpha>0.5$. For any $\eps\in(0,1)$, let $\hat{\tau}_{\eps,\delta}$ be the upper bound of the expected $(\eps,\delta)$-unverifiable sample complexity reported in \citet[Theorem 7]{Katz20} for BUCB. 
Then, BUCB satisfies
\begin{align*} \ee{\tau_{\eps,\delta}}\le\hat{\tau}_{\eps,\delta}=\hTT{\eps^{-\fr{2\alpha-1}{\alpha}}n\log\rbr{\fr{1}{\delta}}}. 
\end{align*}
On the other hand, BSH satisfies, with probability $1-\delta$,
\begin{align*}
    \tau_{\eps,\delta}\le\hO{\eps^{-2}\log\rbr{\fr{1}{\delta}}}~.
\end{align*}
\end{corollary}
The proof is in Appendix \ref{app-sec-bucb_poly}.
The $(\eps,\delta)$-unverifiable sample complexity of BSH for the Polynomial$(\alpha)$ instance does not scale with the the number of arms $n$ polynomially, unlike BUCB. 
For large-scale instances, BSH provides a much stronger guarantee than BUCB. 
Even when $n$ is not too large, as long as $\Delta_2\le\eps$ (which implies $\eps^{\frac{1}{\alpha}}\le n$), BSH still has a better sample complexity bound than BUCB.

\section{Experiments}
\label{sec-expr}
We test Bracketing SH on a real-world dataset called The New Yorker Cartoon Caption Contest~\citep{jain2020new}.
For each cartoon, the editors of The New Yorker collect the evaluation score of $n$ captions from the participants.
Specifically, upon arrival of a participant, the algorithm sequentially shows a number of captions and receives the evaluation score of ``unfunny'', ``somewhat funny'', or ``funny''. 
Following \citet{Katz20}, we use the proportion of times the score was ``somewhat funny'' or ``funny'' as the ground truth mean reward for each caption.
We then set the reward distribution as the Bernoulli.
We choose contests 780, 781, and 782, which contain 6509 arms, 5969 arms, and 4389 arms, respectively. 
As we are especially interested in the data-poor regime, we report the performance of the algorithms up to time step 10,000, which is only about twice larger than the number of arms in our dataset.
We ran BSH and BUCB with various $\delta$ choices to see the best version of BUCB, which was repeated 500 times with a different random seed.
We used a desktop with AMD Ryzen 5 PRO 4650GE CPU and 16GB RAM to conduct the experiment, which took two hours to produce each plot. 
We summarize the results in Figure \ref{fig-1}, where BSH is a clear winner over BUCB.
Specifically, in contest 781, at time step $4,000$, Bracketing UCB achieves simple regret about $0.18$ for $\delta=0.2$, while BSH achieves simple regret about $0.11$, amounting to $39\%$ improvement.
We explain more on the implementation in Appendix~\ref{sec-discussion-implementation}.

\begin{figure}[t]
\centering
\minipage{0.32\textwidth}
  \includegraphics[width=\linewidth]{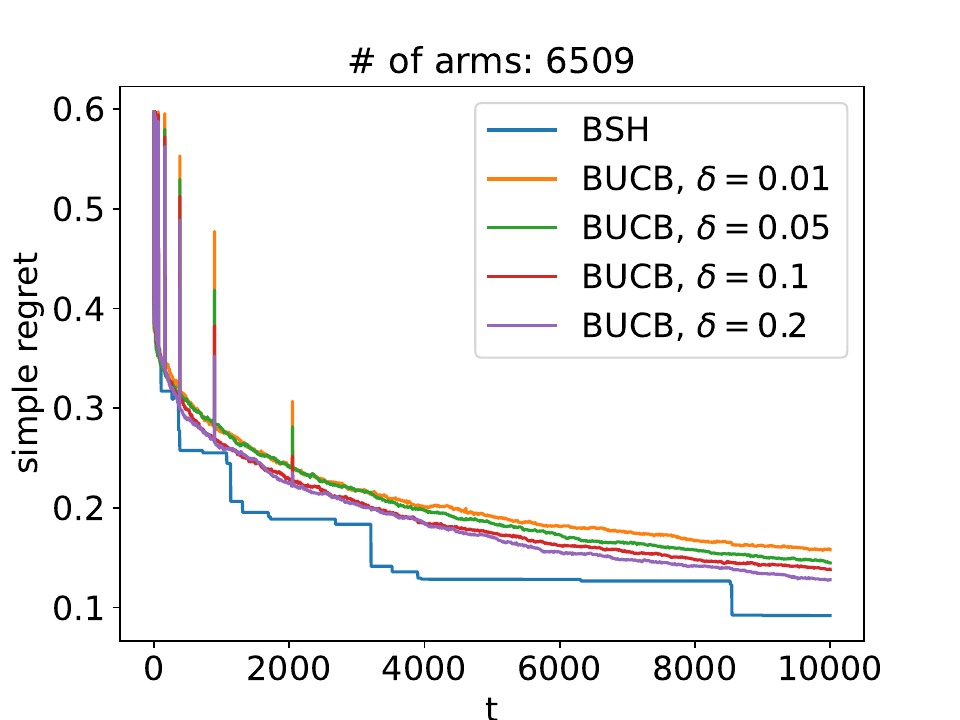}
\endminipage\hfill
\minipage{0.32\textwidth}
  \includegraphics[width=\linewidth]{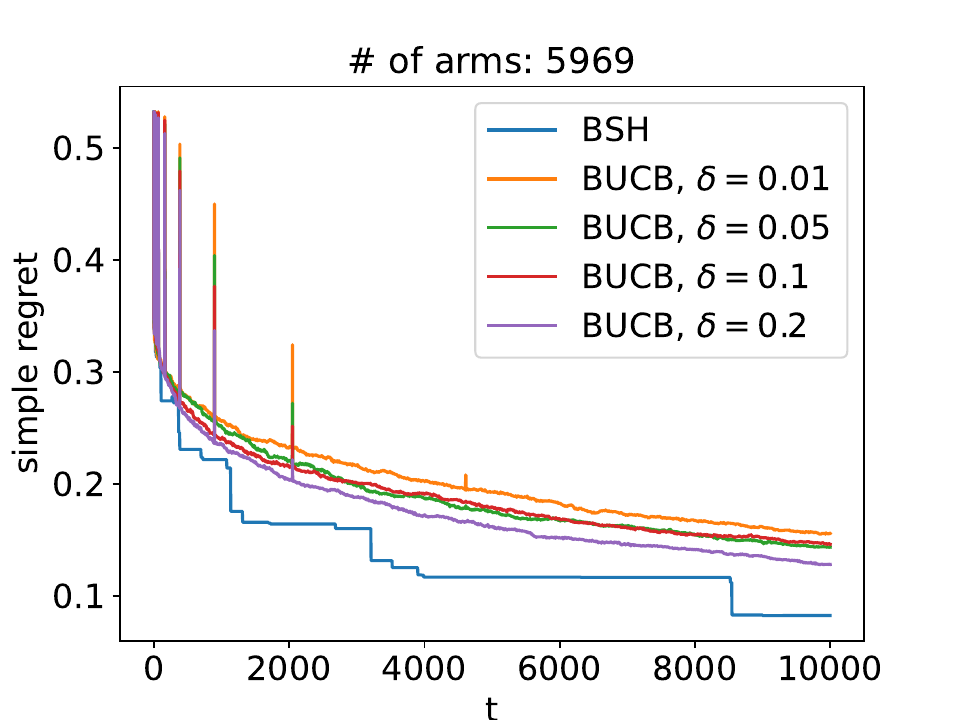}
\endminipage\hfill
\minipage{0.32\textwidth}%
  \includegraphics[width=\linewidth]{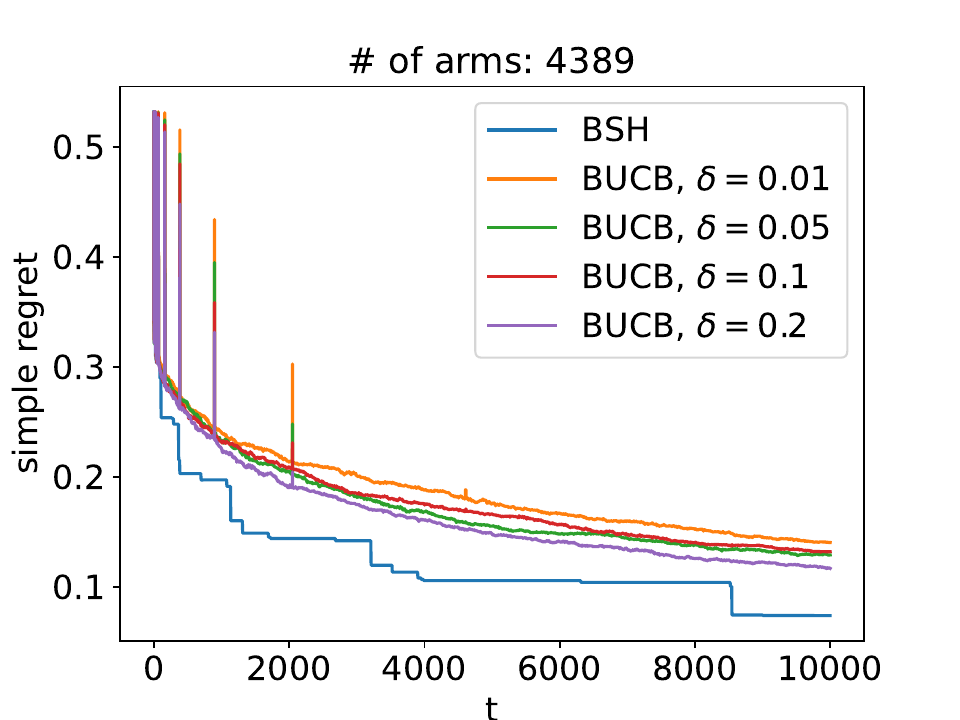}
\endminipage
\caption{The simple regret comparison between Bracketing UCB and Bracketing SH for the New Yorker Cartoon Caption contest 780, 781, 782. }
\label{fig-1}
\end{figure}

\section{Conclusion}
\label{sec-future}

Obtaining a uniform $\eps$-error probability bound is precisely a way to characterize the distribution function of $\Delta_{J_T}$ induced by a particular algorithm, based on which we believe that a uniform $\eps$-error probability bound is a fundamental quantity for any measures for returning an $\eps$-good arm.

As such, we believe our paper opens up numerous exciting open problems. 
First, SH does not achieve an optimal worst-case sample complexity of $\Theta(\frac{n}{\eps^2}\ln(1/\dt))$ (due to logarithmic factors).
While adjusting its sample allocation scheme does achieve the optimal rate as we report in the appendix, it does not seem to achieve the usual instance-dependent sample complexity of $\tilde\Theta(\sum_i (\eps \vee \Delta_i)^{-2})$ achieved by many existing algorithms.
We wonder if a modified SH or other algorithms can simultaneously be optimal for both sample complexity measures.
Second, another practical pure exploration algorithm is track-and-stop by~\citet{garivier16optimal} and their variants.
They are also parameter-free if we only take its sampling strategy and not the stopping strategy.
It would be interesting to investigate if track-and-stop achieves similar near-optimal guarantees as DSH.

\subsection*{Acknowledgements}

We thank Jaehyeok Shin for providing help in the initial development of the improved bound for SH. %

\putbib[library-shared,ref]
\end{bibunit}


%
%
%
%
%





\clearpage


%
  

\newpage
\appendix
\onecolumn

\part{Appendix}
\parttoc
\begin{bibunit}[plainnat]

\section{Related work}

\textbf{Pure exploration.} Pure exploration has a broad position in several closely related research directions, including multi-armed bandits \citep{audibert10,karnin2013almost,Carpentier2016,aziz2018pure,karnin2013almost,jamieson2014,garivier2016,chen2017towards,Simchowitz2017a,hassidim2020optimal} and reinforcement learning \citep{even2006action,azar2017minimax,kaufmann2021adaptive}. 
Even more, the Monte-Carlo tree search with tree depth $1$ is also a pure exploration problem \citep{kocsis2006bandit}. 
Even if we only consider multi-armed bandits, there are various types of problems that have been formulated. 
Our investigation focuses on the standard $K$-armed bandits model. But we note here pure exploration has also been studied in structured bandits like kernel bandits \citep{camilleri2021high} and linear bandits \citep{soare2014best,zhu2022near}.

Started from \citet{even2006action,mannor2004sample}, pure exploration in $K$-armed bandits is studied with the celebrated $(\eps,\delta)$-PAC framework. 
\citet{mannor2004sample} show a lower bound, which matches the upper bound of Median Elimination in \citet{even2006action}. 
The recent work by \citet{hassidim2020optimal} further shows
that it is possible to achieve $\fr{n}{2\eps^2}\log\fr{1}{\delta}$ asymptotically, where the constant factor of $1/2$ is the optimal one. 
\citet{aziz2018pure} study $(\eps,\delta)$-PAC for the infinitely-armed bandits model, where they model the infinitely-armed bandits setting as the arm reservoir. 
Among them, \citet{even2006action,hassidim2020optimal} focus on the worst-case guarantee. \citet{aziz2018pure} provide more instance-dependent bounds.

\textbf{Best arm identification.} 
The best arm identification is the factually dominating setting for studying the pure exploration problem. 
Specifically, there are two problem setups of the best arm identification, fixed budget \citep{audibert10,karnin2013almost,Carpentier2016} and fixed confidence \citep{karnin2013almost,jamieson2014,garivier2016,chen2017towards,Simchowitz2017a}. 
In the fixed budget setting, the algorithm takes a budget as the input and is required to return an output before exhausting all the sample budget. The paper by \citet{audibert10} is the first one to open up this direction and proposes Successive Rejects algorithm, which achieves the optimality up to logarithmic factors. 
The follow-up work of \citet{karnin2013almost} proposes SH algorithm, which is also an elimination-based algorithm as Successive Reject, but it eliminates empirically bad arms more aggressively. 
It was not until \citet{Carpentier2016} SH was proven to be optimal up to doubly-logarithmic factors.

Many meaningful results have been developed for the fixed confidence setting.
To our knowledge, the paper by \citet{garivier2016} is the first one that claims asymptotic optimality. 
They propose a non-asymptotic lower bound for the sample complexity and also an algorithm called Track-and-Stop that matches the lower bound asymptotically as $\delta$ goes to $0$.

\textbf{Simple regret.} 
In our opinion, simple regret is the least understood measure for pure exploration. Though many works claim simple regret as their target, their performance bounds are either on best arm misidentification or $(\eps,\delta)$-PAC identification. Of course, by \citet{audibert10}, one can always use the probability of best arm misidentification to upper bound simple regret. However, in the data-poor regime, i.e., $t<n$, the probability of best arm misidentification is vacuous since we cannot even guarantee each arm has been sampled at least once.
For $(\eps,\delta)$-PAC identification, the algorithm usually requires $\eps$ as the input of the algorithm.
However, simple regret itself is a parameter-free performance measure. Perhaps only \citet{carpentier2015simple} truly deals with simple regret directly.
The algorithm of \citet{carpentier2015simple} works for the simple regret minimization problem since it does not take such a predefined threshold $\eps$ as the input. There is another different problem setup in which the algorithm is required to output one of the top-$m$ arms with $\eps$ slackness  \citep{chaudhuri2017pac,chaudhuri2019pac}.
Though such a problem setup also characterizes how close the output is to the best one, it still needs $m$ and $\eps$ as the input of the algorithm. Recent work by \citet{de2021bandits} proposes an algorithm that guarantees that it returns one of the $m$ best arms with the same mean reward without knowledge of $m$ as input to the algorithm.
The algorithm of \citet{de2021bandits} works for both finite and infinite arms settings. However, in light of its limited problem setting that requires the top $m$ arms to have the same mean reward, it is unclear to us if the result can be extended to guarantee simple regret or $\eps$-error probability bound as we do. 
The model of our paper can be regarded as a generalization of \citet{de2021bandits} because all the results of \citet{de2021bandits} can be matched with our results up to logarithmic factors if we consider their setting, but our paper does not require the identity among the optimal arms. Furthermore, we provide analysis on $\eps$-error probability, a generalization of misidentification probability, and analysis on simple regret.

\textbf{$(\eps,\delta)$-Verifiable and unverifiable sample complexity.} The fixed confidence best arm identification shares some similarities with the $(\eps,\delta)$-PAC identification. For both settings, the algorithm is required to verify the output can meet the corresponding criteria, simple regret of $0$ or simple regret of $\eps$, with error probability $\delta$. Verifiability requires the algorithm to pull each arm at least once. Therefore the $(\eps,\delta)$-PAC is inherently impossible for the data-poor regime, e.g., $t<n$. \citet{Katz20} argue $(\eps,\delta)$-PAC is not aligned with practical usage. They propose the $(\eps,\delta)$-unverifiable sample complexity, which does not require the algorithm to verify the output is $\eps$-good. Instead, the $(\eps,\delta)$-unverifiable sample complexity represents the number of samplings that the algorithm minimally needs such that it has the ability to output an $\eps$-good arm. Note the sample complexity of a fixed budget algorithm (or, to be more accurate, the doubling trick version of the fixed budget algorithm) also captures the nature of the $(\eps,\delta)$-unverifiable sample complexity. In the reinforcement learning setting, \citet{dann2017unifying} also discuss the limitation of $(\eps,\delta)$-PAC and propose a notion called Uniform-PAC.

\textbf{Top-$k$ arm identification.}
The top-$k$ arm identification problem requires the algorithm to return $k$ arms with the highest mean reward instead of only the best one \citep{kalyanakrishnan2010efficient,kalyanakrishnan2012pac,chen2017nearly,Simchowitz2017a}.
The seminal work is \citet{kalyanakrishnan2010efficient}. Their problem formulation is a direct extension of the $(\eps,\delta)$-PAC identification problem. The goal is to return $k$ arms that have mean rewards no less than $\mu_k-\eps$. 
The worst-case lower bound is shown in \citet{kalyanakrishnan2012pac}. The first instance-dependent lower bound for the $k$-identification problem is given by \citet{Simchowitz2017a} and \citet{chen2017nearly} almost at the same time, together with the near-optimal algorithms. The best arm identification can be regarded as a special case of the top-$k$ arm identification problem. The similarity is that we are both interested in good arms, but we are measuring the likelihood of identifying one arm out of the top-$k$ rather than the set of top-$k$ arms.

\section{Improved Analyses of Sequential Halving}

\subsection{Implication of Theorem \ref{them:minmax} for \texorpdfstring{$m=1$}-}

The following theorem corresponds to the implication of $m=1$ of Theorem \ref{them:minmax}. However, we prove it independently here.
\begin{theorem}
\label{them-m1}
For any $\eps\in(0,1)$, the error probability of SH for identifying an $\eps$-good arm satisfies,
\begin{align*}
\pp{\mu_{J_T}<\mu_1-\eps}\le3\log_2n\cd\exp\rbr{-\fr{\eps^2}{32}\fr{T}{n\log_2n}}.
\end{align*}
\end{theorem}
\begin{proof}
To avoid redundancy and for the sake of readability, we assume $n$ is of a power of 2. It is easy to verify the result for any $n$. Let $\eps_1=\eps/4$, $T':=\fr{T}{\log_2 n}$. And define $\eps_{\ell+1}=\fr{3}{4}\cdot\eps_\ell$.
For each stage $\ell$, define the event $G_\ell$ as 
\begin{align*}
    G_\ell:=\cbr{\max_{i\in S_{\ell+1}}\mu_i\ge\max_{i\in S_\ell}\mu_i-\eps_\ell}.    
\end{align*}
Thus as long as $\bigcap_{\ell=1}^{\log_2 n}G_\ell$ happens, we have that the arm returned after the final stage is an $\eps$-good arm, because
\begin{align*}
    \sum_{\ell=1}^{\log_2 n}\eps_\ell<\sum_{\ell=1}^{\infty}\rbr{\fr{3}{4}}^{\ell-1}\cdot\eps_1=\fr{\eps}{4}\sum_{\ell=1}^{\infty}\rbr{\fr{3}{4}}^{\ell-1}\le\fr{\eps}{4}\lim_{n\rightarrow\infty}\fr{1-(3/4)^n}{1-3/4}=\eps.
\end{align*}
Thus, by a union bound,
\begin{align}
\pp{\mu_{J_T}<\mu_1-\eps}\le&\pp{\rbr{\bigcap_{\ell=1}^{\log_2 n}G_\ell}^c} \nn\\
\le&\sum_{\ell=1}^{\log_2n}\pp{G_\ell^c} \label{eq-2}.
\end{align}
Let $a_\ell$ be the best arm in $S_\ell$,
\begin{align*}
    \pp{G^c_\ell}&=\pp{G^c_\ell,\hat{\mu}_{a_\ell}<\mu_{a_\ell}-\eps_{\ell}/2}+\pp{G^c_\ell,\hat{\mu}_{a_\ell}\ge\mu_{a_\ell}-\eps_{\ell}/2}\\
    &\le\pp{\hat{\mu}_{a_\ell}<\mu_{a_\ell}-\eps_{\ell}/2}+\pp{G^c_\ell \mid \hat{\mu}_{a_\ell}\ge\mu_{a_\ell}-\eps_{\ell}/2}.\\
    &\le\ep{-\fr{\eps^2_{\ell}}{2}\fr{T'}{|S_\ell|}}+\pp{G^c_\ell \mid \hat{\mu}_{a_\ell}\ge\mu_{a_\ell}-\eps_{\ell}/2}.
\end{align*}
For the second term, 
\begin{align*}
    \pp{G^c_\ell \mid \hat{\mu}_{a_\ell}\ge\mu_{a_\ell}-\eps_{\ell}/2}&\le\pp{\abr{\{i\in S_\ell \mid \hat{\mu}_i>\mu_{i}+\eps_{\ell}/2\}}\ge |S_\ell|/2}\\
    &\stackrel{(a_1)}{\le}\fr{\ee{\abr{\{i\in S_\ell \mid \hat{\mu}_i>\mu_{i}+\eps_{\ell}/2\}}}}{|S_\ell|/2}\\
    &\le\fr{|S_\ell|\exp\rbr{-\fr{\eps^2_{\ell}}{2}\fr{T'}{|S_\ell|}}}{|S_\ell|/2}\\
    &=2\exp\rbr{-\fr{\eps^2_{\ell}}{2}\fr{T'}{|S_\ell|}}.
\end{align*}
For $(a_1)$, we use Markov's inequality. Then,
\begin{align*}
    \pp{G^c_\ell}\le3\exp\rbr{-\fr{\eps^2_{\ell}}{2}\fr{T'}{|S_\ell|}}=3\exp\rbr{-\rbr{\fr{9}{16}}^{\ell-1}\fr{\eps^2}{32}\fr{T'}{2^{-(\ell-1)}n}}=3\exp\rbr{-\rbr{\fr{9}{8}}^{\ell-1}\fr{\eps^2}{32}\fr{T'}{n}}.
\end{align*}
Taking the above into \eqref{eq-2}, we have
\begin{align*}
\pp{\mu_{J_T}<\mu_1-\eps}\le&\sum_{\ell=1}^{\log_2 n}\pp{G_\ell^c}\\
\le&\sum_{\ell=1}^{\log_2 n}3\exp\rbr{-\rbr{\fr{9}{8}}^{\ell-1}\fr{\eps^2}{32}\fr{T'}{n}}\\
\le&3\log_2n\cd\exp\rbr{-\fr{\eps^2}{32}\fr{T}{n\log_2n}}.
\end{align*}
\end{proof}

\subsection{Proof of Corollary \ref{cor-2}}
\label{app-sec-sr}
We state the following result that includes two different budget allocation schemes for SH. The appendix \ref{app-sec-1} should be read first where we define the two budget allocation schemes:
\begin{align*}
    \text{Option 1:}\;T_\ell=\dr{\fr{T}{|S_\ell|\ur{\log_2 n}}},\quad\text{Option 2:}\;T_\ell=\dr{\fr{T}{81n}\cdot\rbr{\frac{16}{9}}^{\ell-1}\cdot \ell}.
\end{align*}
\newtheorem*{cor-2}{Corollary~\ref{cor-2}}
\begin{cor-2}
The simple regret of SH satisfies
\begin{align*}
    \text{Option 1:}\;\SReg_T\le\OO{\sqrt{\fr{n\log^3 n}{T}}},\quad\text{Option 2:}\;\SReg_T\le\OO{\sqrt{\fr{n}{T}}}.
\end{align*}
\end{cor-2}
\begin{proof}
The simple regret can be calculated by considering the following integral with respect to $\eps$. We use SH with Option $2$. For Option $1$, the same analysis holds.
\begin{align*}
    \SReg_T=&\int_{0}^{\infty}\pp{\mu_1-\mu_{J_T}>\eps} \dif\eps\\
    =&\int_{0}^{\infty}\rbr{1\wedge\exp\rbr{-\eps^2\fr{T}{n}}}\dif\eps\\
    \le&\int_{0}^{\sqrt{\fr{n}{T}}}1\dif\eps+\int_{0}^{\infty}\exp\rbr{-\eps^2\fr{T}{n}}\dif\eps\\
    \le&\sqrt{\fr{n}{T}}+\int_{0}^{\infty}\exp\rbr{-\eps^2\fr{T}{n}}\dif\eps.
\end{align*}
We can borrow the result from Gaussian distribution with mean $0$ and variance $\sigma^2$,
\begin{align*}
    \int_{0}^{\infty}\fr{1}{\sigma\sqrt{2\pi}}\exp\rbr{-\fr{x^2}{2\sigma^2}}\dif x=\fr{1}{2}.
\end{align*}
Taking $\fr{1}{2\sigma^2}=\fr{T}{n}$, we have
\begin{align*}
    \int_{0}^{\infty}\exp\rbr{-\eps^2\fr{T}{n}}\dif\eps=\fr{\sigma\sqrt{2\pi}}{2}=\fr{\sqrt{\pi}}{2}\sqrt{\fr{n}{T}}.
\end{align*}
Thus, we have
\begin{align*}
    \SReg_T\le\OO{\sqrt{\fr{n}{T}}}.
\end{align*}
\end{proof}

\subsection{Proof of Theorem~\ref{them:minmax}}
\label{app-sec-th2}

Theorem~\ref{them:minmax} is an equivalent result to Theorem~\ref{them:otherEpsBound}, so the proof of one implies the other.
While Theorem~\ref{them:otherEpsBound} takes a form that is easier to instantiate the bound for specific instances, Theorem~\ref{them:minmax} takes a form that is easier to prove.
Thus, we choose to prove Theorem~\ref{them:minmax} directly, leaving the proof of Theorem~\ref{them:otherEpsBound} as a consequence of Theorem~\ref{them:minmax}.

Let us first provide an intuitive explanation. 
Imagine one possible `typical' scenario where, in each stage, the set of surviving arms for the next stage happens to maintain the fraction of good arms as at least $m/n$. 
This means that, after $\Theta(\log_2 m)$ stages, we expect to have at least one good arm in the surviving arm set.
The number of surviving arms at that time is around $n/m$. 
At this point, the rest of the procedure can be analyzed by the standard analysis of SH where the goal is to find the arm that is the best in the current surviving arm set.
Thus, it remains to bound how likely it is to have the `typical' event above.
That is, we would like to bound the failure probability of this event at each stage as tightly as possible.
Bounding this failure probability is our key technical innovation (Proposition~\ref{prop-1}), which promotes the fast rate of the failure probability that improves as a function of the number of good arms. Another key technical step of the proof is a careful design of events that would guarantee the chosen arm's suboptimality in the last stage to be at most $\eps$, which requires splitting $\eps$ into smaller pieces to be distributed over the stages.

\newtheorem*{cor-3}{Theorem~\ref{them:minmax}}
\begin{cor-3}
For any $m\le n$ and any $\eps\in(0,1)$, the error probability of SH for identifying an $(m,\eps)$-good arm satisfies
\begin{align*}
    \pp{\mu_{J_T}<\mu_m-\eps}\le&\log_2n\cd\ep{-\textup{const}\cd m\rbr{\fr{\eps^2T}{4n\log_2^2(2m)\log_2n}-\ln(4e)}}.
\end{align*}
Thus, there exist a positive constant $c_1$ such that for $T\ge c_1\fr{\rbr{\ln(4e)+\ln\ln n}n\log_2^2(2m)\log_2n}{\eps^2}=\tilde{\Theta}\rbr{\fr{n}{\eps^2}}$,
\begin{align*}
\pp{\mu_{J_T}<\mu_m-\eps}\le\ep{-\tilde{\Theta}\rbr{m\fr{\eps^2T}{n}}},
\end{align*}
where, $\tilde\Theta$ means ignoring the logarithmic factors of $m,n$ and constants.
\end{cor-3}
\begin{proof}
Let us consider the case of $m \le n/2$ first.

Let $\ell^*=\ur{\log_2m}$, $\eps'=\fr{\eps}{2\log_2m}$, $T':=\ur{\fr{T}{\log_2 n}}$. 
And $g_\ell$ denotes the number of $(m,\ell\cd\eps')$-good arms after finishing stage $\ell\in\sbr{\ur{\log_2 m}}$.  
For stage $\ell$, we define the event $G_\ell$ as, 
\begin{align*}
    G_\ell:=\cbr{g_\ell\ge2^{-\ell}\cd m}.
\end{align*}
Specifically, we have $G_{\ell^*}$ to be the event that the number of $(m,\eps/2)$-good arms after finishing stage $\ell^*$ is at least $1$. We define $G_{\ell^*+1}$ as the algorithm succeeds to return an arm in $\text{Top}_m(\eps)$,
\begin{align*}
    G_{\ell^*+1}:=\cbr{\mu_{J_T}\ge\mu_m-\eps}.
\end{align*}
Then the event $\bigcap_{\ell=1}^{\ell^*+1}G_\ell$ is a possible path the algorithm returns an arm in $\text{Top}_m(\eps)$ in the end. Thus the probability of missing all of the $(m,\eps)$-good arms can be upper bounded as follows. Define $F_\ell=\rbr{\bigcap_{i=1}^{\ell-1}G_i}\cap G^c_\ell$ for $\ell>1$, $F_1=G^c_1$. Since
$\bigcap_{\ell=1}^{\ell^*+1}G_\ell\subset G_{\ell^*+1}$ implies $G^c_{\ell^*+1}=\cbr{\mu_{J_T}<\mu_m-\eps}\subset\rbr{\bigcap_{\ell=1}^{\ell^*+1}G_\ell}^c$,
\begin{align}
    \pp{\mu_{J_T}<\mu_m-\eps}\le&\pp{\rbr{\bigcap_{\ell=1}^{\ell^*+1}G_\ell}^c} \nn\\
    =&\pp{\bigcup_{\ell=1}^{\ell^*+1}F_\ell} \nn\\
    \stackrel{(a_1)}{\le}&\sum_{\ell=1}^{\ell^*+1}\pp{G^c_\ell \mid \bigcap_{i=1}^{\ell-1}G_i} \nn\\
    =&\sum_{\ell=1}^{\ell^*}\pp{G^c_\ell\mid \bigcap_{i=1}^{\ell-1}G_i}+\pp{G^c_{\ell^*+1}\mid \bigcap_{i=1}^{\ell^*}G_i} \label{eq-1}.
\end{align}
For $(a_1)$, we use $\PP(A,B) \le \PP(A|B)$.

For the first term of \eqref{eq-1}, we apply the result of Proposition \ref{prop-1} to each stage of SH with the parameters therein as $k=2^{-\ell}\cd m, s=2^{-\ell}\cd n$, $m'=2^{-\ell+1}\cd m$, $\eps = \eps'$, and $T = T'$.
Thus, for stage $\ell$,
\begin{align}
    &\pp{G^c_\ell \mid G_{\ell-1}} \notag\\
    &\le \cc{2^{-\ell+1}\cd m}{2^{-\ell}\cd m} \ep{-2^{-\ell}\cd m\fr{\eps'^2T'}{2\cd2^{-\rbr{\ell-1}}\cd n}} +        
    \cc{2^{-\ell+1}\cd n}{2^{-\ell}\cd(n-m)}
    \ep{-2^{-\ell}\cd(n-m)\fr{\eps'^2T'}{2\cd2^{-\rbr{\ell-1}}\cd n}} \nn\\
    &\stackrel{(a_3)}{\le} h_{\ell,1}\ep{-2^{-\ell}\cd m\fr{\eps'^2T'}{2\cd2^{-\rbr{\ell-1}}\cd n}}+h_{\ell,2}\ep{-2^{-\ell}\cd(n-m)\fr{\eps'^2T'}{2\cd2^{-\rbr{\ell-1}}\cd n}} \nn\\
    &\le h_{\ell,1}\ep{-\fr{m}{2}\fr{\eps'^2T'}{2 n}}+h_{\ell,2}\ep{-\fr{n-m}{2}\fr{\eps'^2T'}{2 n}} \nn\\
    &\stackrel{(a_4)}{\le} h_{\ell,1}\ep{-\fr{m}{2}\fr{\eps'^2T'}{2 n}}+h_{\ell,2}\ep{-\fr{n}{4}\fr{\eps'^2T'}{2 n}} \label{eq-5}\\
    &\le \ep{-\fr{m}{2}\fr{\eps^2T}{2 n\log_2^2m\log_2n}+\ln(2e)2^{-\ell}\cd m}\nn\\
    &\quad+\ep{-\fr{n}{4}\fr{\eps^2T}{2n\log_2^2m\log_2n}+\ln(4e)2^{-\ell}\cd (n-m)} \nn\\
    &\le\ep{-m\rbr{\fr{\eps^2T}{4n\log_2^2m\log_2n}-2\ln(4e)2^{-\ell}}}\nn\\
    &\quad +\ep{-\fr{n}{2}\rbr{\fr{\eps^2T}{4n\log_2^2m\log_2n}-2\ln(4e)2^{-\ell}}} \nn\\
    &\le 2\ep{-m\rbr{\fr{\eps^2T}{4n\log_2^2m\log_2n}-2\ln(4e)2^{-\ell}}} \nn
\end{align}
where $(a_3)$ is $     \cc{2^{-\ell+1}\cd m}{2^{-\ell}\cd m}\le(2e)^{2^{-\ell}\cd m} =: h_{\ell,1}
$ (Sterling's formula $\cc{x}{y}\le\rbr{\fr{ex}{y}}^y$) and $
    \cc{2^{-\ell+1}\cd n}{2^{-\ell}\cd(n-m)}\le(2en/(n-m))^{2^{-\ell}\cd(n-m)}\le(4e)^{2^{-\ell}\cd n} =: h_{\ell,2}$ and $(a_4)$ is by the assumption $m\le n/2$.
    
Then we have,
\begin{align}
    \sum_{\ell=1}^{\ell^*}\pp{G^c_\ell\mid G_{\ell-1}}\le\text{const}\cd\log_2m\cd\ep{-m\rbr{\fr{\eps^2T}{4n\log_2^2m\log_2n}-\ln(4e)}}. \label{eq-3} 
\end{align}
For the second term of \eqref{eq-1}, $G^c_{\ell^*+1}\mid G_{\ell^*}$ indicates the events where SH fails to return an $(m,\eps)$-good arm given the event that there is at least one $(m,\eps/2)$-good arm after finishing stage $\ell^*$. Note the stages from $\ell^*$ to the last stage can be regarded as a whole process of SH with a budget $\rbr{\ur{\log_2 n}-\ell^*}T'\ge\text{const}\cd\log_2\rbr{\fr{n}{m}}\fr{T}{\log_2\rbr{n}}$. The initial arms are the surviving arms after finishing stage $\ell^*$. Note we have $2^{-\ell^*}\cd n=n/m$ arms surviving, denoted by $S_{\ell^*}$, after finishing stage $\ell^*$. Let $c$ be the index of the true best arm in $S_{\ell^*}$. Due to the event $G_{\ell^*}$, we have $\mu_c\ge\mu_m-\fr{\eps}{2}$. Then, by applying the result of Theorem \ref{them-m1}, 
\begin{align}
    \pp{G^c_{\ell^*+1}\mid G_{\ell^*}}
    &=   \pp{\mu_{J_T}<\mu_m-\eps\mid G_{\ell^*}} \nn\\
    &\le \pp{\mu_{J_T}<\mu_c-\fr{\eps}{2}\mid G_{\ell^*}} \nn\\
    &\le 3\log_2\rbr{\fr{n}{m}}\cd\exp\rbr{-\fr{m}{128}\fr{\eps^2\rbr{\ur{\log_2 n}-\ell^*}T'}{n\log_2\rbr{\fr{n}{m}}}}\nn\\
    &\le \log_2\rbr{\fr{n}{m}}\cd\exp\rbr{-\text{const}\cd m\fr{\eps^2T}{n\log_2n}}~. \label{eq-4}
\end{align}
Bringing \eqref{eq-3},\eqref{eq-4} into \eqref{eq-1}, we have,
\begin{align*}
\pp{\mu_{J_T}<\mu_m-\eps}\le\log_2n\cd\ep{-\text{const}\cd m\rbr{\fr{\eps^2T}{4n\log_2^2(m)\log_2n}-\ln(4e)}}.
\end{align*}
Note the above analysis is for $m>1$. To incorporate the result of Theorem \ref{them-m1} for $m=1$, we rewrite the final formula as 
\begin{align*}
\pp{\mu_{J_T}<\mu_m-\eps}\le\log_2n\cd\ep{-\text{const}\cd m\rbr{\fr{\eps^2T}{4n\log_2^2(2m)\log_2n}-\ln(4e)}}.
\end{align*}
Thus, there exist a positive constant $c_1$ such that for $T\ge c_1\fr{\rbr{\ln(4e)+\ln\ln n}n\log_2^2(2m)\log_2n}{\eps^2}=\tilde{\Theta}\rbr{\fr{n}{\eps^2}}$,
\begin{align*}
\pp{\mu_{J_T}<\mu_m-\eps}\le\ep{-\tilde{\Theta}\rbr{m\fr{\eps^2T}{n}}}.
\end{align*}

For the case of $m>n/2$, we consider
\begin{align*}
    \pp{\mu_{J_T}<\mu_m-\eps}\le\pp{\mu_{J_T}<\mu_{n/2}-\eps}.
\end{align*}
Thus, one can repeat the same analysis as above with $m$ replaced by $n/2$.
Using $m/2 \le n/2 \le m$, the statement of this theorem holds.
\end{proof}

The result of Theorem~\ref{them:minmax} directly improves the previous works including $\mathcal{O}\rbr{\frac{n}{m\eps^2}\log^2\rbr{\frac{1}{\delta}}}$ of \citet{chaudhuri2017pac}, and $\mathcal{O}\rbr{\frac{1}{\eps^2}\rbr{\frac{n}{m}\log\rbr{\frac{1}{\delta}}+\log^2\rbr{\frac{1}{\delta}}}}$ of \citet{chaudhuri2019pac}. 
They share a similar strategy that, assuming the knowledge of $m, \eps$ and $\delta$, uniformly samples $\Theta\rbr{\frac{n}{m}\log\rbr{\frac{1}{\delta}}}$ arms first and then runs the Median-Elimination algorithm of \cite{even2006action} or LUCB of \citet{kaufmann2013information} for this subset. 

\subsection{Proof of Theorem \ref{them:otherEpsBound}}
\begin{theorem}
\label{them:otherEpsBound}
For any $\eps\in(0,1)$, the error probability of SH for identifying an $\eps$-good arm satisfies,
\begin{align*}
    \pp{\mu_{J_T}<\mu_{1}-\eps}\le&\min_{\cbr{\rbr{m',\eps'}:\Delta_{m'}+\eps'\le\eps}}\log_2n\cd\ep{-\textup{const}\cd m'\rbr{\fr{\eps'^2T}{4n\log_2^2(2m')\log_2n}-\ln(4e)}}.
\end{align*}
\end{theorem}
\begin{proof}
Note the result we proved in Theorem~\ref{them:minmax} holds for any $m\le n$ and any $\eps\in(0,1)$. Thus $\pp{\mu_{J_T}<\mu_{1}-\eps}$ can be upper bounded for any $\rbr{m',\eps'}\in\cbr{\rbr{m',\eps'}:\Delta_{m'}+\eps'\le\eps}$. The result is therefore implied.
\end{proof}

\subsection{Proof of Theorem \ref{them:ins_dep}}
\label{app-sec-th1}
\newtheorem*{them:ins_dep}{Theorem~\ref{them:ins_dep}}
\begin{them:ins_dep}
For any $\eps\in[0,1)$, the $\eps$-error probability of SH satisfies
\begin{align*}
    \pp{\mu_{J_T}<\mu_1-\eps}\le\ep{-\hTT{\fr{T}{\max_{i\ge g(\eps)+1}\frac{i}{\Delta_i^2}\cd\frac{1}{g\rbr{\fr{\eps}{2}}}}}},
\end{align*}
for $T\ge\hTT{\max_{i\ge g(\eps)+1}\frac{i}{\Delta_i^2}}$.
\end{them:ins_dep}
\begin{proof}
To avoid redundancy and for the sake of readability, we assume $n$ is of a power of 2. It is easy to verify the result for any $n$. Recall $g\rbr{\fr{\eps}{2}}$ is the number of $\fr{\eps}{2}$-good arms in the instance, and $g_\ell\rbr{\fr{\eps}{2}}$ is the number of $\fr{\eps}{2}$-good arms at the beginning of stage $\ell$. Thus $g_1\rbr{\fr{\eps}{2}}=g\rbr{\fr{\eps}{2}}$. 
Let $\ell'=\max\cbr{1\le\ell\le\log_2(n)-2 \mid\Delta_{\fr{n}{2^{\ell+1}}}>\eps}$.
There are three cases we need to consider:
\begin{itemize}
    \item $\ell'=\log_2n-2$.
    \item $1\le\ell'<\log_2n-2$.
    \item Such an $\ell'$ does not exist.
\end{itemize}

The last case that such a $\ell'$ does not exit means $\eps\ge\Delta_{\frac{n}{4}}$, which means that the instance has $\Theta(n)$ amount of $\eps$-good arms, therefore an easy instance. We address this case at the end of the proof.

\textbf{Case 1. } $\ell'=\log_2n-2$ \\
In this case, we have $\Delta_2>\eps$, then the problem is a fixed budget best arm identification problem rather than an $\eps$-good arm identification problem. 
Since $g\rbr{\eps}=g\rbr{\fr{\eps}{2}}=1$ for $\Delta_2>\eps$, the theorem statement reduces to
\begin{align*}
    \pp{\mu_1-\mu_{J_T}>\eps}\le\ep{-\hTT{\fr{T}{\max_{i\ge 2}\frac{i}{\Delta_i^2}}}},
\end{align*}
which is a trivial consequence of \cite[Theorem 4.1]{karnin2013almost} under the theorem's condition on $T$. 

\textbf{Case 2. } $1\le\ell'<\log_2n-2$ \\
For $1\le\ell'<\log_2n-2$, recall that $|A_\ell| = \frac{n}{2^{\ell-1}}$. One can show that
\begin{align}
  |A_{\ell'+1}|/4 \le g(\eps) \le |A_{\ell'}|/4 \text{~~ and ~~} \Delta_{|A_{\ell'+1}|/4} \le \Delta_{g(\eps)} \le \eps < \Delta_{|A_{\ell'}|/4}~.\label{ieq:sizel'}
\end{align}

For $\ell\in[\ell'+1]$, define the good event for stage $\ell$ as 
\begin{align*}
G_\ell:=\cbr{g_\ell\rbr{\fr{\eps}{2}}\ge g\rbr{\fr{\eps}{2}}-\urd{\rbr{\ell-1}\fr{g\rbr{\fr{\eps}{2}}}{2\log_2n}}+1}.
\end{align*}
Note $G_1$ holds with probability $1$. 
Define $m':=\urd{\fr{g\rbr{\fr{\eps}{2}}}{2\log_2n}}=\hTT{g\rbr{\fr{\eps}{2}}}$. 
Intuitively, we allow at most $m'$ of $\fr{\eps}{2}$-good arms to be eliminated per stage in the initial phases, i.e., $\ell\in[\ell'+1]$ of running a Sequential Halving. The $\eps$-error probability can be expressed as
\begin{align*}
    \pp{\mu_1-\mu_{J_T}>\eps}=&\pp{\mu_1-\mu_{J_T}>\eps,\cap_{\ell=2}^{\ell'+1}G_{\ell}}+\pp{\mu_1-\mu_{J_T}>\eps,\rbr{\cap_{\ell=2}^{\ell'+1}G_\ell}^c}.
\end{align*}
Since
\begin{align*}
    \pp{\mu_1-\mu_{J_T}>\eps,\cap_{\ell=2}^{\ell'+1}G_{\ell}}\le\pp{\mu_1-\mu_{J_T}>\eps \mid \cap_{\ell=2}^{\ell'+1}G_{\ell}},
\end{align*}
and
\begin{align*}
    \pp{\mu_1-\mu_{J_T}>\eps,\rbr{\cap_{\ell=2}^{\ell'+1}G_{\ell}}^c} \le\pp{\rbr{\cap_{\ell=2}^{\ell'+1}G_{\ell}}^c}\le\sum_{\ell=1}^{\ell'}\pp{G_{\ell+1}^c\mid \cap_{i=1}^{\ell}G_{i}},
\end{align*}
we have 
\begin{align}
    \pp{\mu_1-\mu_{J_T}>\eps}\le\pp{\mu_1-\mu_{J_T}>\eps \mid \cap_{\ell=2}^{\ell'+1}G_{\ell}}+\sum_{\ell=1}^{\ell'}\pp{G_{\ell+1}^c\mid \cap_{i=1}^{\ell}G_{i}}.\label{ieq:polyInsEvent}
\end{align}
We bound the two terms in \eqref{ieq:polyInsEvent} respectively. 

We start with the second term. 
Let $\ell\in[\ell']$.
To bound $\pp{G_{\ell+1}^c\mid \cap_{i=1}^{\ell}G_{i}}$ with Lemma~\ref{lemma:insufficientUE}, we need to relate the quantities in stage $\ell$ with those in Lemma~\ref{lemma:insufficientUE}.

We first claim that, given $\cap_{i=1}^{\ell}G_{i}$, the event $G^c_{\ell+1}$ implies that there are at least $m'$ of $\fr{\eps}{2}$-good arms that are eliminated during the stage $\ell$, i.e., $g_\ell(\frac{\eps}{2}) - g_{\ell+1}(\frac \eps 2) \ge m'$.

To see why, we have by $\cap_{i=1}^{\ell}G_{i}$, 
\begin{align*}
    g_\ell\rbr{\fr{\eps}{2}}\ge g\rbr{\fr{\eps}{2}}-\urd{\rbr{\ell-1}\fr{g\rbr{\fr{\eps}{2}}}{2\log_2n}}+1,
\end{align*}
and by $G_{\ell+1}^c$,
\begin{align*}
    g_{\ell+1}\rbr{\fr{\eps}{2}}\le g\rbr{\fr{\eps}{2}}-\urd{\ell\fr{g\rbr{\fr{\eps}{2}}}{2\log_2n}}.
\end{align*}
Therefore, using $\urd{x+y}-\urd{x}+1\ge\urd{y}$, we have
\begin{align*}
    g_\ell(\frac{\eps}{2}) - g_{\ell+1}(\frac \eps 2)
    = \urd{\ell\fr{g\rbr{\fr{\eps}{2}}}{2\log_2n}}+1-\urd{\rbr{\ell-1}\fr{g\rbr{\fr{\eps}{2}}}{2\log_2n}}\ge\urd{\fr{g\rbr{\fr{\eps}{2}}}{2\log_2n}}=m'~,
\end{align*}
which proves the claim.

Define $A_\ell$ as the set of surviving arms at the beginning of stage $\ell$, $A_{\ell,i}$ as $i$-th best arm in $A_\ell$ (ties are broken arbitrarily), and $A_{\ell}\sbr{L+1:R}:=\cbr{A_{\ell,i},i\in\sbr{R}\backslash\sbr{L}}$.
Let $d_{\ell}(\eps)$ be the smallest gap of the mean rewards between any two arms chosen from $A_{\ell}\sbr{\fr{n}{2^{\ell+1}}+1:\fr{n}{2^{\ell-1}}}$ and $A_{\ell}\sbr{1:g_\ell\rbr{\fr{\eps}{2}}}$ respectively, i.e.,
\begin{align*}
    d_{\ell}(\eps)=\min\cbr{\mu_{i}-\mu_j\mid j\in A_{\ell}\sbr{\fr{n}{2^{\ell+1}}+1:\fr{n}{2^{\ell-1}}}, i\in A_{\ell}\sbr{1:g_\ell\rbr{\fr{\eps}{2}}}}.
\end{align*}
The definition of $d_{\ell}(\eps)$ is same as $d_q$ of Lemma \ref{lemma:insufficientUE}.
By Lemma \ref{lemma:mingap}, we have $d_{\ell}(\eps)\ge\frac{1}{2}\Delta_{\fr{n}{2^{\ell+1}}}$. Define $S_\ell(\eps):=\fr{ \Delta_{\fr{n}{2^{\ell+1}}}^2}{n\cd2^{-\ell+1}}$. 

We can now apply Lemma~\ref{lemma:insufficientUE} with $T=T', s=\fr{n}{2^{\ell}}, k=m'$, and $q=\fr{n}{2^{\ell+1}}$ as follows:
\begin{align}
    &\pp{G_{\ell+1}^c\mid \cap_{i=1}^{\ell}G_{i}}\nn\\
    \le&\cc{g_{\ell}\rbr{\fr{\eps}{2}}}{m'} \ep{-m'\fr{d^2_{\ell}(\eps)T'}{8n\cd2^{-\ell+1}}}+\cc{\fr{3}{4}n\cd2^{-\ell+1}}{\fr{1}{4}n\cd2^{-\ell+1}+m'}\ep{-\rbr{\fr{1}{4}n\cd2^{-\ell+1}+m'}\fr{d^2_{\ell}(\eps)T'}{8n\cd2^{-\ell+1}}}\nn\\
    \stkl{1}&\cc{g_{\ell}\rbr{\fr{\eps}{2}}}{m'}\ep{-m'\fr{d^2_{\ell}(\eps)T'}{8n\cd2^{-\ell+1}}}+\cc{n\cd2^{-\ell+1}}{n\cd2^{-\ell}}\ep{-\fr{n}{2^{\ell+1}}\fr{d^2_{\ell}(\eps)T'}{8n\cd2^{-\ell+1}}}\nn\\
    \stkl{2}&\ep{-m'\fr{d^2_{\ell}(\eps)T'}{8n\cd2^{-\ell+1}}+m'\ln\rbr{\fr{eg_{\ell}\rbr{\fr{\eps}{2}}}{m'}}}+\ep{-\fr{n}{2^{\ell+1}}\fr{d^2_{\ell}(\eps)T'}{8n\cd2^{-\ell+1}}+\fr{n}{2^{\ell}}\ln(2e)}\nn\\
    \le&\ep{-m'\fr{\Delta^2_{\fr{n}{2^{\ell+1}}}T'}{32n\cd2^{-\ell+1}}+m'\ln\rbr{\fr{eg_{\ell}\rbr{\fr{\eps}{2}}}{m'}}}+\ep{-\fr{n}{2^{\ell+1}}\fr{\Delta^2_{\fr{n}{2^{\ell+1}}}T'}{32n\cd2^{-\ell+1}}+\fr{n}{2^{\ell}}\ln(2e)}\nn\tag{$d_{\ell}(\eps)\ge\frac{1}{2}\Delta_{\fr{n}{2^{\ell+1}}}$}\\
    \stkl{3}&\ep{-\hTT{m'S_\ell(\eps)T}}+\ep{-\hTT{\fr{n}{2^{\ell+1}}S_\ell(\eps)T}},  \tag{$T\ge\hTT{\fr{1}{S_\ell(\eps)}}$}\\\label{ieq:sh2rd}
\end{align}
where, 
\begin{itemize}
    \item $b_1$ is by $\cc{x}{y}<\cc{x'}{y}$ for $x'>x$ and $\max_y\cc{x}{y}=\cc{x}{\drd{x/2}}=\cc{x}{\urd{x/2}}$,
    \item $b_2$ is by upper bounding the binomial coefficients by Sterling’s formula $\cc{x}{y}\le\rbr{\fr{ex}{y}}^y$ and moving the coefficients into the exponent.
    \item $b_3$ is by $T\ge\hTT{\fr{1}{S_\ell(\eps)}}$ which is implied by $T\ge\hTT{\max_{i\ge g(\eps)+1}\frac{i}{\Delta_i^2}}$ in the theorem statement.
\end{itemize}
By the definition of $\ell'$ and the fact of $\ell\in[\ell']$, we have 
\begin{align*}
    \Delta_{\fr{n}{2^{\ell+1}}}>\eps.
\end{align*} 
By the definition of $g\rbr{\fr{\eps}{2}}$, we have
\begin{align*}
    \Delta_{g\rbr{\fr{\eps}{2}}}\le\frac{\eps}{2}.    
\end{align*}
Thus
\begin{align*}
    \Delta_{\fr{n}{2^{\ell+1}}}-\Delta_{g\rbr{\fr{\eps}{2}}}>\frac{\eps}{2},
\end{align*}
which means $\fr{n}{2^{\ell+1}}>g\rbr{\fr{\eps}{2}}\ge m'$. We continue to upper bound \eqref{ieq:sh2rd}. For $T\ge\hTT{\fr{1}{S_\ell(\eps)}}$, we have
\begin{align}
    \pp{G_{\ell+1}^c\mid \cap_{i=1}^{\ell}G_{i}}\le&\ep{-\hTT{m'S_\ell(\eps)T}}+\ep{-\hTT{\fr{n}{2^{\ell+1}}S_\ell(\eps)T}}\nn\\
    \le&\ep{-\hTT{m'S_\ell(\eps)T}}+\ep{-\hTT{m'S_\ell(\eps)T}}\nn\\
    \le&\ep{-\hTT{g\rbr{\fr{\eps}{2}}S_\ell(\eps)T}}.\tag{$m'=\hTT{g\rbr{\fr{\eps}{2}}}$}\\\label{ieq:genIns2}
\end{align}

Next, we bound the first term in \eqref{ieq:polyInsEvent}. The stages $\ell\ge\ell'+1$ can be viewed as a new Sequential Halving game with $A_{\ell'+1}$ as the initial arm set, $g_{\ell'+1}\rbr{\fr{\eps}{2}}$ of which are $\fr{\eps}{2}$-good. 
When $\cap_{i=1}^{\ell'+1}G_i$ happens, we have, as $\ell'<\log_2n$, 
\begin{align}
g_{\ell'+1}\rbr{\fr{\eps}{2}}\ge g\rbr{\fr{\eps}{2}}-\urd{\ell'\fr{g\rbr{\fr{\eps}{2}}}{2\log_2n}}+1\ge g\rbr{\fr{\eps}{2}}-\urd{\fr{g\rbr{\fr{\eps}{2}}}{2}}+1\ge\fr{g\rbr{\fr{\eps}{2}}}{2}=\hTT{g\rbr{\fr{\eps}{2}}}.\nn
\end{align} 
By Lemma~\ref{lemma:eqiv_eps_delta}, a $\rbr{g_{\ell'+1}\rbr{\fr{\eps}{2}}, \fr{\Delta_{g(\eps)+1}}{2}}$-good arm in $A_{\ell'+1}$ returned after finishing the final stage of Sequential Halving is an $\eps$-good arm. 
We define $\mu_m(A)$ as the $m$-th largest value in $\{\mu_j: j \in A\}$. 
We apply Theorem \ref{them:minmax} and set $m, n, \eps$ of Theorem \ref{them:minmax} as $m=g_{\ell'+1}\rbr{\fr{\eps}{2}}, n=\abr{A_{\ell'+1}}, \eps=\fr{\Delta_{g(\eps)+1}}{2}$. 
To apply Theorem \ref{them:minmax}, we need the condition of $T \ge \hTT{\frac{n}{\eps^2}}$ in Theorem \ref{them:minmax} to be true, which means we need $T\ge\hTT{\fr{\abr{A_{\ell'+1}}}{\Delta^2_{g(\eps)+1}}}$. 
Note by \eqref{ieq:sizel'}, we have $\fr{n}{2^{\ell'+2}}\le g(\eps)$, thus $\abr{A_{\ell'+1}}=\fr{n}{2^{\ell'}}\le4g\rbr{\eps}\le4\rbr{g\rbr{\eps}+1}$. 
Using this bound and the condition on $T$ from the theorem statement, we have  $T\ge\hTT{\max_{i\ge g(\eps)+1}\frac{i}{\Delta_i^2}} \ge \hTT{\fr{g(\eps)+1}{\Delta^2_{g(\eps)+1}}}\ge\hTT{\fr{\abr{A_{\ell'+1}}}{\Delta^2_{g(\eps)+1}}}$
So, the condition of Theorem \ref{them:minmax} is satisfied.
Then, we have
\begin{align}
    &\pp{\mu_1-\mu_{J_T}>\eps \mid \cap_{i=1}^{\ell'+1}G_i}\nn\\ \le&\pp{\mu_{g_{\ell'+1}\rbr{\fr{\eps}{2}}}(A_{\ell'+1})-\mu_{J_T}>
    \fr12 \Delta_{g(\eps)+1} \mid \cap_{i=1}^{\ell'+1}G_i}\nn\tag{Lemma~\ref{lemma:eqiv_eps_delta}}\\
    \le&\ep{-\hTT{g_{\ell'+1}\rbr{\fr{\eps}{2}}\cd\fr{\Delta^2_{g(\eps)+1}T'}{\abr{A_{\ell'+1}}}}}\nn\tag{$T\ge\hTT{\fr{g(\eps)+1}{\Delta^2_{g(\eps)+1}}}$}\\
    \le&\ep{-\hTT{g\rbr{\fr{\eps}{2}}\cd\fr{\Delta^2_{g(\eps)+1}T'}{4\rbr{g\rbr{\eps}+1}}}}.\tag{$g_{\ell'+1}\rbr{\fr{\eps}{2}}=\hTT{g\rbr{\fr{\eps}{2}}}$}\\\label{ieq:genIns1}
\end{align}

The role of considering $\ell'$ is that as long as the Sequential Halving goes well ($\cap_{i=1}^{\ell'+1}G_i$ happens) the ratio of $\frac{\eps}{2}$-good arms to the all surviving arms is increasing to $g\rbr{\fr{\eps}{2}}/g\rbr{\eps}$ after finishing stage $\ell'$. Specifically, for the Polynomial$(\alpha)$ instance of our interest where $\Delta_i=\rbr{\frac{i}{n}}^\alpha$ with $\alpha>0.5$, $g\rbr{\fr{\eps}{2}}/g\rbr{\eps}$ can be lower bounded by a constant. The last step is to
combine \eqref{ieq:genIns2} and \eqref{ieq:genIns1}. 
For $T\ge\hTT{\max_{i\ge g(\eps)+1}\frac{i}{\Delta_i^2}} $, we have 
\begin{align*}
    \pp{\mu_1-\mu_{J_T}>\eps}\le&\pp{\mu_1-\mu_{J_T}>\eps \mid \cap_{\ell=2}^{\ell'+1}G_{\ell}}+\sum_{\ell=1}^{\ell'}\pp{G_{\ell+1}^c\mid \cap_{i=1}^{\ell}G_{i}}\\ \le&\ep{-\hTT{g\rbr{\fr{\eps}{2}}\cd\fr{\Delta^2_{g(\eps)+1}T}{4\rbr{g\rbr{\eps}+1}}}}+\log \rbr{n}\ep{-\hTT{g\rbr{\fr{\eps}{2}}\min_{\ell\le\ell'}S_\ell(\eps)T}}\\
    \le&\ep{-\hTT{\min\cbr{g\rbr{\fr{\eps}{2}}\cd\fr{\Delta^2_{g(\eps)+1}}{4\rbr{g\rbr{\eps}+1}}, g\rbr{\fr{\eps}{2}}\min_{\ell\le\ell'}S_\ell(\eps)}T}}\tag{$T\ge\hTT{\max_{i\ge g(\eps)+1}\frac{i}{\Delta_i^2}}$}\\
    \le&\ep{-\hTT{g\rbr{\fr{\eps}{2}}\cd\min\cbr{\fr{\Delta^2_{g(\eps)+1}}{4\rbr{g\rbr{\eps}+1}}, \min_{\ell\le\ell'}\fr{\Delta_{\fr{n}{2^{\ell+1}}}^2}{\fr{4n}{2^{\ell+1}}}}T}}\\
    \le&\ep{-\hTT{g\rbr{\fr{\eps}{2}}\cd\min_{i\ge g\rbr{\eps}+1}\fr{\Delta_{i}^2}{i}T}}\\
    =&\ep{-\hTT{\fr{T}{\max_{i\ge g(\eps)+1}\frac{i}{\Delta_i^2}\cd\frac{1}{g\rbr{\fr{\eps}{2}}}}}}.
\end{align*}

\textbf{Case 3. } $\ell'$ does not exist \\
This case condition implies that $\Delta_{\fr{n}{4}}\le\eps$, which means the number of $\eps$-good arms is $\Theta\rbr{n}$, so we can use \eqref{ieq:genIns1} solely to claim our bound.
\end{proof}

\begin{lemma}
\label{lemma:mingap}
 In the case of $1\le\ell'<\log_2n-2$, for $\ell\in[\ell']$, define $d_{\ell}(\eps)$ as the minimum gap of the mean rewards between any two arms chosen from $A_{\ell}\sbr{\fr{n}{2^{\ell+1}}+1:\fr{n}{2^{\ell-1}}}$ and $A_{\ell}\sbr{1:g_\ell\rbr{\fr{\eps}{2}}}$ respectively,
\begin{align*}
    d_{\ell}(\eps)=\min\cbr{\mu_{i}-\mu_j\mid\forall j\in A_{\ell}\sbr{\fr{n}{2^{\ell+1}}+1:\fr{n}{2^{\ell-1}}}, \forall i\in A_{\ell}\sbr{1:g_\ell\rbr{\fr{\eps}{2}}}}.
\end{align*}
Then $d_{\ell}(\eps)$ satisfies
\begin{align*}
    d_{\ell}(\eps)\ge\frac{1}{2}\Delta_{\fr{n}{2^{\ell+1}}}.
\end{align*}
\end{lemma}
\begin{proof}
By definition, we have $A_{\ell}\subset A_{1}$, which implies the $i$-th best arm in $A_{\ell}$ has an mean reward lower than or equal to the $i$-th best arm in $A_{1}$. Then we have
\begin{align*}
    \max\cbr{\mu_i \mid i\in A_{\ell}\sbr{\fr{n}{2^{\ell+1}}+1:\fr{n}{2^{\ell-1}}}}&\le\max\cbr{\mu_i \mid i\in A_{1}\sbr{\fr{n}{2^{\ell+1}}+1:\fr{n}{2^{\ell-1}}}}\\
    &\le\mu_{\fr{n}{2^{\ell+1}}}\\
    &=\mu_1-\Delta_{\fr{n}{2^{\ell+1}}}.
\end{align*}
Also,
\begin{align*}
    \min\cbr{\mu_i, i\in A_{\ell}\sbr{1:g_\ell\rbr{\fr{\eps}{2}}}}\ge\mu_1-\fr{\eps}{2}.
\end{align*}
Therefore for $\forall j\in A_{\ell}\sbr{\fr{n}{2^{\ell+1}}+1:\fr{n}{2^{\ell-1}}}$ and $\forall i\in A_{\ell}\sbr{1:g_\ell\rbr{\fr{\eps}{2}}}$,
\begin{align*}
    \mu_{i}-\mu_j\ge\mu_1-\fr{\eps}{2}-\rbr{\mu_1-\Delta_{\fr{n}{2^{\ell+1}}}}=\Delta_{\fr{n}{2^{\ell+1}}}-\fr{\eps}{2}.
\end{align*}
Since $\ell\in[\ell']$, by the definition of $\ell'$ we have 
\begin{align*}
    \Delta_{\fr{n}{2^{\ell+1}}}>\eps.
\end{align*}
Thus
\begin{align*}
    \mu_{i}-\mu_j\ge\Delta_{\fr{n}{2^{\ell+1}}}-\fr{\eps}{2}=\frac{\Delta_{\fr{n}{2^{\ell+1}}}}{2}+\frac{\Delta_{\fr{n}{2^{\ell+1}}}}{2}-\fr{\eps}{2}>\frac{\Delta_{\fr{n}{2^{\ell+1}}}}{2}.
\end{align*}
\end{proof}

\begin{lemma}
\label{lemma:eqiv_eps_delta}
A $\rbr{g_{\ell'+1}\rbr{\fr{\eps}{2}}, \fr{\Delta_{g(\eps)+1}}{2}}$-good arm in $A_{\ell'+1}$ returned after finishing the final stage of Sequential Halving is an $\eps$-good arm.
\end{lemma}
\begin{proof}
We first claim that 
\begin{align*}
  \Delta_{J_T} > \eps \iff \Delta_{J_T} > \Delta_{g(\eps)}~.
\end{align*}
To see this, for the forward direction, we have $\Delta_{J_T} > \eps \ge \Delta_{g(\eps) }$.
For the other direction, if $\Delta_{J_T} > \Delta_{g(\eps)}$ is true, then $J_T > g(\eps)$. This means that $\Delta_{J_T} > \eps$. Then, we can also show that
\begin{align*}
  \Delta_{J_T} > \eps \iff \Delta_{J_T} \ge \Delta_{g(\eps)+1},
\end{align*}
which can be shown easily by showing \begin{align*}
\Delta_{J_T} \ge \Delta_{g(\eps)+1} \iff \Delta_{J_T} > \Delta_{g(\eps)}.
\end{align*} Define $\Delta_m(A)$ to be the $m$-th smallest value in $\{\Delta_i: i \in A\}$. Then, we have $\Delta_{g_{\ell'+1}\rbr{\fr{\eps}{2}}}(A_{\ell'+1}) \le \fr12 \eps$.
Using the fact that $\eps<\Delta_{g(\eps)+1}$, we have
\begin{align*}
  \Delta_{J_T} \ge \Delta_{g(\eps)+1} 
    &=   \Delta_{g_{\ell'+1}\rbr{\fr{\eps}{2}}}(A_{\ell'+1}) + \Delta_{g(\eps)+1} - \Delta_{g_{\ell'+1}\rbr{\fr{\eps}{2}}}(A_{\ell'+1})
  \\&\ge \Delta_{g_{\ell'+1}\rbr{\fr{\eps}{2}}}(A_{\ell'+1}) + \Delta_{g(\eps)+1} - \fr12 \eps
  \\&> \Delta_{g_{\ell'+1}\rbr{\fr{\eps}{2}}}(A_{\ell'+1}) + \fr12 \Delta_{g(\eps)+1}.\tag{$\eps<\Delta_{g(\eps)+1}$}
\end{align*}

By defining $\mu_m(A)$ as the $m$-th largest value in $\{\mu_j: j \in A\}$, the above equivalently means
\begin{align*}
    \mu_{J_T} &< \mu_{g_{\ell'+1}\rbr{\fr{\eps}{2}}}(A_{\ell'+1}) - \fr12 \Delta_{g(\eps)+1}.
\end{align*}
Therefore the output is not a $\rbr{g_{\ell'+1}\rbr{\fr{\eps}{2}}, \fr{\Delta_{g(\eps)+1}}{2}}$-good arm in $A_{\ell'+1}$. Our claim holds by contraposition.
\end{proof}

\subsection{Proof of Corollary \ref{cor:poly}}
\label{app-sec-poly}
\newtheorem*{cor:poly}{Corollary~\ref{cor:poly}}
\begin{cor:poly}
For the Polynomial$(\alpha)$ instance where $\Delta_i=\rbr{\frac{i}{n}}^\alpha$ with $\alpha>0.5$, we have
\begin{align*}
H_2(\eps)=\max_{i\ge g(\eps)+1}\frac{i}{\Delta_i^2}\cd\frac{1}{g\rbr{\fr{\eps}{2}}}<\frac{4}{\eps^2}.
\end{align*}
\end{cor:poly}
\begin{proof}
\begin{align*}
\max_{i\ge g(\eps)+1}\frac{i}{\Delta_i^2}\cd\frac{1}{g\rbr{\fr{\eps}{2}}}=\max_{i\ge g\rbr{\eps}+1}i^{1-2\alpha}n^{2\alpha}\fr{1}{n\rbr{\fr{\eps}{2}}^{1/\alpha}}<\rbr{n\eps^{1/\alpha}}^{1-2\alpha}n^{2\alpha}\fr{1}{n\rbr{\fr{\eps}{2}}^{1/\alpha}}<\frac{4}{\eps^2}.
\end{align*}
where both the last inequality and the last equality are by $\alpha>0.5$.
\end{proof}

\subsection{Proof of Corollary \ref{cor:kjlow}}
\label{app-sec-kjlow}

\newtheorem*{cor:kjlow}{Corollary~\ref{cor:kjlow}}
\begin{cor:kjlow}
\cite{Katz20} shows a lower bound of the sample complexity for identifying an $\eps$-good arm that scales as
\begin{align*}
        \Omega\del{H^{\text{low}}(\eps) := {\frac{1}{g(\eps)}\sum_{i=g(\eps)+1}^{n}\rbr{\frac{1}{\Delta_i^2}}-\frac{1}{\Delta_{g(\eps)+1}^2}}}.
\end{align*}
In the case that $g(\eps)$ and $g\rbr{\frac{\eps}{2}}$ have the same order ($\fr{g\rbr{\fr{\eps}{2}}}{g\rbr{\eps}}$ is irrelevant to $\eps$), the sample complexity measure $H_2(\eps)$ satisfies
\begin{align*}
    H^{\text{low}}(\eps)\lesssim H_2(\eps)\lesssim H^{\text{low}}(\eps) +\frac{2}{\Delta_{g(\eps)+1}^2},
\end{align*}
where $\lesssim$ hides logarithms of $n$ and constants.
\end{cor:kjlow}
\begin{proof}
\begin{align*}
    H^{\text{low}}(\eps)=&\frac{1}{g(\eps)}\sum_{i=g(\eps)+1}^{n}\frac{1}{\Delta_i^2}-\frac{1}{\Delta_{g(\eps)+1}^2}\\
    =&\frac{1}{g(\eps)}\sum_{i=g(\eps)+1}^{n}\frac{1}{i}\frac{i}{\Delta_i^2}-\frac{1}{\Delta_{g(\eps)+1}^2}\\
    \le&\frac{1}{g(\eps)}\max_{i\ge g(\eps)+1}\frac{i}{\Delta_i^2}\cd\sum_{i=g(\eps)+1}^{n}\frac{1}{i}-\frac{1}{\Delta_{g(\eps)+1}^2}\\
    \stackrel{(a_1)}{\le}&\frac{1}{g(\eps)}\max_{i\ge g(\eps)+1}\frac{i}{\Delta_{i}^2}\cd\log\rbr{2n}-\frac{1}{\Delta_{g(\eps)+1}^2}\\
    =&\frac{g\rbr{\fr{\eps}{2}}}{g\rbr{\eps}}H_2(\eps)\cd\log\rbr{2n}-\frac{1}{\Delta_{g(\eps)+1}^2}\\
    \le&\frac{g\rbr{\fr{\eps}{2}}}{g\rbr{\eps}}H_2(\eps)\cd\log\rbr{2n},
\end{align*}
for $(a_1)$, we use $\sum_{i=g(\eps)}^{n}\frac{1}{i}\le\sum_{i\in[n]}\frac{1}{i}\le\log(2n)$, which can be found in \cite{audibert10}. On the other hand,
\begin{align*}
H^{\text{low}}(\eps)=&\frac{1}{g(\eps)}\sum_{i=g(\eps)+1}^{n}\frac{1}{\Delta_i^2}-\frac{1}{\Delta_{g(\eps)+1}^2}\\
\ge&\frac{1}{g(\eps)}\max_{i\ge g(\eps)+1}\frac{i-g(\eps)}{\Delta_i^2}-\frac{1}{\Delta_{g(\eps)+1}^2}\\
=&\frac{1}{g(\eps)}\rbr{\max_{i\ge g(\eps)+1}\frac{i}{\Delta_i^2}-\frac{g(\eps)}{\Delta_i^2}}-\frac{1}{\Delta_{g(\eps)+1}^2}\\
\ge&\frac{1}{g(\eps)}\rbr{\max_{i\ge g(\eps)+1}\frac{i}{\Delta_i^2}-\frac{g(\eps)}{\Delta_{g(\eps)+1}^2}}-\frac{1}{\Delta_{g(\eps)+1}^2}\\
=&\frac{1}{g(\eps)}\max_{i\ge g(\eps)+1}\frac{i}{\Delta_i^2}-\frac{2}{\Delta_{g(\eps)+1}^2}\\
=&\frac{g\rbr{\fr{\eps}{2}}}{g\rbr{\eps}}H_2(\eps)-\frac{2}{\Delta_{g(\eps)+1}^2}.
\end{align*}
\end{proof}

\subsection{Proof of Theorem \ref{thm:gen_lb}}
\label{app-sec-genlb}
For an $n$-armed bandit instance $\nu$ and a fixed value of $\eps > 0$, consider a partition of the arms of $\nu$, $\cA = \set{[g(\eps)],[g(\eps)+1:2\cd g(\eps)],\ldots, [n-g(\eps)+1:n]}$, where we avoid rounding by making the assumption that $n$ is an integer multiple of $g(\eps)$. We use $\cA_i \coloneq [(i-1)\cd g(\eps):i\cd g(\eps)]$ to denote the subsets of the partition $\cA$. Making use of this partition, we define
\begin{equation*}
    \tilde{H}(\nu, \eps) \coloneq \sum_{i=2}^{n/g(\eps)} \tilde{\Delta}_{\cA_i}^{-2},
\end{equation*}
where for an arbitrary set of arms $A \subseteq [n]$ we define $\tilde{\Delta}_A \coloneq \max_{j\in A} \Delta_j$, the largest gap among the arms in $\cA$.

\begin{proof}[{Proof of Theorem \ref{thm:gen_lb}}]
The proof structure follows \citet[Theorem 1]{Carpentier2016}. Fix $n>2$, $\eps>0$ an $n$-armed unit-variance gaussian instance $\nu$ and an arbitrary policy $\pi$. Let $\PP_{\nu, \pi}$ denote the probability measure over interaction histories induced by the interaction of $\pi$ with instance $\nu$. Without loss of generality, we take the arm means of $\nu$ to be ordered such that $\mu_1 \ge \mu_2\ge \ldots \ge \mu_n$. Understanding $g(\eps)$ to refer to the number of $\eps$-good arms on $\nu$, we define instances $\nu^i$ for $i \in [2:n/g(\eps)]$ from $\nu$ by setting $\nu^i_j = \nu_j$ for $j \notin \cA_i$,  and setting $\mu^i_j \coloneq \mu_j + 2\Delta_j$ for $j \in \cA_i$. Each instance $\nu^i$ has $g(\eps)$ $\eps$-good arms with indices in $\cA_i$, and $\tilde{H}(\nu^i, \eps) \le \tilde{H}(\nu, \eps) = a$ for all $i \in [2:n/g(\eps)]$.
        
        We will make use of a change of measure between $\nu$ and these alternate instances, restricted on an event on which the interaction histories on $\nu$ and $\nu^i$ are hard to distinguish from one another:
    \begin{equation*}
            \mathcal{E}_i = \set{\sum_{t=1}^T \sum_{j=1}^n \1{I_t=j} \sbr{\Lambda^i_t(j) - (1+\rho) \textsf{KL}(\nu^i_j, \nu_j)} \le \dfrac{1}{\rho} \log(1/\delta)},
    \end{equation*}
        where $\Lambda^i_t(j) = \log\rbr{\dfrac{\mathrm{d}\nu_j\rbr{X_t}}{\mathrm{d}\nu^i_j\rbr{X_t}}}$ for $i\in[2:n/g(\eps)]$, and $\rho > 0$ is a parameter to be chosen later.

        \begin{lemma}[{\citet[Lemma 10]{Jang2022}}]\label{lem:kl_conc} For all $i \in[2:n/g(\eps)]$ and arbitrary $\pi$, $\rho > 0$, and $\delta \in (0,1)$,
            \begin{equation*}
                \PP_{\nu, \pi} \rbr{\cE_i} \ge 1-\delta.
            \end{equation*}
        \end{lemma}
        We can now prove the result. 
        Denoting the number of times arm $j\in[n]$ has been pulled in $T$ rounds of interaction by $N_j(T)$, let $t_i = \EE_{\nu, \pi}\sbr{\sum_{j\in\cA_i} N_j(T)}$ be the expected number of pulls that arm group $\cA_i$ receives, for $i \in [2:n/g(\eps)]$. By Markov's inequality, for all $i \in [2:n/g(\eps)]$ we have
        \begin{equation*}
            \PP_{\nu, \pi}\rbr{\sum_{j\in\cA_i} N_j(T) \ge 6t_i} \le 1/6.
        \end{equation*}
        Defining the event
        \begin{equation*}
            E_i = \set{{J_t} \in \cA_1}\cap \set{\cE_i} \cap \set{\sum_{j\in\cA_i} N_j(T) \le 6t_i}
        \end{equation*}
        for $i \in[2:n/g(\eps)]$ and setting $\delta = 1/6$ in Lemma \ref{lem:kl_conc}, for all $i\in [2:n/g(\eps)]$ and arbitrary $\pi$
        \begin{equation*}
            \PP_{\nu,\pi}\rbr{E_i} \ge 1- (1/6+1/6+\xi) = 2/3-\xi,
        \end{equation*}
        where $\xi \coloneq \PP_{\nu,\pi}\rbr{{J_t} \notin \cA_1}$.

        It follows that for $i \in [2:n/g(\eps)]$
        \begin{align*}
            \PP_{\nu^i,\pi}\rbr{{J_t} \notin \cA_i}& \ge \PP_{\nu^i,\pi}\rbr{{J_t} \in \cA_1}\\
            &\ge \PP_{\nu^i, \pi}\rbr{E_i}.
        \end{align*}
        Making use of a standard change of measure technique for bandits~\citep{audibert10best}, we have 
        \begin{align*}
        \PP_{\nu^i, \pi}\rbr{E_i} = &\EE_{\nu,\pi}\sbr{\1{E_i} \exp\rbr{-\sum_{t=1}^T\sum_{j\in\cA_i} \1{I_t = a}\Lambda^i_t(j)}}\\
            &\ge \EE_{\nu,\pi}\sbr{\1{E_i} \exp\rbr{-\sum_{t=1}^T \sum_{j\in\cA_i} \1{I_t=j} 2 \textsf{KL}(\nu^i_j, \nu_j) - \log(6)}}\tag{on $E_i$, taking $\rho=1$ and $\delta=1/6$}\\
            &\ge \EE_{\nu,\pi}\sbr{\1{E_i} \exp\rbr{-\sum_{j\in \cA_i} 4 N_j(T)\cd  \max_{k\in\cA_i} \Delta_k^2 - \log(6)}} \tag{evaluating the KL-divergence}\\
            &= \dfrac{1}{6}\EE_{\nu,\pi}\sbr{\1{E_i} \exp\rbr{-\sum_{a\in \cA_i} 4N_j(T) \tilde{\Delta}^2_{\cA_i}}}\tag{definition of $\tilde{\Delta}_{\cA_i}$}\\
            &\ge \dfrac{1}{6} \EE_{\nu,\pi}\sbr{\1{E_i} \exp\rbr{-24t_i \tilde{\Delta}^2_{\cA_i}}} \tag{on $E_i$}\\
            &=\frac{1}{6} \exp\rbr{-24t_i \tilde{\Delta}^2_{\cA_i}}\PP_{\nu, \pi}\rbr{E_i}.
        \end{align*}
        If $\xi \ge 1/2$, the result follows. Otherwise if $\xi \le 1/2$, then  $\PP_{\nu}\rbr{E_i}\ge 1/6$ and it follows that $\PP_{\nu^i}\rbr{{J_t} \notin \cA_i} \ge \frac{1}{36} \exp\rbr{-24 t_i \tilde{\Delta}^2_{\cA_i}}$.

        Observing that $\tilde{H}(\eps, \nu) = \sum_{i=2}^{n/g(\eps)} \tilde{\Delta}_{\cA_i}^{-2}$, and $\sum_{i=1}^{n/m} t_i = T$, there exists some $i \in [2:n/g(\eps)]$ such that
        \begin{equation*}
            t_i \le \dfrac{T}{\tilde{H}(\eps, \nu) \tilde{\Delta}^2_{\cA_i}} = \dfrac{T}{a \tilde{\Delta^2_{\cA_i}}},
        \end{equation*}
        which can easily be shown by way of contradiction.
        For this $i$, by our display above, we have 
        \begin{equation*}
            \max\set{\PP_{\nu^i,\pi}\rbr{{J_t} \notin \cA_i}, \PP_{\nu,\pi}\rbr{{J_t} \notin \cA_1}} \ge \max\set{1/2, \frac{1}{36} \exp\rbr{-24 T/a}}.
        \end{equation*}
    \end{proof}

\subsection{Proof of Lemma \ref{lem:poly_lb}}
\label{app-sec-polylb}
    \begin{proof}[{Proof of Lemma \ref{lem:poly_lb}}]
        To see this, for a given value of $n$ and $\alpha > 1/2$, fix $\eps\le 1/2$. 
        We have 
        \begin{align*}
            \tilde{H}(\eps) &= \sum_{i=2}^{n/g(\eps)} \tilde{\Delta}^{-2}_{\cA_i}\\
            &= \sum_{i=2}^{n/g(\eps)} \rbr{\dfrac{n}{i\cd g(\eps)}}^{2\alpha} \tag{Polynomial$(\alpha)$ instance}\\
            &\ge \rbr{\dfrac{n}{g(\eps)}}^{2\alpha} \sbr{\sum_{i=2}^{n/g(\eps)} \int_{i}^{i+1}x^{-2\alpha} dx} \tag{$x^{-2\alpha}$ decreases monotonically on $x>0$, for $\alpha > 0$}\\
            &= \rbr{\dfrac{n}{g(\eps)}}^{2\alpha} \int_2^{n/g(\eps)+1} x^{-2\alpha} dx\\
            &\ge \rbr{\dfrac{n}{g(\eps)}}^{2\alpha} \int_2^{\eps^{-1/\alpha}+1} x^{-2\alpha} dx \tag{$\rbr{\frac{g(\eps)}{n}}^\alpha < \eps$, by definition of $g(\eps)$.}\\
            &= \dfrac{1}{\eps^2} \dfrac{2^{-2\alpha+1}-\rbr{\eps^{-1/\alpha} + 1}^{-2\alpha+1}}{2\alpha -1}\\
            &\ge \dfrac{1}{\eps^2} \sbr{\dfrac{2^{-2\alpha+1}-\rbr{2^{1/\alpha} + 1}^{-2\alpha+1}}{2\alpha -1}} \tag{assuming $\eps < 1/2$}.
        \end{align*}   
    \end{proof}

\subsection{Analysis of the uniform sampling}
Uniform sampling means the algorithm that takes a budget $T$ as the input and equally allocates the budget to all the arms. 
Define  $\dsym{\text{Top}_m(\eps)}: =\cbr{i:\mu_i\ge\mu_m-\eps}$ and a shortcut of $\dsym{\text{Top}_m} := \Top_m(0)$.

\begin{proposition}
\label{prop-1}
Suppose we run the uniform sampling on a $n$-armed bandit with a budget $T$ and output the $s$ arms with the largest empirical rewards. Let $C_T(s,m',\eps)$ be the $(m',\eps)$-good arms of the output. Then,
\begin{align*}
    \pp{\abr{C_T(s,m',\eps)}\le k}\le \cc{m'}{m'-k} \ep{-\rbr{m'-k}\fr{\eps^2T}{8n}}+\cc{n-\abr{\textup{Top}_{m'}(\eps)}}{s-k}\ep{-\rbr{s-k}\fr{\eps^2T}{8n}}.
\end{align*}
\end{proposition}
This proposition is where our key contribution is rooted at.
We show that our bound is tight for the uniform sampling in Section~\ref{sec-uniform-lb}.

\begin{proof}
\textbf{Step 1}\\
We claim that the intersection of the two conditions below is a sufficient condition for $\abr{C_T(s,m',\eps)}> k$.
\begin{itemize}
    \item $\abr{\cbr{i\in\text{Top}_{m'}:\hat{\mu}_i>\mu_i-\fr{\eps}{2}}}\ge k+1$.
    \item $\abr{\cbr{i\in\text{Top}^c_{m'}(\eps):\hat{\mu}_i<\mu_i+\fr{\eps}{2}}}\ge \abr{\text{Top}^c_{m'}(\eps)}-s+k+1$.
\end{itemize}
If all arms in $\cbr{i\in\text{Top}_{m'}:\hat{\mu}_i>\mu_i-\fr{\eps}{2}}$ are included in the output, the claim naturally holds. 
Thus we focus on the situation where at least one of the arms in $\cbr{i\in\text{Top}_{m'}:\hat{\mu}_i>\mu_i-\fr{\eps}{2}}$ is not included in the output. 
We prove the claim by contradiction. 
Suppose $\abr{C_T(s,m',\eps)}\le k$. Then the output includes no less than $s-k$ non-$(m',\eps)$-good arms. Since
\begin{align*}
    s-k+\abr{\text{Top}^c_{m'}(\eps)}-s+k+1=\abr{\text{Top}^c_{m'}(\eps)}+1>\abr{\text{Top}^c_{m'}(\eps)},
\end{align*}
at least one arm in $\cbr{i\in\text{Top}^c_{m'}(\eps):\hat{\mu}_i<\mu_i+\fr{\eps}{2}}$ is included in the output. To be more specific, there must exist overlap between $\cbr{i\in\text{Top}^c_{m'}(\eps):\hat{\mu}_i<\mu_i+\fr{\eps}{2}}$ and the set of the non-$(m',\eps)$-good arms included in the output; otherwise the union of these two sets is even larger than the set of all the non-$(m',\eps)$-good arms we have in the entire arm set. Additionally, all arms in $\cbr{i\in\text{Top}^c_{m'}(\eps):\hat{\mu}_i<\mu_i+\fr{\eps}{2}}$ have less empirical rewards than any arms in $\cbr{i\in\text{Top}_{m'}:\hat{\mu}_i>\mu_i-\fr{\eps}{2}}$. Thus the output includes all arms in $\cbr{i\in\text{Top}_{m'}:\hat{\mu}_i>\mu_i-\fr{\eps}{2}}$ as well. This leads to the contradiction with our supposition that at least one of the arms in $\cbr{i\in\text{Top}_{m'}:\hat{\mu}_i>\mu_i-\fr{\eps}{2}}$ are not included in the output.\\

\textbf{Step 2}\\
We apply union bound and Hoeffding's inequality for the event that at least one of the two conditions does not hold. 
\begin{align*}
    \pp{\abr{C_T(s,{m'},\eps)}\le k}\le&\pp{\abr{\cbr{i\in\text{Top}_{m'}:\hat{\mu}_i>\mu_i-\fr{\eps}{2}}}< k+1}\\
    &+\pp{\abr{\cbr{i\in\text{Top}^c_{m'}(\eps):\hat{\mu}_i<\mu_i+\fr{\eps}{2}}}< \abr{\text{Top}^c_{m'}(\eps)}-s+k+1}\\
    =&\pp{\abr{\cbr{i\in\text{Top}_{m'}:\hat{\mu}_i\le\mu_i-\fr{\eps}{2}}}\ge m'-k}\\
    &+\pp{\abr{\cbr{i\in\text{Top}^c_{m'}(\eps):\hat{\mu}_i\ge\mu_i+\fr{\eps}{2}}}\ge s-k}\\
    \le&\pp{\exists\mathcal{A}\subset\text{Top}_{m'}, \text{s.t.}\; \abr{\mathcal{A}}=m'-k\;\text{and}\;\forall i\in\mathcal{A},\hat{\mu}_i\le\mu_i-\fr{\eps}{2}}\\
    &+\pp{\exists\mathcal{A}\subset\text{Top}^c_{m'}(\eps), \text{s.t.}\; \abr{\mathcal{A}}=s-k\;\text{and}\;\forall i\in\mathcal{A},\hat{\mu}_i\ge\mu_i+\fr{\eps}{2}}\\
    \le&\cc{m'}{m'-k}\ep{-\rbr{m'-k}\fr{\eps^2T}{8n}}+\cc{n-\abr{\text{Top}_{m'}(\eps)}}{s-k}\ep{-\rbr{s-k}\fr{\eps^2T}{8n}} .
\end{align*}
\end{proof}

\begin{lemma}
\label{lemma:insufficientUE}
Suppose we run the uniform sampling on a $n$-armed bandit with a budget of $T$ and output the $s$ arms with the largest empirical rewards. Let $C_T(s,\eps)$ be the $\eps$-good arms in the output. Assume $s>g(\eps)$ with $g(\eps)$ being the number of $\eps$-good arms. Let $q$ be a parameter satisfying $g(\eps)<q<s$ and $d_{q}=\mu_{g(\eps)}-\mu_{q+1}$, then
\begin{align*}
    \pp{\abr{C_T(s,\eps)}\le g(\eps)-k}\le \cc{g\rbr{\eps}}{k}\ep{-k\fr{d_{q}^2T}{8n}}+\cc{n-q}{s-q+k}\ep{-\rbr{s-q+k}\fr{d_{q}^2T}{8n}}.
\end{align*}
Especially for Sequential Halving where $s=\fr{n}{2}$, choosing $q=\fr{n}{4}$ we have
\begin{align*}
    \pp{\abr{C_T\rbr{\fr{n}{2},\eps}}\le g(\eps)-k}\le \cc{g(\eps)}{k} \ep{-k\fr{d_{q}^2T}{8n}}+\cc{\fr{3n}{4}}{\fr{n}{4}+k}\ep{-\rbr{\fr{n}{4}+k}\fr{d_{q}^2T}{8n}}.
\end{align*}
\end{lemma}
\begin{proof}
The event
\begin{align*}
    \cbr{\abr{C_T(s,\eps)}\le g(\eps)-k}
\end{align*}
implies that there are at least $k$ $\eps$-good arms eliminated. Since $g(\eps)<q<s$, the $s$ output arms include at most $q-k$ arms that are from the best $q$ arms $\cbr{i\mid\mu_i\ge\mu_q}$. Such that, the $s$ output arms include at least $s-q+k$ arms that are from the worst $n-q$ arms $\cbr{i\mid\mu_i<\mu_q}$; otherwise, the number of outputs is even less than $s$. Denote the set of eliminated arms that are $\eps$-good as $A'$, and the set of output arms that are from the worst $n-q$ arms as $B'$, we have $\abr{A'}\ge k, \abr{B'}\ge s-q+k$. We then have the existence of $A, B$ satisfying
\begin{align*}
    \left\{\exists A\subset\sbr{g(\eps)}, \abr{A}=k,\exists B\subset\sbr{n}\backslash\sbr{q},\abr{B}=s-q+k,\text{s.t., }\forall i\in A, j\in B, \hmu_i\le\hmu_j\right\}.
\end{align*}
The above implication can also be seen as a generalization of the technique of \citet{karnin2013almost}. We then have
\begin{align*}
    \left\{\exists A\subset\sbr{g(\eps)}, \abr{A}=k,\exists B\subset\sbr{n}\backslash\sbr{q},\abr{B}=s-q+k,\text{s.t., }\forall i\in A, j\in B, \hmu_i\le\mu_i-\fr{d_{q}}{2}\text{ or }\hmu_j\ge\mu_j+\fr{d_{q}}{2}\right\}.
\end{align*}
By the distributive law of Boolean algebra, we have at least one of the following two holds,
\begin{align*}
    \left\{\exists A\subset\sbr{g(\eps)}, \abr{A}=k, \text{s.t., }\forall i\in A, \hmu_i\le\mu_i-\fr{d_{q}}{2}\right\},
\end{align*}
or
\begin{align*}
    \left\{\exists B\subset\sbr{n}\backslash\sbr{q},\abr{B}=s-q+k,\text{s.t., }j\in B, \hmu_j\ge\mu_j+\fr{d_{q}}{2}\right\}.
\end{align*}
Thereby the probability of $\cbr{\abr{C_T(s,\eps)}\le g(\eps)-k}$ can be upper bounded as
\begin{align}
    &\pp{\abr{C_T(s,\eps)}\le g(\eps)-k}\nn\\
    \le&\pp{\exists A\subset\sbr{g(\eps)}, \abr{A}=k, \text{s.t., }\forall i\in A, \hmu_i\le\mu_i-\fr{d_{q}}{2}}\nn\\
    &+\pp{\exists B\subset\sbr{n}\backslash\sbr{q},\abr{B}=s-q+k,\text{s.t., }j\in B, \hmu_j\ge\mu_j+\fr{d_{q}}{2}}\nn\\
    \le&\cc{g\rbr{\eps}}{k}\ep{-k\fr{\rbr{\fr{d_{q}}{2}}^2T}{2n}}+\cc{n-q}{s-q+k}\ep{-\rbr{s-q+k}\fr{\rbr{\fr{d_{q}}{2}}^2T}{2n}}\nn\\
    = &\cc{g\rbr{\eps}}{k}\ep{-k\fr{d_{q}^2T}{8n}}+\cc{n-q}{s-q+k}\ep{-\rbr{s-q+k}\fr{d_{q}^2T}{8n}}.\nn
\end{align}
\end{proof}

\subsection{Lower bound for uniform sampling}
\label{sec-uniform-lb}

\begin{theorem}
  Consider a family $\mathcal{E}(n,m,\epsilon)$ of all two-level univariate Gaussian bandit instances with $\blue{n}$ arms, of which $\blue{m}$ arms have mean $\blue{\epsilon} > 0$ and all other arms have to mean $0$. 
  We consider the problem in which samples are drawn from some instance $\theta \in \mathcal{E}(n,m,\epsilon)$ in an off-line, uniform fashion with $\blue{B}$ samples from each arm (i.e., with a total sampling budget of $\blue{T}$ samples, $B = T/n$ ignoring integer effects) and a subset of arms $\blue{\hS} \subseteq [n]$ of size $s$ is then chosen based only on the observed samples (i.e., without knowledge of $\theta$) in a symmetric fashion. Here `symmetric' is taken to mean that for any subset $S \subseteq [n]$, $\mathbb{P}_\theta (\hS = S) = \mathbb{P}_{\sig(\theta)} (\hS = \sig(S))$ where $\sig \in \blue{\Sig_n}$ is an element of the permutation group on sets of size $n$, and $\sig(S)$ is taken to mean the element-wise application of $\sig$ to the elements of $S$.\footnote{Note that the assumption of a symmetric selection $\hS$ is equivalent to considering a lower bound for general selections $\hS$ in the face of a uniform mixture over all permutations of $\theta$ (see \cite{Simchowitz2017a} for a more in-depth discussion)} 
  That is to say that the choice of $\hS$ is independent of the specific indices of the arms chosen.

  Suppose $m \le s \wed n/3$ and $s < \fr{6}{11}n$.
  Let $\tM_\th$ be the number of false negatives (i.e. $| | \ge $)
  Then, there exists $\th\in \cE(n,m,\eps)$ s.t.
  \begin{align*}
    \PP_\th(
    \tM_\th \ge \fr14 m
    ) \ge \min\del{\fr12, ~~ 2 \del{\fr{6}{5}\cd \fr{n-s}{s}}^{\fr{3}{16}m }    \exp\del{-\del{1+ 16\cd \fr{\ln(en/m)}{\ln\del{\fr{6}{5}\cd \fr{n-s}{s}}} }mB\eps^2} }
  \end{align*}
  where $\mathbb{P}_\theta$ is the probability measure induced by sampling from instance $\theta$.
\end{theorem}

\begin{proof}
Let $\blue{\th} = (\eps,\ldots,\eps,0,\ldots,0) \in \RR^n$ with $\eps>0$ be a vector of mean rewards for each arm where the first $m$ coordinates are nonzero.
Let $\blue{\hS} \in [n]$ be the output of the algorithm (recall $|\hS| = s$).
Let $\blue{\tM_\th}$ be the number of false negatives when taking $\hS$ as the prediction for the true support $[s]$ of $\th$. 
Let $\blue{\tQ_a} \subset 2^{[n]}$ be the collection of all subsets of $[n]$ of size $s$ such that $\tM_\th = a$.
For convenience, let $\blue{\gam} = \fr18 m$.
Let $k=\fr34 s$.
\begin{align*}
  \blue{\xi} := \PP_\th(\tM_\th \ge m-k)
  &= \sum_{a=m-k}^{k} \PP_\th(\tM_\th = a)
  \\&\ge \sum_{a=3\gam}^{5\gam} \PP_\th(\tM_\th = a) .
\end{align*}
Let $\blue{\first(\tQ_a)}$ be the first member of $\tQ_a$ in lexicographic order.
For example, $\first(tQ_a) = [a+1:a+s]$ where $[A:B] := \{A,A+1,\ldots,B\}$.
Then, by symmetry, one can see that $\PP_\th(\tM_\th = a) = |\tQ_a|\PP_\th(\hS=\first(\tQ_a))$.

We consider the event on which the `empirical KL-divergence' concentrates, which can be shown to be true with probability at least $1-\dt$ by \citet[Lemma 5]{jun20crush}:
\begin{equation*}
  \conc(\th',\th) := \left\{ \sum_{i=1}^n \sum_{t=1}^B \ln\left( \dfrac{p_{\theta^\prime}(X_{i,t} | A_t=i)}{p_\theta(X_{i,t} | A_t=i)}\right) -  (1+\rho) B \sum_{i=1}^n \textsf{KL}(\th'_i, \th_i)  \le \frac{1}{\rho} \ln(1/\delta)\right\}.
\end{equation*}

We now employ a change of measure argument to $\PP_\th(\hS=\first(\tQ_a))$ and switch $\th$ with $\th'$ that would result in $\tM_{\th'}=a - 3\gam =: \blue{b}$.
For $a \in[3\gam:5\gam]$, this can be achieved with the choice of $\th'=(0,\ldots,0,\eps,\ldots,\eps,0,\ldots,0)$ that is supported on $[3\gam+1:m+3\gam]$. 
Let $\Th$ be the set of permutations of $\th$.
Then,
\begin{align*}
  &\PP_\th(\tM_\th = a)
  \\ &= |\tQ_a|\PP_\th(\hS = \first(\tQ_a))
  \\&\ge |\tQ_a| \PP_{\th'} (\hS = \first(\tQ_a), \conc(\th',\th)) \exp(-(1+\rho) 2mB\cd(\eps^2/2) - \rho^{-1}\ln(1/\dt))  \tag{change of measure; $\rho>0$}
  \\&\ge |\tQ_a| \PP_{\th'} (\hS = \first(\tQ_a), \cap_{\sig\in\Sig} \conc(\th',\sig(\th')) ) \exp(-(1+\rho) mB \eps^2 - \rho^{-1}\ln(1/\dt)) \tag{$\Sig$: symmetric group of $[n]$}
  \\&\ge |\tQ_a| \PP_{\th} (\hS = \first(\tQ_b), \cap_{\sig\in\Sig} \conc(\th,\sig(\th))) \exp(-(1+\rho) mB \eps^2 - \rho^{-1}\ln(1/\dt)) \tag{symmetry}
  \\&= \fr{|\tQ_a|}{|\tQ_b|} \PP_{\th} (\tM_\th = b, \cap_{\sig\in\Sig} \conc(\th,\sig(\th))) \exp(-(1+\rho) mB \eps^2 - \rho^{-1}\ln(1/\dt)) \tag{symmetry}
  \\&\ge \fr{|\tQ_a|}{|\tQ_b|} 
  \del{\PP_{\th} (\tM_\th = b) - |\Theta|\dt } 
  \exp\del{-(1+\rho) mB \eps^2 - \rho^{-1}\ln(1/\dt)}~. \tag{$\PP(A,B) \ge \PP(A) - \PP(B^c)$; union bound over $\{\sig(\th): \sig\in\Sig\}$}
\end{align*}
Thus,
\begin{align*}
  &\sum_{a=3\gam}^{5\gam}\PP_\th(\tM_\th = a)
  \\&\ge \underbrace{\min_{a \in [3\gam:5\gam ]} \fr{|\tQ_a|}{|\tQ_{a-3\gam}|}}_{\tsty =: \blue{Y}}  \del{\underbrace{ \PP_\th(\tM_\th \in [0:2\gam])}_{\tsty \ge 1- \xi} - (2\gam+1)|\Th|\dt } \exp\del{-(1+\rho) mB \eps^2 - \rho^{-1}\ln(1/\dt)} ~.
\end{align*}
Let us choose $\blue{\dt} = \fr12 \cd \fr{1-\xi}{(2\gam + 1)|\Th|}$.
Since the LHS above is at most $\xi$, we have
\begin{align*}
  \xi \ge Y \fr{1-\xi}{2} \exp\del{-(1+\rho)mB\eps^2 - \rho^{-1} \ln\del{\fr{2(2\gam+1)|\Th|}{1-\xi} } }  
\end{align*}
One can consider two cases, namely $\xi\ge \fr12$ and $\xi<\fr12$, to arrive at
\begin{align*}
  \xi \ge \min\del{\fr12, ~~ \exp\del{-(1+\rho)mB\eps^2 - \rho^{-1} \ln\del{4(2\gam+1)|\Th|}  + \ln(4Y)} }
\end{align*}
It remains to find an appropriate value of $\rho$.
One simple choice is 
\begin{align*}
  \rho^{-1} = \fr12 \fr{\ln(4Y)}{ \ln\del{4(2\gam+1)|\Th|} } ~,
\end{align*}
which satisfies that $\rho^{-1} > 0$ as show later.
Thus, we have 
\begin{align*}
  \xi \ge \min\del{\fr12, ~~ 2 \sqrt{Y}  \exp\del{-\del{1+2\fr{\ln\del{4(2\gam+1)|\Th|}}{\ln(4Y)} }mB\eps^2} }
\end{align*}
It remains to figure out bounds for $Y$ and $|\Th|$.
For $Y$, note that $|\tQ_a| = \nck{m}{m-a} \nck{n-m}{s-(m-a)}$.
So, for $a \in [3\gam:5\gam]$ and $b = a - 3\gam$,
\begin{align*}
  \min_a\fr{|\tQ_a|}{|\tQ_b| } 
  &= \fr{\nck{m}{m-a} \nck{n-m}{s-m+a}}{\nck{m}{m-b} \nck{n-m}{s-m+b}} 
  \\&= \fr{(m-b)(m-b-1)\cdots(m-a+1)}{a(a-1)\cdots(b+1)} \cd \fr{(n-s-b)(n-s-b-1)\cdots(n-s-a+1) }{(s-m+a)(s-m+a-1) \cdots  (s-m+b+1)  }
  \\&\sr{(a)}{\ge} \del{\fr{m-b}{a} }^{a-b}  \cd \del{\fr{n-s- b}{s-m+a} }^{a-b}
  \\&\ge \del{\fr65 }^{3\gam}  \cd \del{\fr{n-s- 2\gam}{s-3\gam} }^{3\gam}
  ~\ge \del{\fr65 }^{3\gam}  \cd \del{\fr{n-s}{s} }^{3\gam}
  \\ \implies & Y=\min_{a\in[3\gam:5\gam] } \fr{|\tQ_a|}{|\tQ_{a-2\gam}| } \ge \del{\fr65 }^{3\gam}  \cd \del{\fr{n-s}{s} }^{3\gam}
\end{align*}
where $(a)$ is due to the fact that $\fr{m-b}{a}  > 1$ implies $(m-b-i)/(a-i) \ge (m-b)/a$ for $i\in[0:a-b-1]$ and, with a similar reasoning, $n \ge s/2 \implies \fr{n-s-a+3\gam}{s-m+a} > 1 \implies (n-s-b-i)/(s-m+a-i) \ge (n-s-b)/(s-m+a)$. 
Then, using $|\Th| = \nck{n}{m} \le \del{\fr{en}{m}}^m$, we have
\begin{align*}
  \rho = 2\fr{\ln\del{4(2\gam+1)|\Th|}}{\ln(4Y)}
  &\le 2\fr{\ln(m+4) + m \ln\del{\fr{en}{m}}}{\ln(4) + \fr{m}{4} \ln\del{\fr{6}{5}\cd \fr{n-s}{s}}  }  
  \\&\sr{(a)}{\le} 4\fr{ m \ln\del{\fr{en}{m}}}{\ln(4) + \fr{m}{4} \ln\del{\fr{6}{5}\cd \fr{n-s}{s}}  }  
  \\&\le 16\fr{ \ln(en/m)}{\ln\del{\fr{6}{5}\cd \fr{n-s}{s}}  }  
\end{align*}
where $(a)$ is by $m \le n/3 \implies \ln(m+4) \le m\ln(en/m)$.

Altogether,
\begin{align*}
  \PP_\th(\tM \ge \fr14 m) \ge \min\del{\fr12, ~~ 2 \del{\fr{6}{5}\cd \fr{n-s}{s}}^{\fr{3}{16}m }    \exp\del{-\del{1+ 16\cd \fr{\ln(en/m)}{\ln\del{\fr{6}{5}\cd \fr{n-s}{s}}} }mB\eps^2} }
\end{align*}
To verify that our choice of $\rho$ is nonnegative, it suffices to show that $\fr{6}{5}\cd \fr{n-s}{s} > 1$, which is true for $s < \fr{6}{11}n$.

\end{proof}

\subsection{A minimax optimal budget allocation scheme for SH}
\label{app-sec-1}
We show that with a different budget allocation scheme $T_\ell=\dr{\fr{T}{81n}\cdot\rbr{\frac{16}{9}}^{\ell-1}\cdot \ell}$, SH can achieve a sample complexity of $\OO{\fr{n}{\eps^2}\log\fr{1}{\delta}}$ (without any logarithmic terms of $n$). We call this budget allocation scheme Option 2 and the original one of \citet{karnin2013almost} Option 1. 
The price of achieving this exact minimax optimality is the potential loss of the near instance optimality (or at least the standard proof technique does not work).
We summarize the comparison in Table~\ref{tab:sh}. 
It is not clear to us whether there exists a different value of $T_\ell$ in SH that can achieve both the minimax and instance-dependent optimality. 
It is also not clear whether Option 2 can achieve a similar bound as Theorem \ref{them:ins_dep}.
Note that for the fixed budget setting, there is still a $\log\log(n)$ gap between the instance-dependent upper and lower bounds.
\begin{theorem}
For any $\eps\in(0,1)$, with Option $2$, there exists an absolute constant $c_1$, such that the error probability of SH for identifying an $\eps$-good arm satisfies,
\begin{align*}
\pp{\mu_{J_T}<\mu_1-\eps}\le\ep{-\Theta\rbr{\fr{\eps^2T}{n}}},
\end{align*}
for $T>c_1\cd\fr{n}{\eps^2}$.
\end{theorem}
\begin{proof}
Still let $\eps_1=\eps/4$. And define $\eps_{\ell+1}=\fr{3}{4}\cdot\eps_\ell$.  For each stage $\ell$, define the event $G_\ell$ as 
\begin{align*}
    G_\ell:=\cbr{\max_{i\in S_{\ell+1}}\mu_i\ge\max_{i\in S_\ell}\mu_i-\eps_\ell}.
\end{align*}
We have
\begin{align*}
\sum_{\ell=1}^{\log n}\eps_\ell<\sum_{\ell=1}^{\infty}\rbr{\fr{3}{4}}^{\ell-1}\cdot\eps_1=\fr{\eps}{4}\sum_{\ell=1}^{\infty}\rbr{\fr{3}{4}}^{\ell-1}\le\fr{\eps}{4}\lim_{n\rightarrow\infty}\fr{1-(3/4)^n}{1-3/4}=\eps.
\end{align*}
For stage $\ell$, we assign budget $T_\ell\cdot n2^{-(\ell-1)}:=\fr{T}{81n}\cdot\rbr{\frac{16}{9}}^{\ell-1}\cdot\ell\cdot n2^{-(\ell-1)}=\fr{T}{81}\cdot\rbr{\frac{8}{9}}^{\ell-1}\cdot\ell$. We can verify this allocation scheme does not exceed the budget as follows,
\begin{align*}
    \sum_{\ell=1}^{\log n}T_\ell\cdot n2^{-(\ell-1)}\le\sum_{\ell=1}^{\infty}\fr{T}{81}\cdot\rbr{\frac{8}{9}}^{\ell-1}\cdot\ell\le\fr{T}{9}\lim_{\ell\rightarrow\infty}\frac{1-\rbr{8/9}^\ell}{1-8/9}=T.
\end{align*}
Let $a_\ell$ be the best arm in $S_\ell$,
\begin{align*}
    \pp{G^c_\ell}&=\pp{G^c_\ell,\hat{\mu}_{a_\ell}<\mu_{a_\ell}-\eps_{\ell}/2}+\pp{G^c_\ell,\hat{\mu}_{a_\ell}\ge\mu_{a_\ell}-\eps_{\ell}/2}\\
    &\le\pp{\hat{\mu}_{a_\ell}<\mu_{a_\ell}-\eps_{\ell}/2}+\pp{G^c_\ell \mid \hat{\mu}_{a_\ell}\ge\mu_{a_\ell}-\eps_{\ell}/2}.
\end{align*}
For the first term 
\begin{align*}
    \pp{\hat{\mu}_{a_\ell}<\mu_{a_\ell}-\eps_{\ell}/2}\le&\exp\rbr{-\fr{\eps^2_{\ell}}{2}\fr{T_\ell}{|S_\ell|}}\\
    \le&\exp\rbr{-\fr{\eps_1^2\rbr{\frac{9}{16}}^{\ell-1}}{2}\fr{\fr{T}{81n}\cdot\rbr{\frac{16}{9}}^{\ell-1}\cdot\ell\cdot n2^{-(\ell-1)}}{n2^{-(\ell-1)}}}\\
    \le&\exp\rbr{-\fr{\eps^2T}{32\cdot81n}\ell}.
\end{align*}
For the second term, 
\begin{align*}
    \pp{G^c_\ell \mid \hat{\mu}_{a_\ell}\ge\mu_{a_\ell}-\eps_{\ell}/2}&\le\pp{|\{i\in S_\ell \mid \hat{\mu_i}>\mu_{i}+\eps_{\ell}/2\}|\ge |S_\ell|/2}\\
    &\le\fr{\ee{|\{i\in S_\ell \mid \hat{\mu_i}>\mu_{i}+\eps_{\ell}/2\}|}}{|S_\ell|/2}\\
    &\le\fr{|S_\ell|\exp\rbr{-\fr{\eps^2T}{32\cdot81n}\ell}}{|S_\ell|/2}\\
    &=2\exp\rbr{-\fr{\eps^2T}{32\cdot81n}\ell}.
\end{align*}
By contraposition,
\begin{align*}
    \pp{\mu_{J_T}<\mu_1-\eps}\le&\sum_{\ell=1}^{\log n}\pp{G_\ell^c}\\
    \le&\sum_{\ell=1}^{\infty}3\exp\rbr{-\fr{\eps^2T}{32\cdot81n}\ell}.
\end{align*}
Let $X:=\fr{\eps^2T}{32\cdot81n}>0$,
\begin{align*}
    \pp{\mu_{J_T}<\mu_1-\eps}&\le \sum_{\ell=1}^{\infty}3\exp\rbr{-\ell X}\\
    &\le3\exp\rbr{-X}\frac{1-\exp\rbr{-\ell X}}{1-\exp\rbr{-X}}\\
    &\le3\exp\rbr{-X}\frac{1}{1-\exp\rbr{-X}}.
\end{align*}
As $\frac{1}{1-\exp\rbr{-X}}$ is a monotonically decreasing function, there exists an absolute constant $c_1$, such that
\begin{align*}
    \pp{\mu_{J_T}<\mu_1-\eps}\le\ep{-\Theta\rbr{\fr{\eps^2T}{n}}},
\end{align*}
holds for $T>c_1\cd\fr{n}{\eps^2}$.
\end{proof}

\begin{table}[t]
\caption{The $\eps$-error probability of SH (worst-case) and $(\eps=0)$-error probability (instance-dependent) for Option 1 and 2 where ${H_2} := \max_i\fr{i}{\Delta_i^2}$ is the problem hardness parameter. Here, we omit constant factors.
Check marks for the worst-case means that it is optimal, and those for the instance-dependent means that it is optimal up to $\log\log(n)$ factors in the sample complexity \citep{mannor2004sample,Carpentier2016}. 
}
\centering
\begin{tabular}{|c|c|c|c|}
\hline
 & Worst-case & Instance-dependent \\ \hline
Option $1$: $T_\ell=\dr{\fr{T}{|S_\ell|\ur{\log_2 n}}}$ & \clr{\ding{55}} $\log_2 n\cd\exp\rbr{-\fr{\eps^2T}{n\log_2 n}}$ & \clg{\ding{51}} $\log_2n\cd\exp\rbr{-\fr{T}{\log_2nH_2}}$ \\ \hline
Option $2$: $T_\ell=\dr{\fr{T}{81n}\cdot\rbr{\frac{16}{9}}^{\ell-1}\cdot \ell}$ & \clg{\ding{51}} $\exp\rbr{-\fr{\eps^2T}{n}}$ & \clr{\ding{55}} $\log_2n\cd\exp\rbr{-\fr{T}{n\log_2nH_2}}$ \\ \hline
\end{tabular}%
\label{tab:sh}
\end{table}

\section{Simple Regret in the Data-Poor Regime}
\subsection{The number of opened brackets}

\begin{lemma}
\label{lem:numOfBrkt}
The number of opened brackets till round $t$, denoted by $L_t$, satisfies 
\begin{align*}
    0.63 \cd \log_2\rbr{1 + \ln(2) t} < L_t \le 1 + \log_2\rbr{1 + \ln(2) t}.
\end{align*}
The size of the largest bracket satisfies
\begin{align*}
    2^{L_t}\ge\hTT{t}.
\end{align*}
\end{lemma}
\begin{proof}
Recall that whenever $t \ge B2^B$ is true, we open a new bracket and increase $B$ by $1$. $B = 0$ before the game starts. First, note that for $t=1$ we have $L_t = 1$. Then, let us focus on $t \ge 2$. It is not hard to verify that
\begin{align*}
  L_t = \max\left\{B: t \ge (B-1) 2^{B-1}\right\}.
\end{align*}

Fix $t\ge2$. Then, we can define 
\begin{align}\label{eq:tprime}
  t' := (L_t - 1) 2^{L_t - 1},
\end{align}
the first time point that the bracket $L_t$ was opened. Clearly $t' \le t$. Solve \eqref{eq:tprime} for $L_t$.
\begin{align*}
  t' &= (L_t - 1) \exp\rbr{ \ln(2) (L_t-1)},
\\  \ln(2) t' &= \ln(2) (L_t - 1) \exp\rbr{ \ln(2) (L_t-1)}
\\            &=: X\exp(X).
\end{align*}
Solving it for $X$ is exactly the Lambert W function.
By Lemma 17 of \citet{orabona2016coin}, we have
\begin{align*}
  0.6321\cdot \ln\rbr{1+\ln(2) t'} \le X \le \ln\rbr{1+\ln(2) t'}.
\end{align*}
This implies that
\begin{align*}
  X=\ln(2) (L_t - 1) \le \ln\rbr{1 + \ln(2) t'} ~\le \ln\rbr{1 + \ln(2)t}
  \implies   L_t \le 1 + \log_2\rbr{1 + \ln(2) t}.
\end{align*}
To obtain a lower bound, define
\begin{align*}
  t'' := L_t 2^{L_t},
\end{align*}
which is strictly larger than $t$. Using similar techniques as above, we can derive
\begin{align*}
   L_t \ge 0.6321\cdot \log_2\rbr{1 + \ln(2) t''} > 0.63\cdot \log_2\rbr{1 + \ln(2) t}.
\end{align*}
Together, we have
\begin{align*}
   0.63 \cdot \log_2\rbr{1 + \ln(2) t} < L_t \le 1 + \log_2\rbr{1 + \ln(2) t}.
\end{align*}
Note that we have 
\begin{align*}
    L_t2^{L_t}\le t\le(L_t+1)2^{L_t+1}.
\end{align*}
Thus 
\begin{align*}
    2^{L_t}\ge \frac{t}{2(L_t+1)}\ge \frac{t}{2\rbr{2 + \log_2\rbr{1 + \ln(2) t}}}\ge\frac{t}{4\log_2t}=\hTT{t}.\tag{$t\ge4$}
\end{align*}
\end{proof}

\subsection{The number of pulls for the representative arms}
\label{app-sec-2}
\begin{lemma}
\label{lem:numOfpull}
At round $t$, the number of times the representative arm of bracket $B$ has been pulled in the latest finished SH, denoted as $D_{B,t}$, satisfies,
\begin{align*}
D_{B,t}\ge\textup{const}\cd\fr{t}{\ln (t)\cd\log_2(n)}.
\end{align*}
\end{lemma}
\begin{proof}
We call the round interval $(L_t-1)2^{L_t-1}\leq t<L_t2^{L_t}$ as \textit{phase} $L_t$, whose length is $L_t2^{L_t}-(L_t-1)2^{L_t-1}=(L_t+1)2^{L_t-1}$. Correspondingly, sampling budget of $\frac{(L_t+1)2^{L_t-1}}{L_t}$ is assigned to each bracket. For each opened bracket $A_B, B\in[L_t-1]$, the sampling budget for it is accumulated from phase $B$ to phase $L_t$, while phase $L_t$ is not finished yet. The sampling budget bracket $A_B$ receives satisfies
\begin{align}\label{eq:budget}
T_{B,t} = & \sum_{i=B}^{L_t-1}\frac{2^{i-1}(i+1)}{i}+\frac{t-2^{L_t-1}(L_t-1)}{L_t} \nonumber\\
> & \frac{2^{L_t-2}L_t}{L_t-1}+\frac{t-2^{L_t-1}(L_t-1)}{L_t} \nonumber\\
> & \frac{t}{4L_t} \nonumber\\
> & \frac{t}{4 + 4\log\rbr{1 + \ln(2)t}}> \frac{t}{4\ln t}.\tag{$t>16$}
\end{align}
Recall we run the SH with the doubling trick on each bracket individually. We claim that the latest finished SH receives a budget of at least $T_{B,t}/4$. To see this, note that we initialize Algorithm \ref{Alg:DSH} with budget starting from $\ur{n\log_2n}$. Then restart the SH algorithm with budget $2\cd\ur{n\log_2n}$, and so on. Before finishing the $k$-th doubling trick, the best arm from $(k-1)$-th doubling trick is output. Thus the portion of the lasted finished SH ranges in 
\begin{equation*}
\left[\frac{\abr{\cT_{k-1}}}{\sum_{i=1}^{k}\abr{\cT_{i}}}, \frac{\abr{\cT_{k-1}}}{\sum_{i=1}^{k-1}\abr{\cT_{i}}}\right),
\end{equation*}
which, by some simple calculations, is 
\begin{equation}\label{1/4budget}
\left[\frac{b_{r-1}}{4b_{r-1}-b_0}, \frac{b_{r-1}}{2b_{r-1}-b_0}\right).
\end{equation}
That is to say, the most recently finished SH, which gives the current output, has a budget of at least $T_{B,t}/4$. We notice that $T_{B,t}$ is essentially irrelevant to $B$. This means each bracket receives an order-wise equal amount of budget to query the SH algorithm. In the SH algorithm, the output arm is pulled a fixed number of times when the number of arms and the sampling budget is fixed. It is pulled $\Theta\rbr{\fr{T}{\log_2n}}$ times for budget $T$. We have all we need to determine the lower bound for $D_{B,t}$.
\begin{align*}
D_{B,t}\ge\text{const}\cd\fr{t}{\ln (t)\cd\log_2(n)}.
\end{align*}
\end{proof}

\subsection{Lower bounds and their implications in special instances}
\label{app-sec-lb}
While \citet{Katz20} shows a lower bound for the $(\eps,\delta)$-unverifiable sample complexity that can match the upper bound in special instances, the final result they presented has a negative term, which results in looseness, and the bound can even go to 0 in certain instances.
Recall the definition of $\tau_{\eps,\dt}$, the $(\eps,\delta)$-unverifiable sample complexity (Definition~\ref{def-uver}).
\newtheorem*{supp-6}{Theorem~1 of \citet{Katz20}}
\begin{theorem}[Theorem 1 of \citet{Katz20}]
\label{them:kjlow}
Fix $\eps>0,\delta\in(0,1/16)$, and a vector $\mu\in\mathbb{R}^n$. Consider $n$ arms where rewards from the $i$-th arm are distributed according to $\mathcal{N}(\mu_i,1)$. For every permutation $\pi\in\mathbb{S}^n$, let $(\mathcal{F}_t^\pi)_{t\in\mathbb{N}}$ be the filtration generated by the algorithm playing on instance $\pi(\mu)$. Then the $(\eps,\delta)$-unverifiable sample complexity satisfies, 
\begin{align*}
    \ee{\tau_{\eps,\delta}}\ge\fr{1}{64}\rbr{-\left(\mu_{1}-\mu_{g(\eps)+1}\right)^{-2}+\frac{1}{g(\eps)} \sum_{i=g(\eps)+1}^n\left(\mu_{1}-\mu_{i}\right)^{-2}},
\end{align*}
where the expectation is with respect to $\pi\in\mathbb{S}^n$ and $\pi(\mu)$.
\end{theorem}

In fact, the proof of \citet[Theorem 1]{Katz20} shows a stronger bound, which is more useful for deriving lower bounds for specific instances.
\begin{theorem}[A stronger version of Theorem 1 of \citet{Katz20}]
\label{low_fixed}
Fix $\eps>0,\delta\in(0,1/16)$, and a vector $\mu\in\mathbb{R}^n$. Consider $n$ arms where rewards from the $i$-th arm are distributed according to $\mathcal{N}(\mu_i,1)$. For every permutation $\pi\in\mathbb{S}^n$, let $(\mathcal{F}_t^\pi)_{t\in\mathbb{N}}$ be the filtration generated by the algorithm playing on instance $\pi(\mu)$. Then the $(\eps,\delta)$-unverifiable sample complexity satisfies, 
\begin{align*}
    \ee{\tau_{\eps,\delta}}\ge\fr{1}{64}\frac{1}{g(\eps)} \sum_{i=2g(\eps)+1}^n\left(\mu_{1}-\mu_{i}\right)^{-2},
\end{align*}
where the expectation is with respect to $\pi\in\mathbb{S}^n$ and $\pi(\mu)$.
\end{theorem}

\begin{corollary}
For the EqualGap$(m)$ instance, Proposition \ref{low_fixed} shows a lower bound of the $(\eps,\delta)$-unverifiable sample complexity as
\begin{align*}
    \ee{\tau_\pi}\ge&\fr{1}{100}\rbr{\fr{n}{m}\frac{1}{\eps^2}-\frac{2}{\eps^2}}.
\end{align*}
\end{corollary}
\begin{proof}
\begin{align*}
    \fr{1}{64}\frac{1}{g(\eps)} \sum_{i=2g(\eps)+1}^n\left(\mu_{1}-\mu_{i}\right)^{-2}&=\fr{1}{64}\fr{1}{m}\sum_{i=2m+1}^{n}\frac{16}{25\eps^2}\\
    &=\fr{1}{64}\fr{n-2m}{m}\frac{16}{25\eps^2}\\
    &=\fr{1}{100}\rbr{\fr{n}{m}\frac{1}{\eps^2}-\frac{2}{\eps^2}}.
\end{align*}
\end{proof}
\begin{corollary}
For the Polynomial$(\alpha)$ instance, Proposition \ref{low_fixed} shows a lower bound of the $(\eps,\delta)$-unverifiable sample complexity as
\begin{align*}
    \ee{\tau_\pi}\ge&\fr{1}{128}\fr{1}{2\alpha-1}\rbr{2^{1-2\alpha}\eps^{-2}-\eps^{-\fr{1}{\alpha}}}.
\end{align*}
\end{corollary}
\begin{proof}
For the Polynomial$(\alpha)$ instance, 
\begin{align*}
    \fr{1}{64}\frac{1}{g(\eps)} \sum_{i=2g(\eps)+1}^n\left(\mu_{1}-\mu_{i}\right)^{-2}&=\fr{1}{64}\fr{1}{g(\eps)}\sum_{i=2g(\eps)+1}^{n}\rbr{\fr{i}{n}}^{-2\alpha}\\
    &=\fr{1}{64}\fr{n^{2\alpha}}{n\eps^{1/\alpha}}\sum_{i=2g(\eps)+1}^{n}\rbr{\fr{1}{i}}^{2\alpha}\\
    &=\fr{1}{64}\fr{n^{2\alpha-1}}{\eps^{1/\alpha}}\sum_{i=2g(\eps)+1}^{n}\rbr{\fr{1}{i}}^{2\alpha}.
\end{align*}

The summation of $\sum_{i=2g(\eps)+1}^{n}\rbr{\fr{1}{i}}^{2\alpha}$ can be bounded as
\begin{align*}
    \sum_{i=2g(\eps)+1}^{n}\rbr{\fr{1}{i}}^{2\alpha}
    &\stackrel{(a_1)}{\ge}\frac{1}{2}\int_{2g(\eps)}^n\rbr{\fr{1}{i}}^{2\alpha}di\\
    &=\frac{1}{2}\rbr{f\rbr{n}-f\rbr{2g(\eps)}}\tag{$f(x)=\fr{1}{-2\alpha+1}x^{-2\alpha+1}$}\\
    &=\frac{1}{2}\rbr{\fr{1}{-2\alpha+1}n^{-2\alpha+1}-\fr{1}{-2\alpha+1}(2g(\eps))^{-2\alpha+1}}\\
    &=\frac{1}{2}\rbr{\fr{1}{-2\alpha+1}n^{-2\alpha+1}-\fr{1}{-2\alpha+1}(2n\eps^{1/\alpha})^{-2\alpha+1}},
\end{align*}
where $(a_1)$ is by the geometric meaning of integral and
\begin{align*}
\frac{1}{2}\rbr{\fr{1}{i}}^{2\alpha}=\rbr{\fr{1}{2^{\frac{1}{2\alpha}}i}}^{2\alpha}\le\rbr{\fr{1}{2i}}^{2\alpha}\le\rbr{\fr{1}{i+1}}^{2\alpha}.\tag{$i\ge1,\alpha>0.5$}
\end{align*}

Thus 
\begin{align*}
    \fr{1}{64}\frac{1}{g(\eps)} \sum_{i=2g(\eps)+1}^n\left(\mu_{1}-\mu_{i}\right)^{-2}&=\fr{1}{64}\fr{n^{2\alpha-1}}{\eps^{1/\alpha}}\sum_{i=2g(\eps)+1}^{n}\rbr{\fr{1}{i}}^{2\alpha}\\
    &\ge\fr{1}{128}\fr{n^{2\alpha-1}}{\eps^{1/\alpha}}\rbr{\fr{1}{-2\alpha+1}n^{-2\alpha+1}-\fr{1}{-2\alpha+1}(2n\eps^{1/\alpha})^{-2\alpha+1}}\\
    &=\fr{1}{128}\fr{1}{\eps^{1/\alpha}}\rbr{\fr{1}{-2\alpha+1}-\fr{2^{1-2\alpha}}{-2\alpha+1}\eps^{\fr{-2\alpha+1}{\alpha}}}\\
    &=\fr{1}{128}\eps^{-\fr{1}{\alpha}}\fr{1}{2\alpha-1}\rbr{2^{1-2\alpha}\eps^{\fr{-2\alpha+1}{\alpha}}-1}\\
    &=\fr{1}{128}\fr{1}{2\alpha-1}\rbr{2^{1-2\alpha}\eps^{-2}-\eps^{-\fr{1}{\alpha}}}.
\end{align*}
\end{proof}

\subsection{Proof of Theorem \ref{them:datapoor}}
\label{app-sec-them_datapoor}
\newtheorem*{datapoor}{Theorem~\ref{them:datapoor}}
\begin{datapoor}
For any $\eps\in(0,1)$, the $\eps$-error probability of BSH satisfies
\begin{align*}
    \pp{\mu_1-\mu_{J_t}>\eps}\le\ep{-\hTT{\fr{t}{ 
    \frac{1}{\eps^2} + 
    \max_{i\ge g\rbr{\fr{\eps}{2}}+1}\frac{i}{\Delta^{2}_i}\cd\frac{1}{g\rbr{\fr{\eps}{4}}}
    }}  
    },
\end{align*}
for $t\ge\hTT{H_2\rbr{\fr{\eps}{2}}}$.
\end{datapoor}

\begin{proof}
Due to the doubling nature of the bracketing scheme, it is not hard to see that each bracket receives an order-wise equal amount of sampling budget $\hTT{t}$ (details in Appendix~\ref{app-sec-2}).
Our analysis is centered around finding the ideal bracket whose representative arm is expected to be of the highest quality, which we call the \textit{best bracket}. Intuitively, on the one hand, if the best bracket is too large, the per-arm sampling budget will be too small to guarantee a meaningful output. On the other hand, if the best bracket is too small, it may not well represent the entire instance. The core idea of the proof is balancing this trade-off. 

We define the following notations.
\begin{itemize}
    \item $A^*$: the best bracket, which is a multiset satisfying $\forall i\in A^*, i\in[n]$. We have $\abr{A^*}=L$ with $L$ being specified later.
    \item $r=\fr{L}{n}$: the subsampling ratio of the best bracket $A^*$.
    \item $\eps_k:=2^{k-1}\cd\fr{\eps}{4}$, for $k\in[K]$, where $K=\urd{\log_2\rbr{\fr{8}{\eps}}}$. Also $K$ satisfies $\eps_{K-1}\le1, \eps_{K}\ge1$. 
    \item $Z_k=\cbr{i\in[n]\mid\Delta_i\in\left[0,\min\cbr{\eps_{k}, 1}\right]}$.
    \item $g(\eps_{k})=\abr{Z_k}$.
    \item $Z'_k=\cbr{i\in A^*\mid\Delta_i\in\left[0,\min\cbr{\eps_{k}, 1}\right]}$.
    \item $g'(\eps_{k})=\abr{Z'_k}$.
    \item $\mu_i(A)$: the mean reward of the $i$-th best arm in $A$.
    \item $\Delta'_{i}=\max\cbr{\mu_j(A^*)\mid j\in A^*}-\mu_i(A^*)$: the suboptimality gap of $i$-th best arm in $A^*$ with respect to the best arm in $A^*$.
    \item $\Delta''_{i}=\mu_1-\mu_i(A^*)$: the suboptimality gap of $i$-th best arm in $A^*$ with respect to the best arm in the entire instance.
    \item $U_1\rbr{\fr{\eps}{2}}=\max_{i\ge g\rbr{\eps_2}+1}\fr{i}{\Delta_i^2}$, and $U_2\rbr{\fr{\eps}{2}}=\max_{i\ge g'\rbr{\eps_2}+1}\fr{i}{\Delta^{'2}_{i}}$. We call this the un-accelerated sample complexity on the entire instance and the best bracket for identifying an $\fr{\eps}{2}$-good arm.
\end{itemize}
We define the following events:
\begin{itemize}
    \item $E_1$: $\forall k\in[K], \abr{\fr{g'(\eps_{k})}{L}-\frac{g(\eps_{k})}{n}}<\fr{g(\eps_{k})}{2n}$.
    \item $E_2$: an $\fr{\eps}{2}$-good arm represents the best bracket.
\end{itemize}
The event $E_1$ ensures that the best bracket can well represent the entire arm set in the sense that, for each peeling piece $k$, the ratio of the arms therein to all the arms in the best bracket $\fr{g'(\eps_{k})}{L}$ is close to its original ratio $\frac{g(\eps_{k})}{n}$. The $\eps$-error probability can now be bounded as 
\begin{align}
\pp{\mu_{J_t}<\mu_1-\eps}=&\pp{\mu_{J_t}<\mu_1-\eps, E^c_1}+\pp{\mu_{J_t}<\mu_1-\eps,E_1,E^c_2}+\pp{\mu_{J_t}<\mu_1-\eps,E_1,E_2}\nn\\
\le&\pp{E^c_1}+\pp{E_1,E^c_2}+\pp{\mu_{J_t}<\mu_1-\eps,E_2}\nn\\
\le&\pp{E^c_1}+\pp{E^c_2\mid E_1}+\pp{\mu_{J_t}<\mu_1-\eps\mid E_2}.\label{bracketing}
\end{align}
We bound each term separately. For the first term.
\begin{align*}
    \pp{E_1^c}=&\pp{\exists k, \abr{\fr{g'(\eps_{k})}{L}-\frac{g(\eps_{k})}{n}}>\fr{g(\eps_{k})}{2n}}\\
    \le&\sum_{k=1}^{K}\pp{\abr{\fr{g'(\eps_{k})}{L}-\frac{g(\eps_{k})}{n}}>\fr{g(\eps_{k})}{2n}}\\
    \le&\sum_{k=1}^{K}\pp{\fr{g'(\eps_{k})}{L}<\frac{g(\eps_{k})}{2n} \text{ or }\fr{g'(\eps_{k})}{L}>\fr{3g(\eps_{k})}{2n}}\\
    \le&\sum_{k=1}^{K}\pp{\fr{g'(\eps_{k})}{L}<\frac{g(\eps_{k})}{2n}}+\sum_{k=1}^{K}\pp{\fr{g'(\eps_{k})}{L}>\fr{3g(\eps_{k})}{2n}}.
\end{align*}
Let $\textsf{KL}(p,q)$ be the Kullback-Leibler divergence between two Bernoulli distributions with parameters $p, q$. The first summation can be upper bounded by the additive Chernoff bound for binomial distribution \citep{arratia1989tutorial},
\begin{align*}
    \sum_{k=1}^{K}\pp{\fr{g'(\eps_{k})}{L}<\frac{g(\eps_{k})}{2n}}\le&\sum_{k=1}^{K}\ep{-L\textsf{KL}\rbr{\fr{g(\eps_{k})}{2n},\frac{g(\eps_{k})}{n}}}\\
    \le&\sum_{k=1}^{K}\ep{-\frac{1}{8}L\frac{\fr{g^{2}(\eps_{k})}{n^2}}{\fr{g(\eps_{k})}{n}}}\tag{$\textsf{KL}\rbr{p,q}\ge\fr{(p-q)^2}{2\max\{p,q\}}$}\\
    =&\sum_{k=1}^{K}\ep{-\frac{1}{8}L\frac{g(\eps_{k})}{n}}.
\end{align*}
The second summation can be upper bounded by 
\begin{align*}
    \sum_{k=1}^{K}\pp{\fr{g'(\eps_{k})}{L}>\frac{3g(\eps_{k})}{2n}}\le&\sum_{k=1}^{K}\ep{-L\textsf{KL}\rbr{1-\fr{3g(\eps_{k})}{2n},1-\frac{g(\eps_{k})}{n}}}\\
    \stackrel{(a_1)}{\le}&\sum_{k=1}^{K}\ep{-L\textsf{KL}\rbr{\fr{3g(\eps_{k})}{2n},\frac{g(\eps_{k})}{n}}}\\
    \le&\sum_{k=1}^{K}\ep{-\frac{1}{12}L\frac{\fr{g^2(\eps_{k})}{n^2}}{\fr{g(\eps_{k})}{n}}}\tag{$\textsf{KL}\rbr{p,q}\ge\fr{(p-q)^2}{2\max\{p,q\}}$}\\
    \le&\sum_{k=1}^{K}\ep{-\frac{1}{12}L\frac{g(\eps_{k})}{n}},
\end{align*}
where $(a_1)$ is by the definition of KL divergence for Bernoulli distribution.
Since $g(\eps_{k})$ is a non-decreasing function of $k$, we have 
\begin{align}
    \pp{E_1^c}\le&\sum_{k=1}^{K}\ep{-\frac{1}{8}L\frac{g(\eps_{k})}{n}}+\sum_{k=1}^{K}\ep{-\frac{1}{12}L\frac{g(\eps_{k})}{n}}\nn\\
    \le&2\urd{\log_2\rbr{\fr{8}{\eps}}}\ep{-\frac{1}{12}L\frac{g(\eps_{1})}{n}}.\nn
\end{align}
We choose $L=
\hTT{\fr{nt}{g(\eps_{1})H_2\rbr{\fr{\eps}{2}}}}$. Then
\begin{align}
    \pp{E_1^c}\le&2\urd{\log_2\rbr{\fr{8}{\eps}}}\ep{-\hTT{\fr{t}{H_2\rbr{\fr{\eps}{2}}}}}\nn\\
    \le&\ep{-\hTT{\fr{t}{H_2\rbr{\fr{\eps}{2}}}}}.\tag{$t\ge\hTT{H_2\rbr{\fr{\eps}{2}}}$}\nn\\\label{ieq:brk1}
\end{align}
Our choice of $L$ here also ensures that a bracket of size $L$ is opened because the largest bracket is in the size of $\hTT{t}$ by Lemma \ref{lem:numOfBrkt}.
Furthermore, 
\begin{align*}
    L=\hTT{\fr{nt}{g(\eps_{1})H_2\rbr{\fr{\eps}{2}}}}=\hTT{\fr{nt}{\max_{i\ge g\rbr{\eps_2}+1}\fr{i}{\Delta_i^2}}}\le\hTT{\fr{nt}{n/\Delta_n^2}}=\hTT{\Delta_n^2t}\le\hTT{t}.
\end{align*}
The second term of \eqref{bracketing} means the probability that SH fails to return an $\fr{\eps}{2}$-good arm from bracket $A^*$ given that bracket $A^*$ can well represent the entire instance. Denote
$H'_2\rbr{\fr{\eps}{2}}$ as the sample complexity of returning an $\fr{\eps}{2}$-good from the bracket $A^*$, i.e.,
\begin{align*}
    H'_2\rbr{\fr{\eps}{2}}=\max_{i\ge g'(\eps_{2})+1}\frac{i}{\Delta^{'2}_i}\cd\frac{1}{g'(\eps_{1})}.
\end{align*}
Let $H_2\rbr{\fr{\eps}{2}}$ be the sample complexity of returning an $\fr{\eps}{2}$-good arm from the entire instance, i.e.,
\begin{align*}
    H_2\rbr{\fr{\eps}{2}}=\max_{i\ge g(\eps_{2})+1}\frac{i}{\Delta^{2}_i}\cd\frac{1}{g(\eps_{1})}.
\end{align*}
We use Theorem \ref{them:ins_dep} to upper bound the second term of \eqref{bracketing}. To do so, we need to verify that the best bracket is allocated to enough budget, i.e., the budget allocated to the best bracket should be no less than $U_2\rbr{\fr{\eps}{2}}$. By event $E_1$, we have
\begin{align}
\fr{1}{2}rg(\eps_{k})\le g'(\eps_{k})\le 2rg(\eps_{k}).
\end{align} 
Then, by Lemma \ref{lemma:sameComplexity} and our choice of $L$, we have
\begin{align*}
    U_2\rbr{\fr{\eps}{2}}=g'(\eps_{1})H'_2\rbr{\fr{\eps}{2}}\le 2rg(\eps_{1})H'_2\rbr{\fr{\eps}{2}}=2\fr{L}{n}g(\eps_{1})H'_2\rbr{\fr{\eps}{2}}=\hTT{\fr{H'_2\rbr{\fr{\eps}{2}}}{H_2\rbr{\fr{\eps}{2}}}t}=\hTT{t}.
\end{align*}
Meanwhile, each bracket receives $\hTT{t}$ budget by the bracketing design. Thus the best bracket is allocated to enough budget.
By applying Theorem \ref{them:ins_dep} to $(a_1)$ and Lemma \ref{lemma:sameComplexity} to $(a_2)$ below, we have
\begin{align}
    \pp{E^c_2\mid E_1}\stackrel{(a_1)}{\le}\ep{-\hTT{\fr{t}{H'_2\rbr{\fr{\eps}{2}}}}}\stackrel{(a_2)}{\le}\ep{-\hTT{\fr{t}{H_2\rbr{\fr{\eps}{2}}}}}.\label{ieq:brk2}
\end{align}
The third term of \eqref{bracketing} means an arm whose mean reward is less than $\mu_1-\eps$ has a larger empirical reward than an arm whose mean reward is larger than $\mu_1-\eps/2$. Note the recently opened brackets may not have finished their first SH yet. In this case, the DSH always returns an empirical reward of negative infinity. Thus the arms returned from these brackets will never be selected by BSH. Denote $\dot{A}(t):=\cbr{a\in\bigcup_{k=1,A_k\neq A^*}^{L_t}A_{k},\mu_a<\mu_1-\eps}$.
We have 
\begin{align*}
    \abr{\dot{A}(t)}\le&\sum_{k=1}^{L_t}2^k\le2\cd\rbr{2^{1 + \log_2\rbr{1 + \ln(2) t}}-1}\le8\cd\rbr{1 + \ln(2) t}.
\end{align*}
By the number of arm pulls in Lemma \ref{lem:numOfpull},
\begin{align}
    \pp{\mu_{J_t}<\mu_1-\eps\mid E_2}\le&\pp{\exists a_1\in A_{r_i^*(t)},\exists a_2\in\dot{A}(t), \text{s.t.},\mu_{a_1}\ge\mu_1-\eps/2,\hat{\mu}_{a_1}<\hat{\mu}_{a_2}}\nn\\
    \le&\abr{\dot{A}(t)}\abr{A_{r_i^*(t)}}\ep{-\hTT{\eps^2t}}\nn\\
    \le&\ep{-\hTT{\eps^2t}+\ln\rbr{\abr{\dot{A}(t)}\abr{A_{r_i^*(t)}}}}\nn\\
    \le&\ep{-\hTT{\eps^2t}+\ln\rbr{8\cd\rbr{1 + \ln(2) t}t}}\nn\\
    \le&\ep{-\hTT{\eps^2t}+\OO{\ln t}}\nn\\
    \le&\ep{-\tilde{\Theta}\rbr{\eps^2t}}.\label{ieq:brk3}
\end{align}
Given \eqref{ieq:brk1}, \eqref{ieq:brk2} and \eqref{ieq:brk3}, we have the result proved. 
\end{proof}

\begin{lemma}
\label{lemma:sameComplexity}
Let $H'_2\rbr{\fr{\eps}{2}}$ be the complexity
measure of returning an $\fr{\eps}{2}$-good from the bracket $A^*$, i.e.,
\begin{align*}
    H'_2\rbr{\fr{\eps}{2}}=\max_{i\ge g'(\eps_{2})+1}\frac{i}{\Delta^{'2}_i}\cd\frac{1}{g'(\eps_{1})}.
\end{align*}
Let $H_2\rbr{\fr{\eps}{2}}$ be the complexity
measure of returning an $\fr{\eps}{2}$-good from the entire instance, i.e.,
\begin{align*}
    H_2\rbr{\fr{\eps}{2}}=\max_{i\ge g(\eps_{2})+1}\frac{i}{\Delta^{2}_i}\cd\frac{1}{g(\eps_{1})}.
\end{align*}
Given the event $E_1$ happens, we have
\begin{align*}
H'_2\rbr{\fr{\eps}{2}}\le 64H_2\rbr{\fr{\eps}{2}}.
\end{align*}
\end{lemma}
\begin{proof}
For the entire instance, by Lemma \ref{lemma:UASapprox},
\begin{align}
    \max_{k\in\cbr{3,\cdots,K}}\fr{g(\eps_{k})}{\eps_{k}^2}\le U_1\rbr{\fr{\eps}{2}}\le\max_{k\in\cbr{3,\cdots,K}}\fr{4g(\eps_{k})}{\eps_{k}^2}. \label{ieq:u1}
\end{align}
For the best bracket, by Lemma \ref{lemma:twoGapsApprox}, we have
\begin{align*}
    \max_{i\ge g'\rbr{\eps_2}+1}\fr{i}{\Delta^{''2}_{i}}\le\max_{i\ge g'\rbr{\eps_2}+1}\fr{i}{\Delta^{'2}_{i}}\le4\max_{i\ge g'\rbr{\eps_2}+1}\fr{i}{\Delta^{''2}_{i}}.
\end{align*}
Note the term $\max_{i\ge g'\rbr{\eps_2}+1}\fr{i}{\Delta^{''2}_{i}}$ can be considered the un-accelerated sample complexity on a virtual instance where all the arms are the same as the ones of $A^*$, except the best arm of $A^*$ is replaced as the best arm in the entire instance (if the best arm of $A^*$ is same as the best arm in the entire instance, then the virtual instance is same as $A^*$). The above inequalities essentially show that on event $E_1$, setting the best arm in $A^*$ to have the same mean reward as the best arm of the entire instance does not change its un-accelerated sample complexity (up to a constant of 4). Denote 
\begin{align*}
    U'_2\rbr{\fr{\eps}{2}}=\max_{i\ge g'\rbr{\eps_2}+1}\fr{i}{\Delta^{''2}_{i}}.
\end{align*}
Thus, 
\begin{align}
    U'_2\rbr{\fr{\eps}{2}}\le U_2\rbr{\fr{\eps}{2}}\le 4U'_2\rbr{\fr{\eps}{2}}.\label{ieq:u12}
\end{align}
By Lemma \ref{lemma:UASapprox},   
\begin{align*}
    \max_{k\in\cbr{3,\cdots,K}}\fr{g'(\eps_{k})}{\eps_{k}^2}\le U'_2\rbr{\fr{\eps}{2}}\le\max_{k\in\cbr{3,\cdots,K}}\fr{4g'(\eps_{k})}{\eps_{k}^2}.
\end{align*}
By event $E_1$, we have for $k\in[K]$
\begin{align}
\fr{1}{2}rg(\eps_{k})\le g'(\eps_{k})\le 2rg(\eps_{k}).\label{ieq:gr}
\end{align}
Thus 
\begin{align}
    \max_{k\in\cbr{3,\cdots,K}}\fr{rg(\eps_{k})}{2\eps_{k}^2}\le U'_2\rbr{\fr{\eps}{2}}\le\max_{k\in\cbr{3,\cdots,K}}\fr{8rg(\eps_{k})}{\eps_{k}^2}.\label{ieq:u'2}
\end{align}
By \eqref{ieq:u12} and \eqref{ieq:u'2},
\begin{align}
    \max_{k\in\cbr{3,\cdots,K}}\fr{rg(\eps_{k})}{2\eps_{k}^2}\le U_2\rbr{\fr{\eps}{2}}\le\max_{k\in\cbr{3,\cdots,K}}\fr{32rg(\eps_{k})}{\eps_{k}^2}.\label{u2}
\end{align}
Therefore, by \eqref{ieq:u1} and \eqref{u2}, we have
\begin{align}
\fr{1}{32r}\le\fr{U_1\rbr{\fr{\eps}{2}}}{U_2\rbr{\fr{\eps}{2}}}\le\fr{8}{r}.\label{ieq:u12r}
\end{align}
By $U_1\rbr{\fr{\eps}{2}}=g(\eps_{1})H_2\rbr{\fr{\eps}{2}}$, $U_2\rbr{\fr{\eps}{2}}=g'(\eps_{1})H'_2\rbr{\fr{\eps}{2}}$, \eqref{ieq:gr} and \eqref{ieq:u12r}, we have
\begin{align*}
    H'_2\rbr{\fr{\eps}{2}}=\fr{g(\eps_{1})U_2\rbr{\fr{\eps}{2}}}{g'(\eps_{1})U_1\rbr{\fr{\eps}{2}}}H_2\rbr{\fr{\eps}{2}}\le\frac{2}{r}\cd32r\cd H_2\rbr{\fr{\eps}{2}}=64H_2\rbr{\fr{\eps}{2}}.
\end{align*}
\end{proof}

\begin{lemma}
\label{lemma:UASapprox}
Let $U\rbr{\fr{\eps}{2}}$ be the un-accelerated sample complexity on any bandit instance for identifying an $\fr{\eps}{2}$ good arm (note $ \eps_2=\fr{\eps}{2}$), i.e., 
\begin{align*}
    U\rbr{\fr{\eps}{2}}=\max_{i\ge g\rbr{\eps_2}+1}\fr{i}{\Delta_i^2}.
\end{align*}
Then up to a constant of $4$, the un-accelerated sample complexity can be approximated by $\max_{k\in\cbr{3,\cdots,K}}\fr{g(\eps_{k})}{\eps_{k}^2}$, i.e.,
\begin{align*}
    \max_{k\in\cbr{3,\cdots,K}}\fr{g(\eps_{k})}{\eps_{k}^2}\le U\rbr{\fr{\eps}{2}}\le\max_{k\in\cbr{3,\cdots,K}}\fr{4g(\eps_{k})}{\eps_{k}^2}.
\end{align*}
\end{lemma}
\begin{proof}
Let $i_0\ge g\rbr{\eps_2}+1$ satisfy
\begin{align*}
U\rbr{\fr{\eps}{2}}=\max_{i\ge g\rbr{\eps_2}+1}\fr{i}{\Delta_i^2}=\fr{i_0}{\Delta_{i_0}^2}.
\end{align*}
Denote $k_0$ as the smallest index of the peeling interval to which $i_0$ belongs, i.e., $\eps_{k_0-1}<\Delta_{i_0}\le\eps_{k_0}$. Since $i_0\ge g\rbr{\eps_2}+1$, we know $k_0\ge3$. By $\eps_{k_0-1}=\frac{\eps_{k_0}}{2}$, we have $\frac{\eps_{k_0}}{2}<\Delta_{i_0}$. Thus,
\begin{align*}
    U\rbr{\fr{\eps}{2}}=\fr{i_0}{\Delta_{i_0}^2}<\fr{4i_0}{\eps_{k_0}^2}\le\fr{4g(\eps_{k_0})}{\eps_{k_0}^2}\le\max_{k\in\cbr{3,\cdots,K}}\fr{4g(\eps_{k})}{\eps_{k}^2},
\end{align*}
where the second inequality is by $i_0\le g(\eps_{k_0})$.
By the definition of $U\rbr{\fr{\eps}{2}}$, we have 
\begin{align*}
    U\rbr{\fr{\eps}{2}}\ge\max_{k\in\cbr{3,\cdots,K}}\fr{g(\eps_{k})}{\Delta_{g(\eps_{k})}
    ^2}\ge\max_{k\in\cbr{3,\cdots,K}}\fr{g(\eps_{k})}{\eps_{k}^2}
\end{align*}
Thus,
\begin{align*}
    \max_{k\in\cbr{3,\cdots,K}}\fr{g(\eps_{k})}{\eps_{k}^2}\le U\rbr{\fr{\eps}{2}}\le\max_{k\in\cbr{3,\cdots,K}}\fr{4g(\eps_{k})}{\eps_{k}^2}.
\end{align*}
\end{proof}

\begin{lemma}
\label{lemma:twoGapsApprox}
Given that the event $E_1$ holds, for $i\ge g'(\eps_{2})+1$, $\Delta'_{i}$ the suboptimality gap of $i$-th best arm in the best bracket with respect to the best arm in the best bracket can be approximated by $\Delta''_{i}$, the suboptimality gap of $i$-th best arm in the best bracket with respect to the best arm in the entire instance, i.e.,
\begin{align*}
    \fr{1}{\Delta^{''2}_i}\le\fr{1}{{\Delta'_i}^2}\le\fr{4}{\Delta_i^{''2}}.
\end{align*}
\end{lemma}
\begin{proof}
The first inequality holds trivially, and the equality holds when the best arm out of the entire arm set is included in the best bracket. For the second inequality, as the best arm in the best bracket is $\fr{\eps}{4}$-good, we have 
\begin{align*}
    \Delta'_i-\fr{\eps}{4}\ge\Delta''_i-\fr{\eps}{2}\ge\fr{\Delta''_i}{2}-\fr{\eps}{4},
\end{align*}
which results in the second inequality.
\end{proof}

\subsection{Proof of Corollary \ref{cor-5}}
\label{app-sec-bucb_equal}
\newtheorem*{supp-7}{Corollary~\ref{cor-5}}
\begin{supp-7}
Consider the EqualGap$(m)$ instance.
BUCB achieves an expected $(\eps,\delta)$-unverifiable sample complexity as
\begin{align*}
    \ee{\tau_{\eps,\delta}}\le\hO{\fr{n}{\eps^2m}\log\rbr{\fr{1}{\delta}}}. 
\end{align*}
BSH achieves an $(\eps,\delta)$-unverifiable sample complexity as, with probability $1-\delta$,
\begin{align*}
   \tau_{\eps,\delta}\le\hO{\fr{n}{\eps^2m}\log\rbr{\fr{1}{\delta}}}.
\end{align*}
\end{supp-7}
\begin{proof}
For BUCB, Theorem 7 of \citet{Katz20} gives the following instance dependent upper bound for the $(\eps,\delta)$-unverifiable sample complexity, 
\begin{align*}
\ee{\tau_{\eps,\delta}}\le\min_{j\in[g(\eps)]}&\frac{1}{j}\left(\sum_{i=1}^{g(\eps)}\left(\Delta_{i \vee j, g(\eps)+1}\right)^{-2} \ln \left(\frac{n}{j \delta}\right)+\sum_{i=g(\eps)+1}^{n} \Delta_{j, i}^{-2} \ln \left(\frac{1}{\delta}\right)\right).
\end{align*}
Plug the instance 
\begin{align*}
\mu_1 = \frac{5}{4}\eps, \mu_i = \eps, \forall i\in\{2,\ldots,m\}, \mu_i=0, \forall i \ge m+1 \text{ for some } m\ge2, \eps>0~.
\end{align*}
We easily have 
\begin{align*}
    \ee{\tau_{\eps,\delta}}\le\hO{\fr{n}{\eps^2m}\log\rbr{\fr{1}{\delta}}}.
\end{align*}
For BSH,
\begin{align*}
\max_{i\ge g(\eps/2)+1}\frac{i}{\Delta_i^2}\cd\frac{1}{g\rbr{\fr{\eps}{4}}}=\frac{16n}{25m\eps^2}. 
\end{align*}
\end{proof}

\subsection{Proof of Corollary \ref{cor-4}}
\label{app-sec-bucb_poly}
\newtheorem*{supp-4}{Corollary~\ref{cor-4}}
\begin{supp-4}
Consider the Polynomial$(\alpha)$ instance; i.e., $\Delta_i=\rbr{\fr{i}{n}}^\alpha$ with $\alpha>0.5$. For any $\eps\in(0,1)$, let $\hat{\tau}_{\eps,\delta}$ be the upper bound of the expected $(\eps,\delta)$-unverifiable sample complexity reported in \citet[Theorem 7]{Katz20} for BUCB. 
Then, BUCB satisfies
\begin{align*} \ee{\tau_{\eps,\delta}}\le\hat{\tau}_{\eps,\delta}=\hTT{\eps^{-\fr{2\alpha-1}{\alpha}}n\log\rbr{\fr{1}{\delta}}}. 
\end{align*}
On the other hand, BSH satisfies, with probability $1-\delta$,
\begin{align*}
    \tau_{\eps,\delta}\le\hO{\eps^{-2}\log\rbr{\fr{1}{\delta}}}~.
\end{align*}
\end{supp-4}

\begin{proof}
We first show the result for Bracketing UCB. Theorem 7 of \citet{Katz20} gives the following instance dependent upper bound for the $(\eps,\delta)$-unverifiable sample complexity, 
\begin{align}
    \hat{\tau}_{\eps,\delta}=\min_{j\in[g(\eps)]}&\frac{1}{j}\left(\sum_{i=1}^{g(\eps)}\left(\Delta_{i \vee j, g(\eps)+1}\right)^{-2} \ln \left(\frac{n}{j \delta}\right)+\sum_{i=g(\eps)+1}^{n} \Delta_{j, i}^{-2} \ln \left(\frac{1}{\delta}\right)\right).\label{eq-bucb}
\end{align}
Plug the instance $\Delta_i=\rbr{\fr{i}{n}}^\alpha$ into the above,
\begin{align*}
    \hat{\tau}_{\eps,\delta}=&\min_{j\in[g(\eps)]}\frac{1}{j}\left(\sum_{i=1}^{g(\eps)}\left(\Delta_{i \vee j, g(\eps)+1}\right)^{-2} \ln \left(\frac{n}{j \delta}\right)+\sum_{i=g(\eps)+1}^{n} \Delta_{j, i}^{-2} \ln \left(\frac{1}{\delta}\right)\right)\\
    >&\min_{j\in[g(\eps)]}\frac{1}{j}\left(\sum_{i=1}^{g(\eps)}\left(\Delta_{i \vee j, g(\eps)+1}\right)^{-2} \ln \left(\frac{n}{j \delta}\right)\right)\\
    >&\min_{j\in[g(\eps)]}\frac{1}{j}\ln \left(\frac{1}{\delta}\right)\left(\sum_{i=1}^{g(\eps)}\left(\Delta_{i \vee j, g(\eps)+1}\right)^{-2}\right)\\
    >&\min_{j\in[g(\eps)]}\frac{1}{j}\ln \left(\frac{1}{\delta}\right)\left(\left(\Delta_{g(\eps), g(\eps)+1}\right)^{-2}\right)\\
    >&\min_{j\in[g(\eps)]}\frac{n^{2\alpha}}{j}\ln \left(\frac{1}{\delta}\right)\left(\fr{1}{\rbr{g(\eps)+1)^\alpha-(g(\eps))^\alpha}^2}\right)\\
    \stackrel{(a_1)}{>}&\min_{j\in[g(\eps)]}\frac{n^{2\alpha}}{j}\ln \left(\frac{1}{\delta}\right)\left(\fr{1}{\rbr{\alpha\rbr{g(\eps)+1}^{\alpha-1}}^2}\right)\\
    >&\min_{j\in[g(\eps)]}\frac{n^{2\alpha}}{j}\ln \left(\frac{1}{\delta}\right)\left(\fr{1}{\alpha^{2}\rbr{2g(\eps)}^{2\alpha-2}}\right)\\
    =&\frac{2^{-2\alpha+2}\alpha^{-2}n^{2\alpha}}{\rbr{g(\eps)}^{2\alpha-1}}\ln \left(\frac{1}{\delta}\right)\\
    \stackrel{(a_2)}{=}&\frac{2^{-2\alpha+2}\alpha^{-2}n}{\eps^{\fr{2\alpha-1}{\alpha}}}\ln \left(\frac{1}{\delta}\right).
\end{align*}
Note here we lower bound the upper bound in Theorem 7 of \citet{Katz20}. The inequality $(a_1)$ results by 
\begin{align*}
    (g+1)^\alpha - g^\alpha
    = \int_g^{g+1} \alpha x^{\alpha-1} \dif x 
    \le \alpha (g+1)^{\alpha - 1}.
\end{align*}
The equation $(a_2)$ is because of $g(\eps)=n\eps^{1/\alpha}$ by the definition of the instance.

We also upper bound \eqref{eq-bucb}.
Note that for $i>j$, we have 
\begin{align*}
  i^\alpha - j^\alpha = \int_j^i \alpha x^{\alpha - 1} \dif x \ge \alpha (i-j)\cd j^{\alpha - 1} ~.
\end{align*}
With this, we have
\begin{align*}
  \hat{\tau}_{\eps,\delta}
  &\le n^{2\alpha}  \del{ \min_{j\le g(\eps)}  \fr{1}{(g(\eps)+1-j)^2 j^{2\alpha-2}} + \fr{1}{j}\sum_{i=j+1}^{g(\eps)} \fr{1}{(g(\eps)+1-i)^2 i^{2\alpha - 2}} + \fr{1}{j}\sum_{i=g(\eps)+1}^{n}\fr{1}{(i-j)^2 j^{2\alpha-2}    }  }
\end{align*}
We can bound the third sum by
\begin{align*}
  \fr{1}{j}\sum_{i=g(\eps)+1}^{n}\fr{1}{(i-j)^2 j^{2\alpha-2} } \lsim \fr{1}{j^{2\alpha-1}} \cd \fr{1}{g(\eps)+1-j}  
\end{align*}
Take $j=g(\eps) - \sqrt{g(\eps)}$. This means that $j = \Theta(g(\eps))$ and $g(\eps)-j = \sqrt{g(\eps)}$. With this choice,
\begin{align*}
  \hat{\tau}_{\eps,\delta}&\lsim n^{2\alpha} \del{  \fr{1}{g(\eps) \cd g(\eps)^{2\alpha - 2}} 
    + \fr{1}{g(\eps)} \cd \del{\sum_{i=g(\eps)-\sqrt{g(\eps)}+1}^{g(\eps)} \fr{1}{(g(\eps)+1-i)^2} \cd \fr{1}{g(\eps)^{2\alpha - 2}}   } 
    + \fr{1}{g(\eps)^{2\alpha - 1}}\cd \fr{1}{1 + \sqrt{g(\eps)}} 
  } 
  \\&\le n^{2\alpha} \del{  \fr{1}{g(\eps)^{2\alpha - 1}} 
    + \fr{1}{g(\eps)^{2\alpha - 1}} \cd \fr{\pi^2}6 
    + \fr{1}{g(\eps)^{2\alpha - 1}}\cd \fr{1}{1 + \sqrt{g(\eps)}} 
  } 
  \\&\lsim n^{2\alpha} \cd \fr{1}{g(\eps)^{2\alpha-1}} 
         = n^{2\alpha} \fr{1}{(n \eps^{1/\alpha})^{2\alpha - 1}} 
         = \fr{n}{\eps^{\fr{2\alpha - 1}{\alpha} }}. 
\end{align*}

For Bracketing SH,
\begin{align*}
\max_{i\ge g(\eps/2)+1}\frac{i}{\Delta_i^2}\cd\frac{1}{g\rbr{\fr{\eps}{4}}}=\max_{i\ge g\rbr{\eps/2}+1}i^{1-2\alpha}n^{2\alpha}\fr{1}{n\rbr{\fr{\eps}{4}}^{1/\alpha}}<\rbr{n\rbr{\frac{\eps}{2}}^{1/\alpha}}^{1-2\alpha}n^{2\alpha}\fr{1}{n\rbr{\fr{\eps}{4}}^{1/\alpha}}<\frac{16}{\eps^2}.
\end{align*}
where both the last inequality and the last equality are by $\alpha>0.5$.

\end{proof}

\subsection{An alternative upper bound for BSH}

The following bound of BSH is similar to the bound of Theorem \ref{them:datapoor}. Though the bound of Theorem \ref{them:otherBSH} is minimax in nature, we show, in Corollary  \ref{cor:bsh_old_eqgap} and \ref{cor:bsh_old_poly}, that it still achieves the same upper bound as BUCB for the EqualGap$(m)$ instance and a better upper bound for Polynomial$(\alpha)$ than BUCB when $n$ is large enough because it does not scale with $n$ polynomially. Since this bound involves an optimization problem, we speculate it could be tighter than the bound of Theorem \ref{them:datapoor} for certain instances.
\begin{theorem}
\label{them:otherBSH}
For any $\eps\in(0,1)$, the error probability of Bracketing SH satisfies 
\begin{align*}
\pp{\mu_{J_t}<\mu_1-\eps}\le\ep{-\tilde{\Theta}\rbr{\min\cbr{\max_{i\in[g(\eps/2)]}\fr{i\Delta^2_{i,g(\eps/2)+1}}{n},\eps^2}t'}},
\end{align*}
where $t'=\fr{t}{4\ln t}$. Accordingly, $\hO{\max\cbr{\min_{i\in[g(\eps/2)]}\fr{n}{i\Delta^2_{i,g(\eps/2)+1}}\log\rbr{\fr{1}{\delta}},\fr{1}{\eps^2}\log\rbr{\fr{1}{\delta}}}}$ samples are sufficient for Bracketing SH to output an $\eps$-good arm with probability $1-\delta$.
\end{theorem}
\begin{proof}
The proof idea is similar to the proof of Theorem \ref{them:datapoor}. The core idea is to find the best bracket whose size is well-balanced.
At round $t$, define the best bracket as bracket $r_i^*(t)=\ur{\log_2\rbr{c\Delta_{i,g(\eps/2)+1}^2t'}}$, where $i\in[g(\eps/2)]$ is a free parameter and $c$ is a logarithmic term for shorthand $c=\rbr{16\ln(4e)\log_2^2(\frac{i\Delta^2_{i,g(\eps/2)+1}t'}{n})\log_2\rbr{2\Delta_{i,g(\eps/2)+1}^2t'}}^{-1}$. Thus the size of bracket $r_i^*(t)$ satisfies,
\begin{align*}
    c\Delta_{i,g(\eps/2)+1}^2t'\le\abr{A_{r_i^*(t)}}\le2^{\log_2\rbr{c\Delta_{i,g(\eps/2)+1}^2t'}+1}=2c\Delta_{i,g(\eps/2)+1}^2t'.
\end{align*}
Define the following events:
\begin{itemize}
    \item $E_1$: the number of arms with mean reward at least $\mu_i$ included in bracket $r_i^*(t)$ is at least $j_t$, where $j_t=\dr{\frac{ci\Delta^2_{i,g(\eps/2)+1}t'}{2n}}$.
    \item $E_2$: $\mu_{a_{r_i^*(t)}}\ge\mu_1-\eps/2$, where $a_{r_i^*(t)}$ is the arm representing bracket $r_i^*(t)$ at round $t$.
\end{itemize}
Then the error probability of Bracketing SH can be expressed as follows,
\begin{align}
    \pp{\mu_{J_t}<\mu_1-\eps}=&\pp{\mu_{J_t}<\mu_1-\eps, E^c_1}+\pp{\mu_{J_t}<\mu_1-\eps,E_1,E^c_2}+\pp{\mu_{J_t}<\mu_1-\eps,E_1,E_2}\nn\\
    \le&\pp{E^c_1}+\pp{E_1,E^c_2}+\pp{\mu_{J_t}<\mu_1-\eps,E_2}. \label{eq-9}
\end{align}
We bound the three terms respectively. For the first term, we use the same technique as in the proof sketch.
\begin{align}
    \pp{E^c_1}\le&\ep{-\abr{A_{r_i^*(t)}}\cd\textsf{KL}\rbr{\frac{j_t}{\abr{A_{r_i^*(t)}}},\frac{i}{n}}}\nn\\
    \stackrel{(a_1)}{\le}&\ep{-\frac{\abr{A_{r_i^*(t)}}n}{2i}\rbr{\frac{j_t}{\abr{A_{r_i^*(t)}}}-\frac{i}{n}}^2}\nn\\
    \le&\ep{-\frac{\Delta_{i,g(\eps/2)+1}^2t'n}{2i}\rbr{\frac{\frac{i\Delta^2_{i,g(\eps/2)+1}t'}{2n}}{\Delta_{i,g(\eps/2)+1}^2t'}-\frac{i}{n}}^2}\nn\\
    \le&\ep{-\frac{ci\Delta_{i,g(\eps/2)+1}^2t'}{8n}}\nn\\
    \le&\ep{-\tilde{\Theta}\rbr{\frac{i\Delta^2_{i,g(\eps/2)+1}t'}{n}}}.\label{eq-6}
\end{align}
The inequality $(a_1)$ is due to $\textsf{KL}\rbr{p,q}\ge\fr{(p-q)^2}{2\max\{p,q\}}$.
For the second term of \eqref{eq-9}, we bound the event that SH returns a non-$(j_t,\Delta_{i,g(\eps/2)+1})$-good arm, where $j_t$ is with respect to the best bracket. We use $\mu_{(j_t)}$ for the mean reward of the $j_t$-th best arm in bracket $r_i^*(t)$. Thus $\mu_{(j_t)}-\Delta_{i,g(\eps/2)+1}\ge\mu_1-\eps/2$. By Theorem \ref{them:minmax}, 
\begin{align}
    \pp{E_1,E^c_2}\le&\pp{\mu_{a_{r_i^*(t)}}<\mu_1-\eps/2\mid E_1}\nn\\
    \le&\pp{\mu_{a_{r_i^*(t)}}<\mu_{(j_t)}-\Delta_{i,g(\eps/2)+1}\mid E_1}\nn\\
    \stackrel{(a_1)}{\le}&\log_2\abr{A_{r_i^*(t)}}\cd\ep{-\textup{const}\cd j_t\rbr{\fr{\Delta_{i,g(\eps/2)+1}^2t'}{4\abr{A_{r_i^*(t)}}\log_2^2(2j_t)\log_2\abr{A_{r_i^*(t)}}}-\ln(4e)}}\nn\\
    \le&\log_2\abr{A_{r_i^*(t)}}\cd\ep{-\textup{const}\cd \frac{ci\Delta^2_{i,g(\eps/2)+1}t'}{n}}\nn\\
    \le&\ep{-\tilde{\Theta}\rbr{\frac{i\Delta^2_{i,g(\eps/2)+1}t'}{n}}}.\label{eq-7}
\end{align}
For the inequality $(a_1)$, we bound the term in parenthesis as a constant,
\begin{align*}
    &\fr{\Delta_{i,g(\eps/2)+1}^2t'}{4\abr{A_{r_i^*(t)}}\log_2^2(2j_t)\log_2\abr{A_{r_i^*(t)}}}-\ln(4e)\\
    >&\fr{\Delta_{i,g(\eps/2)+1}^2t'}{8c\Delta_{i,g(\eps/2)+1}^2t'\log_2^2(2c\frac{i\Delta^2_{i,g(\eps/2)+1}t'}{2n})\log_2\rbr{2c\Delta_{i,g(\eps/2)+1}^2t'}}-\ln(4e)\\
    >&\fr{1}{8c\log_2^2(\frac{i\Delta^2_{i,g(\eps/2)+1}t'}{n})\log_2\rbr{2\Delta_{i,g(\eps/2)+1}^2t'}}-\ln(4e)\\
    >&2\ln(4e)-\ln(4e)\\
    =&\ln(4e).
\end{align*}
The third term of \eqref{eq-9} means an arm whose mean reward is less than $\mu_1-\eps$ has a larger empirical reward than an arm whose mean reward is larger than $\mu_1-\eps/2$. Note the recently opened brackets may not have finished their first SH yet. In this case, the DSH always returns an empirical reward of negative infinity. Thus the arms returned from these brackets will never be selected by BSH. Denote $\dot{A}(t):=\cbr{a\in\bigcup_{k=1,k\neq r_i^*(t)}^{L_t}A_{k},\mu_a<\mu_1-\eps}$.
We have 
\begin{align*}
    \abr{\dot{A}(t)}\le&\sum_{k=1}^{L_t}2^k\le2\cd\rbr{2^{1 + \log_2\rbr{1 + \ln(2) t}}-1}\le8\cd\rbr{1 + \ln(2) t}.
\end{align*}
By the number of arm pulls in Lemma \ref{lem:numOfpull},
\begin{align}
    \pp{\mu_{J_t}<\mu_1-\eps,E_2}\le&\pp{\exists a_1\in A_{r_i^*(t)},\exists a_2\in\dot{A}(t), \text{s.t.},\mu_{a_1}\ge\mu_1-\eps/2,\hat{\mu}_{a_1}<\hat{\mu}_{a_2}}\nn\\
    \le&\abr{\dot{A}(t)}\abr{A_{r_i^*(t)}}\ep{-\rbr{\fr{\eps}{2}}^2\cd\text{const}\cd\fr{t'}{\log_2n}}\nn\\
    \le&\ep{-\text{const}\cd\eps^2\fr{t'}{\log_2n}+\ln\rbr{\abr{\dot{A}(t)}\abr{A_{r_i^*(t)}}}}\nn\\
    \le&\ep{-\text{const}\cd\eps^2\fr{t'}{\log_2n}+\ln\rbr{8\cd\rbr{1 + \ln(2) t}2\Delta_{i,g(\eps/2)+1}^2t'}}\nn\\
    \le&\ep{-\text{const}\cd\eps^2\fr{t'}{\log_2n}+\OO{\ln t}}\nn\\
    \le&\ep{-\tilde{\Theta}\rbr{\eps^2t'}}.\label{eq-8}
\end{align}
Combine \eqref{eq-6}\eqref{eq-7}\eqref{eq-8} and the fact $i\in[g(\eps/2)]$ is a free parameter, then we have 
\begin{align*}
\pp{\mu_{J_t}<\mu_1-\eps}\le\ep{-\tilde{\Theta}\rbr{\min\cbr{\max_{i\in[g(\eps/2)]}\fr{i\Delta^2_{i,g(\eps/2)+1}}{n},\eps^2}t'}}.
\end{align*}
\end{proof}

\begin{corollary}
\label{cor:bsh_old_eqgap}
Consider the EqualGap$(m)$ instance. 
BUCB achieves an expected $(\eps,\delta)$-unverifiable sample complexity as
\begin{align*}
    \ee{\tau_{\eps,\delta}}\le\hO{\fr{n}{\eps^2m}\log\rbr{\fr{1}{\delta}}}.
\end{align*}
For BSH, the bound of Theorem \ref{them:otherBSH}  shows that BSH achieves an $(\eps,\delta)$-unverifiable sample complexity as
\begin{align*}
   \tau_{\eps,\delta}\le\hO{\fr{n}{\eps^2m}\log\rbr{\fr{1}{\delta}}}
\end{align*}
with probability $1-\delta$.
\end{corollary}
\begin{proof}
For BUCB, the proof is presented in the proof of Corollary~\ref{cor-5}.
For BSH, the result is straightforward by plugging the instance,
\begin{align*}
    \min_{i\in[g(\eps/2)]}\fr{n}{i\Delta^2_{i,g(\eps/2)+1}}\log\rbr{\fr{1}{\delta}}\le&\min\cbr{\fr{n}{\rbr{\eps/2}^2}\log\rbr{\fr{1}{\delta}},\fr{n}{m\eps^2}\log\rbr{\fr{1}{\delta}}}=\fr{n}{m\eps^2}\log\rbr{\fr{1}{\delta}}.
\end{align*}
\end{proof}

The following corollary shows that even if we use the upper bound of Theorem \ref{them:otherBSH} that is minimax in nature, we can still achieve an upper bound for the Polynomial$(\alpha)$ instance that does not scale with the instance size $n$, and for moderately large $\epsilon$ or large enough $n$, the upper bound of Theorem \ref{them:otherBSH} is always better than the upper bound of BUCB.

\begin{corollary}
\label{cor:bsh_old_poly}
Consider the Polynomial$(\alpha)$ instance $\Delta_i=\rbr{\fr{i}{n}}^\alpha$.
BUCB achieves an expected $(\eps,\delta)$-unverifiable sample complexity as
\begin{align*}
\ee{\tau_{\eps,\delta}}\le\hat{\tau}_{\eps,\delta}=\hTT{\eps^{-\fr{2\alpha-1}{\alpha}}n\log\rbr{\fr{1}{\delta}}}. 
\end{align*}
For BSH, the bound of Theorem \ref{them:otherBSH}  shows that BSH achieves an $(\eps,\delta)$-unverifiable sample complexity as
\begin{align*}
    \tau_{\eps,\delta}\le\hO{\eps^{-\fr{2\alpha+1}{\alpha}}\log\rbr{\fr{1}{\delta}}}~.
\end{align*}
with probability $1-\delta$.
\end{corollary}
\begin{proof}
For BUCB, the proof is presented in the proof of Corollary~\ref{cor-4}.
For BSH, first, notice that
\begin{align*}
    \Delta^2_{i,j+1}>\Delta^2_{i,j}=&\rbr{\rbr{\fr{j}{n}}^\alpha-\rbr{\fr{i}{n}}^\alpha}^2\\
    =&\rbr{\fr{j}{n}}^{2\alpha}+\rbr{\fr{i}{n}}^{2\alpha}-2\rbr{\fr{ij}{n^2}}^{\alpha}.
\end{align*}
The sample complexity satisfies,
\begin{align*}
    \min_{i\in[g(\eps/2)]}\fr{n}{i\Delta^2_{i,g(\eps/2)+1}}\log\rbr{\fr{1}{\delta}}\le&\min_{i\in[g(\eps/2)]}\fr{n}{i\rbr{\rbr{\fr{g(\eps/2)}{n}}^{2\alpha}+\rbr{\fr{i}{n}}^{2\alpha}-2\rbr{\fr{ig(\eps/2)}{n^2}}^{\alpha}}}\log\rbr{\fr{1}{\delta}}\\
    =&\min_{i\in[g(\eps/2)]}\fr{n\cd n^{2\alpha}}{i\rbr{\rbr{g(\eps/2)}^{2\alpha}+i^{2\alpha}-2\rbr{ig(\eps/2)}^{\alpha}}}\log\rbr{\fr{1}{\delta}}\\
    \stackrel{(a_1)}{\le}&\fr{2\cd n^{2\alpha+1}}{(g(\eps/2))^{2\alpha+1}\rbr{1+2^{-2\alpha}-2^{-\alpha+1}}}\log\rbr{\fr{1}{\delta}}\\
    =&\fr{2\cd n^{2\alpha+1}}{(n(\eps/2)^{1/\alpha})^{2\alpha+1}\rbr{1+2^{-2\alpha}-2^{-\alpha+1}}}\log\rbr{\fr{1}{\delta}}\\
    =&2^{\fr{2\alpha+1}{\alpha}+1}\rbr{1+2^{-2\alpha}-2^{-\alpha+1}}^{-1}\eps^{-\fr{2\alpha+1}{\alpha}}\log\rbr{\fr{1}{\delta}}.\\
\end{align*}
The inequality $(a_1)$ is by taking $i=g(\eps/2)/2$.
\end{proof}

\subsection{\texorpdfstring{$(\eps,\delta)$}--unverifiable sample complexity of algorithms with \texorpdfstring{$\eps$}--error probability bound}
\label{sec-unverifiable-sc}

\newtheorem*{def-uver}{Definition~\ref{def-uver}}
\begin{def-uver}[$(\eps,\delta)$-unverifiable sample
complexity \citep{Katz20}] 
For an algorithm $\pi$ and an instance $\rho$. Let $\tau_{\eps,\delta}$ be a stopping time such that
\begin{align*}
    \pp{\forall t\ge\tau_{\eps,\delta}: \mu_{J_t}>\mu_1-\eps}\ge1-\delta.
\end{align*}
Then, $\tau_{\eps,\delta}$ is called $(\eps,\delta)$-unverifiable sample complexity of the algorithm with respect to $\rho$.
\end{def-uver}
The $(\eps,\delta)$-unverifiable sample complexity requires all the outputs at and after $t=\tau_{\eps,\delta}$ are $\eps$-good. This is slightly different from the sample complexity of BSH. However, we show they are order-wise equivalent. To see this, let us consider an anytime algorithm that achieves an exponentially decreasing error probability,
\begin{align*}
    \pp{\mu_{J_t}<\mu_{1}-\eps}\le\ep{-c\cd t}.
\end{align*}
Let $t_{\eps,\delta}$ be the first time step that satisfies $\pp{\mu_{J_{t_{\eps,\delta}}}<\mu_{1}-\eps}\le\delta$. Note this only guarantees the output at round $t_{\eps,\delta}$ is $\eps$-good with probability $1-\delta$, instead of all the outputs at and after $t_{\eps,\delta}$. We prove our claim by showing the probability that there is a non-$\eps$-good output at or after $t_{\eps,\delta}$ has an exponentially decreasing rate with the same parameter.
\begin{align*}
    \pp{\exists t\ge t_{\eps,\delta}: \mu_{J_t}<\mu_1-\eps}\le&\sum_{t=t_{\eps,\delta}}^{\infty}\pp{\mu_{J_t}<\mu_1-\eps}\\
    \le&\sum_{t=t_{\eps,\delta}}^{\infty}\ep{-c\cd t}\\
    =&\lim_{t\rightarrow\infty}\ep{-c\cd t_{\eps,\delta}}\fr{1-\ep{-c\cd t}}{1-\ep{-c}}\\
    <&\ep{-c\cd t_{\eps,\delta}}\fr{1}{1-\ep{-c}}\\
    =&\ep{-\Theta(c\cd t_{\eps,\delta})}.
\end{align*}

\section{Discussion on the Practical Algorithm Implementation}
\label{sec-discussion-implementation}
The design of DSH (Algorithm \ref{Alg:DSH}) and BSH (Algorithm \ref{Alg:BSH}) involves two key ideas, bracketing and doubling trick. 
In practice, both two techniques could be implemented in a more efficient way. In addition, the base algorithm SH (Algorithm \ref{Alg:SH}) has a commonly used implementation for reusing samples \citep{baharav2019ultra,jun16anytime}. We summarize these strategies as follows for the practitioner's consideration. However, they do not make any order-wise difference (not more than logarithmic factors) in terms of the theoretical guarantee.
\begin{itemize}
    \item SH: The whole procedure of SH is divided into $\log_2n$ stages. All the stages are independent of each other since the samples of prior stages are abandoned when a new stage starts, as described in Algorithm \ref{Alg:SH}. In practical usage, one can keep all the samples since the first stage for the surviving arms. 
    \item DSH: The common doubling trick does not require keeping the samples after finishing one invocation of the base algorithm. 
    Not requiring to keep the samples is convenient to implement as we only need to repeatedly initialize a new instance of the base algorithm with the doubled budget parameter. In fact, for multi-armed bandit problems, keeping all the samples of prior invocations is beneficial. For the implementation, one can create a class for the base algorithm SH, and we initialize a new instance of the class with initial empirical rewards equal to the empirical rewards saved in the prior instance.
    \item BSH: 
    The bracketing technique does not promise to avoid overlap
    with the already-opened brackets.
    A more practical implementation of BSH is to share the samples of the same arm across different brackets. Such that the empirical reward is more accurate. 
    However, reusing sampling across different brackets is meaningless if we consider the infinitely-armed bandit models. 
    Because, for the infinitely-armed bandit models, we will not draw exactly the same arm more than once. Thus it is barely possible to have overlap among the opened brackets.
\end{itemize}
Note our implementation reported in section \ref{sec-expr} only uses the first reusing strategy for the base algorithm SH.

\putbib[library-shared,ref]
\end{bibunit}


\bibliographystyle{icml2023}
\end{document}